\pdfoutput=1
\documentclass[11pt]{article}

\usepackage[letterpaper,margin=1in]{geometry}
\usepackage{lipsum}

\usepackage[utf8]{inputenc} %
\usepackage[T1]{fontenc}    %

\usepackage{url}            %
\usepackage{booktabs}       %
\usepackage{nicefrac}       %
\usepackage{microtype}      %
\usepackage{enumitem}
\usepackage{dsfont}
\usepackage{multicol}
\usepackage{xcolor}
\usepackage{mdframed}

\usepackage{amsthm,amsfonts,amsmath,amssymb,epsfig,color,float,graphicx,verbatim, enumitem}

\usepackage{algorithmicx}
\usepackage[noend]{algpseudocode}
\usepackage{algorithm}

\usepackage{mathtools}
\usepackage{bbm}

\usepackage{thm-restate}

\definecolor{green}{rgb}{0.0, 0.5, 0.0}
\usepackage[colorlinks,citecolor=blue,linkcolor=magenta,bookmarks=true,hypertexnames=false]{hyperref}
\usepackage[nameinlink,capitalize]{cleveref}

\crefformat{equation}{(#2#1#3)}
\crefname{lemma}{lemma}{lemmata}
\crefname{claim}{claim}{claims}
\crefname{theorem}{theorem}{theorems}
\crefname{proposition}{proposition}{propositions}
\crefname{corollary}{corollary}{corollaries}
\crefname{claim}{claim}{claims}
\crefname{remark}{remark}{remarks}
\crefname{definition}{definition}{definitions}
\crefname{fact}{fact}{facts}
\crefname{question}{question}{questions}
\crefname{condition}{condition}{conditions}
\crefname{algorithm}{algorithm}{algorithms}
\crefname{assumption}{assumption}{assumptions}
\crefname{notation}{notation}{notation}
\crefname{cond}{Condition}{Conditions}

  {%
   \par\noindent{\bfseries\upshape Proof Sketch\ }%
  }%
  {\qed}

\newtheorem{theorem}{Theorem}[section]
\newtheorem{lemma}[theorem]{Lemma}
\newtheorem{proposition}[theorem]{Proposition}
\newtheorem{corollary}[theorem]{Corollary}
\newtheorem{claim}[theorem]{Claim}

\newtheorem{definition}[theorem]{Definition}

\newtheorem{fact}[theorem]{Fact}

\theoremstyle{definition}

\newtheorem{remark}[theorem]{Remark}

\newcommand{\eps}{\epsilon}

\newcommand{\Ind}{\mathds{1}}
\newcommand{\1}{\Ind}

\renewcommand{\Pr}{\operatorname*{\mathbb{P}}}
\newcommand{\Var}{\operatorname*{\mathrm{Var}}}
\newcommand{\Cov}{\operatorname*{\mathrm{Cov}}}
\newcommand{\E}{\operatorname*{\mathbb{E}}}

\newcommand{\poly}{\operatorname*{\mathrm{poly}}}

\def\R{\mathbb R}

\def\Z{\mathbb Z}

\newcommand{\cA}{\mathcal{A}}
\newcommand{\cB}{\mathcal{B}}

\newcommand{\cN}{\mathcal{N}}

\newcommand{\op}{\textnormal{op}}
\newcommand{\fr}{\textnormal{F}}

\def\proj{\mathrm{Proj}}
\newcommand{\tr}{\mathrm{tr}}

\def\filteredvoronoi{\textsc{FilteredVoronoi}}
\def\sizebasedpruning{\textsc{SizeBasedPruning}}
\def\distancebasedpruning{\textsc{DistanceBasedPruning}}

\newcommand{\muhat}{\hat{\mu}}

\def\colorful{0}

\ifnum\colorful=1

\else

\fi

\allowdisplaybreaks

\title{Clustering Mixtures of Bounded Covariance Distributions\\ Under Optimal Separation}

\author{
Ilias Diakonikolas\thanks{Supported by NSF Medium Award CCF-2107079 and NSF Award CCF-1652862 (CAREER).}\\
University of Wisconsin-Madison\\
{\tt ilias@cs.wisc.edu}\\
\and
Daniel M. Kane\thanks{Supported by NSF Medium Award CCF-2107547 and NSF Award CCF-1553288 (CAREER).}\\
University of California, San Diego\\
{\tt dakane@cs.ucsd.edu}
\and
Jasper C.H. Lee\thanks{Supported by NSF Medium Award CCF-2107079, NSF AiTF Award CCF-2006206, and a Croucher Fellowship for Postdoctoral Research.}\\
University of Wisconsin-Madison\\
{\tt jasper.lee@wisc.edu}\\
\and
Thanasis Pittas\thanks{Supported by NSF Medium Award CCF-2107079 and NSF Award DMS-2023239 (TRIPODS).}\\
University of Wisconsin-Madison\\
{\tt pittas@wisc.edu}\\
}
\date{\vspace*{-10pt}\today}

\begin{document}

\maketitle

\vspace*{-0.8cm}

\begin{abstract}%
We study the clustering problem for mixtures of bounded covariance distributions, 
under a fine-grained separation assumption. Specifically, given samples from a 
$k$-component mixture distribution $D = \sum_{i =1}^k w_i P_i$, 
where each $w_i \ge \alpha$ for some known parameter $\alpha$, 
and each $P_i$ has unknown covariance $\Sigma_i \preceq \sigma^2_i \cdot I_d$ 
for some unknown $\sigma_i$, the goal is to cluster the samples 
assuming a pairwise mean separation in the order of $(\sigma_i+\sigma_j)/\sqrt{\alpha}$ between every pair of components $P_i$ and $P_j$. Our main contributions are as follows:
\begin{itemize}[leftmargin=*]
\item For the special case of nearly uniform mixtures, we give the first 
polynomial-time algorithm for this clustering task. Prior work either required separation scaling with the maximum cluster standard deviation (i.e.~$\max_i \sigma_i$)~\cite{diakonikolas2022clustering} or required both additional 
structural assumptions and mean separation scaling as a large degree polynomial in $1/\alpha$~\cite{BKK22}. 

\item For arbitrary (i.e.~general-weight) mixtures, we point out that accurate 
clustering is information-theoretically impossible under our fine-grained mean 
separation assumptions. We introduce the notion of a {\em clustering refinement} 
--- a list of not-too-small subsets satisfying a similar separation,
and which can be merged into a clustering approximating the 
ground truth --- and show that it is possible to efficiently compute an accurate 
clustering refinement of the samples.
Furthermore, under a variant of the ``no large 
sub-cluster'' condition introduced in prior work~\cite{BKK22}, 
we show that our algorithm will output an accurate clustering, not just a refinement, even for 
general-weight mixtures.
As a corollary, we obtain efficient clustering algorithms for mixtures of well-conditioned high-dimensional log-concave distributions. 
\end{itemize}
Moreover, our algorithm is robust to a fraction of adversarial outliers 
comparable to $\alpha$.

At the technical level, our algorithm proceeds by first 
using list-decodable mean estimation 
to generate a  polynomial-size list of possible cluster means, 
before successively pruning candidates using 
a carefully constructed convex program.
In particular, the convex program takes as input 
a candidate mean $\muhat$ and a scale parameter $\hat{s}$, 
and determines the existence 
of a subset of points that could plausibly form a cluster 
with scale $\hat{s}$ centered around $\muhat$.
While the natural way of designing this program makes it non-convex, 
we construct a convex relaxation which remains satisfiable 
by (and only by) not-too-small subsets of true clusters.
\end{abstract}

\setcounter{page}{0}

\thispagestyle{empty}

\newpage

\section{Introduction}

Clustering mixture models is one of the most basic 
and widely-used statistical primitives 
on data samples from high-dimensional distributions, 
with applications in a variety of fields, 
including bioinformatics, astrophysics, and 
marketing~\cite{Lindsay:95, GGMM10}; see \cite{titterington_85} for an extensive list of applications.
Informally, the input is a set of $n$ samples drawn from 
a mixture distribution $D = \sum_{i = 1}^k w_i P_i$ over $\R^d$, 
where $w_i$ is the mixing weight of component $P_i$.
The goal is to cluster (most of) the samples such that 
the clustering is approximately equal to partitioning the data according 
to the ground truth; namely, partitioning samples according to 
which mixture component they were drawn from.
For the clustering task to be information-theoretically possible, 
it is common to make concentration assumptions on 
each mixture component $P_i$ (e.g.~sub-Gaussianity, or a bounded moments 
assumption), as well as on the pairwise separation between the means of the 
components.

The prototypical case is that of Gaussian mixtures and 
has been extensively studied in the literature; 
see, e.g.~\cite{VempalaWang:02,KSV:05,AchlioptasMcSherry:05} 
and references therein.
In more detail,~\cite{VempalaWang:02} studied the clustering 
of data drawn from mixtures of separated spherical Gaussians.
Subsequent work~\cite{KSV:05,AchlioptasMcSherry:05} built on 
the approach of~\cite{VempalaWang:02} to design clustering algorithms 
for mixtures of separated Gaussians with general covariances.
The main algorithmic technique underlying these papers is to apply $k$-PCA 
in order to discover the subspace spanned by the means 
of the mixture components.

The focus of this paper is the more general {\em heavy-tailed} setting, 
where each component is only assumed to have {\em bounded covariance} 
instead of stronger concentration. Specifically, suppose that each
component $P_i$ has unknown covariance matrix $\Sigma_i$ that satisfies 
$\Sigma_i \preceq \sigma^2 \cdot  I_d$, for some unknown parameter $\sigma>0$. For notational 
simplicity, we restrict this discussion to {\em uniform} mixtures 
(corresponding to the case that $w_i = 1/k$ for all $i \in [k]$).
Then, unless the component means have pairwise $\ell_2$-distance 
$\gg \sigma \sqrt{k}$, accurate clustering is information-theoretically impossible in the worst-case. On the positive side, the recent work of~\cite{diakonikolas2022clustering} gave a computationally efficient algorithm which achieved the best worst-case separation: 
if all the components $P_i$ have covariances 
$\Sigma_i \preceq \sigma^2 \cdot I_d$, then~\cite{diakonikolas2022clustering} 
showed that it is possible to accurately cluster when given a pairwise separation 
of $C \sigma \sqrt{k}$, where $C>0$ is a sufficiently large universal 
constant\footnote{We note that~\cite{diakonikolas2022clustering} gave 
an almost-linear time algorithm that 
succeeds under slightly stronger separation (within a $\log(k)$-factor of the optimal). If one allows polynomial-time algorithms, this extra factor does not appear.}. 

The preceding discussion suggests that the algorithmic problem of clustering mixtures of bounded covariance distributions under the information-theoretically optimal mean estimation (within constant factors) is fully resolved.
Yet, consider the simple example shown in \Cref{fig:kPCA} below.
\begin{figure}[h!]
    \centering
    \includegraphics[clip=true,trim={0cm 0.7cm 0cm 0.3cm},keepaspectratio,height=1.3in]{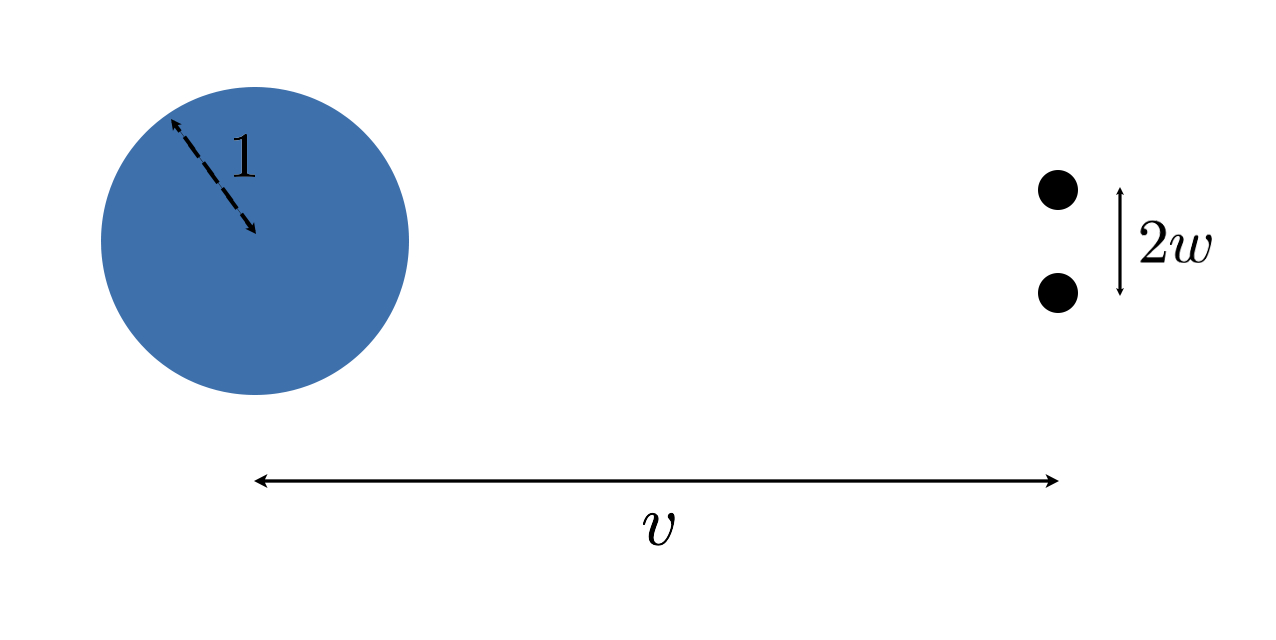}
    \caption{Example ``well-separated'' mixture distribution that cannot be handled by the algorithm of~\cite{diakonikolas2022clustering}.}
    \label{fig:kPCA}
\end{figure}

In this example, we have an identity-covariance distribution on the left, 
separated by distance $v \gg 1$ from a pair of $0$-covariance distributions 
on the right, which are in turn separated by some small distance $2w \ll 1$.
This example is clearly clusterable and ``well-separated'', 
since there is essentially no overlap between any of the mixture components.
However, the example cannot be handled 
by the algorithm of~\cite{diakonikolas2022clustering} or earlier 
algorithms\footnote{We note, for example, that the 
algorithm~\cite{awasthi2012improved} produces an accurate clustering under 
separation $\Delta \gg k \sigma$.}
for the following reason:
the largest variance is $\sigma^2 = 1$, but the two $0$-covariance 
distributions are separated only by $2w \ll 1$ --- 
instead of the required 
$\Theta(\sigma \sqrt{k}) = \Theta(1)$ separation.

The above example illustrates an important conceptual weakness 
of prior work in the heavy-tailed (bounded covariance) setting: 
it requires that the pairwise mean separation is measured 
by the \emph{maximum} covariance across all the mixture components --- 
even if the pair of components in question both have small covariances.
This distinction can make a large quantitative difference 
in both theory and practice.
Indeed, even for the special case that the components are of approximately the same size (cardinality), their relative radii 
may dramatically differ.

\paragraph{Motivation: Achieving fine-grained separation} 
A more reasonable separation assumption that we focus on in this paper is as follows.
Suppose that the components $P_i$ and $P_j$ 
have maximum standard deviations $\sigma_i$ and $\sigma_j$ respectively. 
Then we require the corresponding means $\mu_i$ and $\mu_j$ 
to be separated in $\ell_2$-distance by a quantity scaling with 
$\sigma_i+\sigma_j$. Note that this is much weaker than 
the prior assumption scaling with $\max_i \sigma_i$.
We also point out that clustering under such fine-grained separation 
has been achieved for Gaussian components in earlier works~\cite{AchlioptasMcSherry:05, KSV:05, Brubaker:09}.
However, to the best of our knowledge, no such result was previously known 
for the bounded covariance setting. 
Motivated by this gap in the literature, 
in this paper we ask: 
\begin{quote}
{\em Is it possible to efficiently cluster data from mixtures of bounded 
covariance distributions under the fine-grained separation assumption? 
Specifically, can we efficiently achieve accurate clustering under 
pairwise mean separation in the order of $\sqrt{k} (\sigma_i+\sigma_j)$?}
\end{quote}
As our main contribution, in this paper we study and {\em essentially resolve} 
this question.

\medskip

We emphasize that the heavy-tailed setting introduces a number of technical 
challenges that do not appear in the presence of strong concentration. 
For the sake of intuition, we explain below how $k$-PCA --- 
a standard spectral technique used in prior work --- provably fails in our setting. 

\paragraph{Failure of $k$-PCA}
One of the main standard techniques for clustering mixtures of separated
components is to perform $k$-PCA: 
find the top-$k$ dimensional subspace of the sample covariance, 
and show that with high probability, this subspace captures 
the span of the mixture component means.
However, this technique fails for bounded covariance distributions 
under our fine-grained separation assumption, even with infinitely many samples.
This can be demonstrated through a variant of the example in~\Cref{fig:kPCA}.
Consider the uniform (i.e.~equal weights) mixture 
with a component with unit covariance on a subspace $V$ at the origin, 
and two components with $0$-covariance, 
located at points $v+w$ and $v-w$ with $\|v\|_2 \gg 1$ and $\|w\|_2 \ll 1$.
Suppose also that $V$ is $\Omega(d)$-dimensional, 
and $V, v, w$ are orthogonal to each other.
Denoting the identity matrix in the subspace $V$ by $I_V$, 
the covariance of the full distribution is equal to 
$\frac{1}{3}I_V + \frac{2}{9} vv^\top + \frac{2}{3} ww^\top$.
Given that $\|v\|_2 \gg 1$ and $\|w\|_2 \ll 1$, 
the eigenvectors of this covariance are $v$, 
any $\Omega(d)$-dimensional basis of $V$, finally followed by $w$.
Thus, in order to have the direction $w$ in the subspace found by $k$-PCA, 
we might need as many as $k = \Omega(d)$ dimensions, 
which reduces the dimensionality only mildly.

\paragraph{Summary of contributions}
Our first goal focuses on uniform-weight mixture distributions, 
with the aim of clustering assuming only a pairwise separation of 
$C\cdot(\sigma_i + \sigma_j)\sqrt{k}$ between mixture components $P_i$ and $P_j$ 
satisfying $\Sigma_i \preceq \sigma^2_i \cdot I_d$ and $\Sigma_j \preceq \sigma^2_j \cdot I_d$, for some sufficiently large universal constant $C$.
We note that the individual standard deviations $\sigma_i$ are {\em unknown} to the algorithm.

For this setting, we give the first efficient algorithm (\Cref{alg:main_general}) achieving this guarantee in \Cref{thm:main_uniform}.
We point out that the recent work of~\cite{BKK22} also studies the heavy-tailed setting under a fine-grained separation assumption. However, they require separation which scales like $(\sigma_i+\sigma_j)\poly(k,\log n)$, for a large
degree polynomial\footnote{Their results do not explicit state the degree, but we believe it is at least degree-4 for $k$ according to their algorithm, as opposed to our optimal $\sqrt{k}$ dependence.}. More importantly,
they also require an additional ``no large sub-cluster assumption'' 
on the samples beyond bounded covariance ---
even for the uniform-weight mixture setting.

Our second, more general goal is to study the limits of clustering general-weight mixtures 
of bounded covariance distributions, under the same fine-grained pairwise separation assumption.
Perhaps surprisingly, we point out that it is 
information-theoretically impossible to achieve accurate clustering due to identifiability issues --- there can be multiple valid ground truths for the same mixture and there is no way to tell which one is the ``correct'' one ---
if the mixing weights are (highly) non-uniform.
Nonetheless, our main algorithm (\Cref{alg:main_general}) 
efficiently produces an accurate \emph{refinement} of the ground truth clustering 
(\Cref{thm:main_v1_intro}): informally, a clustering refinement is a list of not-too-small and disjoint subsets of samples such that there \emph{exists} a way to combine them 
 into a clustering close to the ground truth, and furthermore, these subset are themselves well-separated like the ground truth distribution.
This essentially amounts to the information-theoretically strongest possible guarantee in our setting.
We further show that, under a ``no large sub-cluster'' condition 
(\`{a} la~\cite{BKK22}), the same algorithm outputs exactly 
the correct $k$ clusters (up to some small fraction of misclassified points).

Finally, we remark that our algorithm is robust to a fraction of adversarial 
outliers that is comparable to the size of the smallest cluster.

\subsection{Our results}

Even in the special case of uniform-weight mixtures, no prior work can find an 
accurate clustering under a fine-grained separation assumption scaling with 
$\sigma_i+\sigma_j$ between components $P_i$ and $P_j$, even if we allow a sub-optimal $\poly(k)$ scaling.
Here we present our first result, solving both issues simultaneously.
\Cref{alg:main_general} finds an accurate clustering in polynomial time, assuming 
the optimal (up-to-constants) separation in the order of 
$(\sigma_i+\sigma_j)\sqrt{k}$, which is both fine-grained and has the information-theoretically optimal $\sqrt{k}$ dependence.

\begin{theorem}[Clustering uniform-weight bounded covariance mixtures]\label{thm:main_uniform}
    Let $C$ be a sufficiently large constant.
    Consider a uniform-weight mixture distribution $D= \sum_{i=1}^k \frac{1}{k} P_i$ with $k$ components on $\R^d$. %
    Suppose that $\alpha$ is a parameter in $[0.6/k, 1/k]$.
    Let $\mu_i$ and $\Sigma_i$ be the (unknown) mean and covariance of each $P_i$, and assume that $\Sigma_i \preceq \sigma_i^2 \cdot I_d$ (with $\sigma_i$ being unknown) and  $\| \mu_i - \mu_j \|_2 > C (\sigma_i + \sigma_j)/\sqrt{\alpha}$ for all $i \neq j$.
    
    Draw $n$ samples from $D$, and let $S_i$ be the samples from the $i^\text{th}$ mixture component.
    Further fix a failure probability $\delta > 0$.
    If $n > C (d\log(d) +\log(1/(\alpha\delta))) /\alpha^2$, then \Cref{alg:main_general} when given the samples, $\alpha,$ and $\delta$ as input, runs in polynomial time and outputs $k$ disjoint sets $\{B_i\}_{i \in k}$ so that with probability at least $1-\delta$ the following are true, up to a permutation of indices of the output sets:
    \begin{enumerate}[leftmargin=*]

        \item \label{it:concl5_unif} $|S_i \triangle B_i| \leq 0.045 n/k$ for every $i \in [k]$.
        \item \label{it:concl7_unif} The mean of $B_i$ is close to $S_i$: $\| \mu_{B_i} - \mu_i \|_2 = O(\sigma_i)$ for every $i \in [k]$.
    \end{enumerate} 
\end{theorem}

\Cref{alg:main_general} is given as input a minimum-weight parameter $\alpha \in [0.6/k, 1/k]$, and in polynomial-time it returns a list of exactly $k$ sets, $\{B_i\}$, such that, up to a permutation, each $B_i$ has a 95\% overlap with the set $S_i$ of samples drawn from the $i^\text{th}$ mixture component $P_i$ and that the mean $\mu_{B_i}$ of $B_i$ is indeed close to the mean $\mu_i$ of $P_i$, under the minimal assumption that the means of the $i^{th}$ and $j^{th}$ clusters are separated by at least a large constant multiple of $(\sigma_i + \sigma_j)/\sqrt{\alpha}$.
We note that \emph{i}) the 95\% overlap can be made an arbitrarily close constant to 1 by increasing the hidden constant in the separation assumption and adapting corresponding constants in the algorithm, and \emph{ii}) we do not require any ``no large sub-cluster condition'' in the uniform mixture setting.

We also stress that \Cref{alg:main_general} does not require knowing $k$ precisely, and only needs to know a lower bound $\alpha$ for $1/k$, which can be a (small) constant factor different.

We further remark that \Cref{it:concl5_unif} above lower bounds the size of the union of all the $B_i$s by $0.95n$, namely that at least 95\% of all the points are clustered and returned.
As noted above, the 95\% can be made into any constant arbitrarily close to 100\%, by increasing the constant $C$ in the separation assumption.
Alternatively, if we drop \Cref{it:concl7_unif} in the theorem statement, namely that the requirement that the mean of $B_i$ is indeed close to the mean $\mu_i$ of component $P_i$, then it is possible to return all the input samples in the output clustering.

Moreover, \Cref{thm:main_uniform} holds even for almost-uniform mixtures, 
where each mixing weight $w_i \in [0.9/k, 1.1/k]$, 
and if $\alpha \in [0.7/k,\min_i w_i]$.

\medskip

The situation of general (non-uniform) mixing weights is somewhat 
more complicated. For example, in the situation described below 
(also shown in \Cref{fig:identifiability}), 
even if we know the number $k$ of components and even if we have 
infinitely many samples, it is information-theoretically impossible to reliably 
achieve a 90\%-accurate clustering.

\paragraph{Example: Non-identifiability of general mixtures}
Consider a distribution consisting of 4 equal weight 0-covariance components, separated into 2 pairs.
Each pair is at unit distance, and the two pairs are separated by a large distance.
Suppose we are given that $k = 3$ and $\alpha = 1/4$, then there are two possible clusterings that disagree with each other by at least 25\% of the total mass: either group the first pair as a large cluster with weight $1/2$ and leave the second pair as two smaller clusters, or by symmetry we can group the second pair instead. It is allegorically impossible to determine which of these is the ``true'' ground truth clustering even with infinitely many samples from the mixture.

\begin{figure}
    \centering
    \includegraphics[clip=true,trim={0 2cm 0 1.3cm},keepaspectratio,height=1.5in]{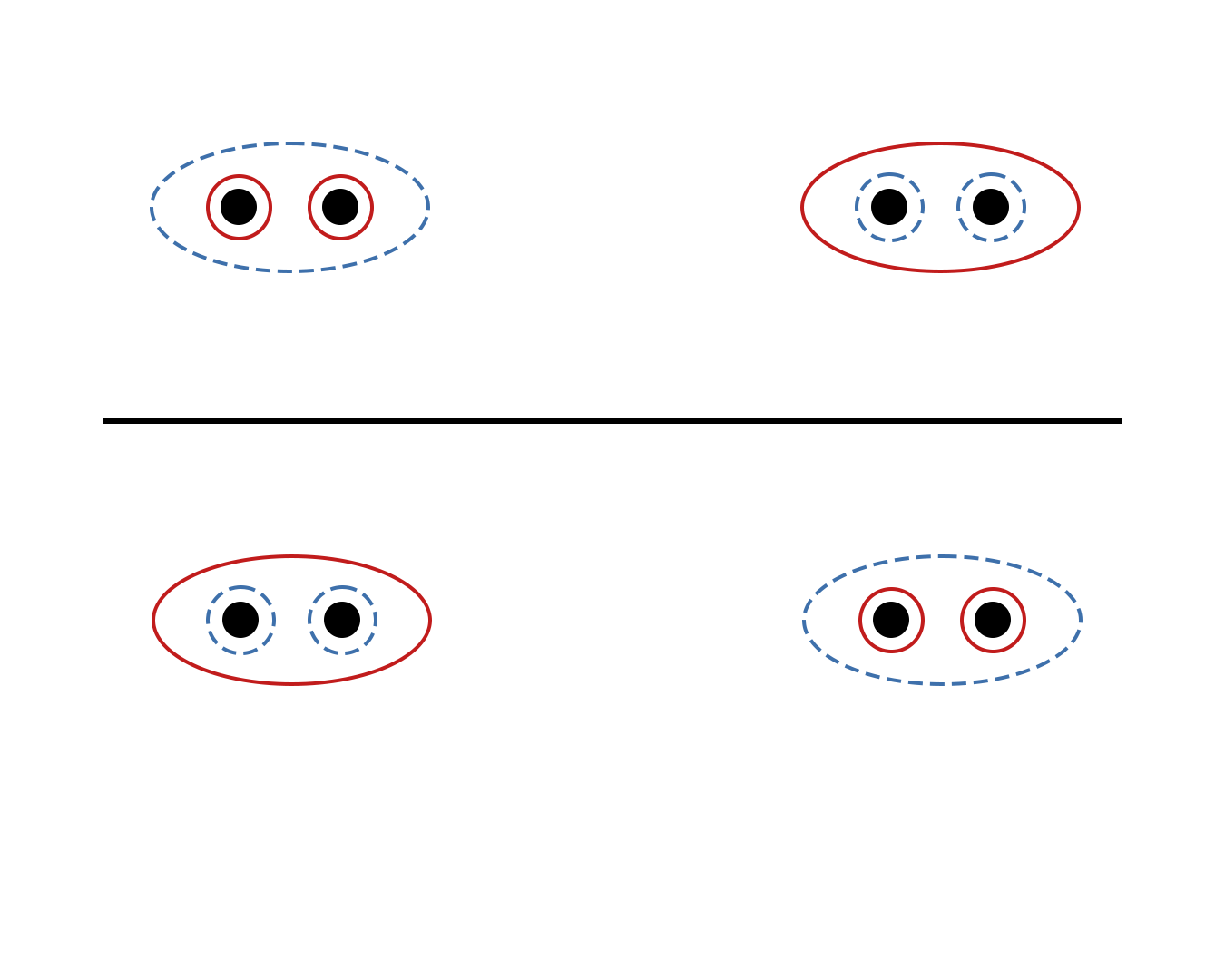}
    \caption{Two different ground truth clusterings for $k = 3$.}
    \label{fig:identifiability}
\end{figure}

\medskip

Given the above impossibility example, the question remains, what \emph{is} possible given only the mixing weight lower bound parameter $\alpha$, and a separation assumption of $C (\sigma_i+\sigma_j)/\sqrt{\alpha}$ between components $P_i$ and $P_j$?
The example highlights the core of the non-identifiability issue: an impossibility to identify which small subsets to group together.
Consequently, we can perhaps hope to compute all the information in the ground truth clustering \emph{except for} such subset grouping.
That is, we can try to identify only the small subsets themselves.
Motivated by this observation, we instead aim to return a \emph{refinement} of the clustering: we will return a list of $\ge k$ subsets (which we will call sub-clusters), each of size at least $\approx \alpha n$, such that there exists some way of grouping the subsets into $k$ larger clusters which then correspond to the ground truth mixture distribution.

For example, in the concrete setting of \Cref{fig:identifiability}, we could 
return the 4 small sub-clusters individually, which is a common refinement of the 
two possible clusterings shown in the figure. Furthermore, (as we will show) we 
can even guarantee that the returned subsets satisfy a pairwise separation 
guarantee similar to what we assume of our underlying mixture distribution.

Our main result (\Cref{thm:main_v1_intro}) of the paper shows that it is indeed 
possible to find an accurate refinement of the ground truth clustering, using 
$\tilde{O}(d/\alpha^2)$ samples and in polynomial time, with 
\Cref{alg:main_general}.
We define an accurate refinement below, as well as state a simplified version of 
our main theorem.

\begin{definition}[Accurate refinement of ground truth clustering]
    \label{def:refinement}
    Let $c>0$ be an absolute constant.
    Suppose we draw $n$ samples from the mixture distribution $D = \sum_{i = 1}^k w_iP_i$, where each $w_i \ge \alpha$ and each $P_i$ has mean $\mu_i$ and standard deviation $\sigma_i$.
    Let $S_i$ be the set of samples drawn from $P_i$.

    An accurate refinement of the clustering $S_i$ is a list of $m$ disjoint sets of samples $\{B_j\}_{j \in [m]}$ for some $m \in [k,O(1/\alpha)]$, such that:
    \begin{enumerate}[leftmargin=*]
        
        \item \label{it:Bsize_refinementdef} The sets $B_1,\ldots,B_{m}$ each have size  $|B_j| \geq 0.92 \alpha n$ for all $j \in [m]$.
        
        \item \label{it:refinement_refinementdef}The indices $[m]$ can be partitioned into $k$ sets $H_1, \ldots, H_k$, such that if $\mathcal{B}_i$ are defined as $\mathcal{B}_i:=\cup_{j \in H_i} B_j$, the following hold:
        \begin{enumerate}
            \item \label{it:SminusB_refinementdef} $|{S}_i \setminus \mathcal B_i| \leq 0.045|{S}_i|$ for every $i \in [k]$.
            \item \label{it:BminusS_refinementdef}$|\mathcal{B}_i \setminus {S}_i| \leq 0.03 \alpha n$ for every $i \in [k]$.
            \item \label{it:meangood_refinementdef} For any $i \in [k]$ and any $j \in H_i$ we have that $\| \mu_{B_j} - \mu_i \|_2 \leq c\, \sigma_i\sqrt{|{S}_i|/|B_j|}$.
            \item For any pair $j\neq j'$ we have that $\|\mu_{B_j} - \mu_{B_{j'}} \|_2 > 100 \, c\, (\sigma_{B_{j}} + \sigma_{B_{j'}})/\sqrt{\alpha}$, where $\sigma_{B_j}$ is the maximum standard deviation of $B_j$.\label{it:guarantee_sep_refinementdef}
        \end{enumerate}
        \item \label{it:classifiedpoints_refinementdef}As a consequence of \Cref{it:SminusB_refinementdef} we have that $|\cup_{j \in [m]} B_j| \geq 0.95 n$, namely that 95\% of the input points are classified into the output sets.
    \end{enumerate} 
\end{definition}

\Cref{it:Bsize_refinementdef} above says that each returned set 
must have size at least $\approx \alpha n$, 
given that each mixture component is supposed to have weight at least $\alpha$.
\Cref{it:refinement_refinementdef} captures the core idea of a refinement: 
there exists some way to combining the returned sets into sets $\cB_1,\ldots,\cB_k$, each corresponding to a mixture component 
$P_1,\ldots,P_k$, with the following guarantees.
\Cref{it:SminusB_refinementdef,it:BminusS_refinementdef} say that the symmetric 
difference between the samples $S_i$ drawn from component $P_i$ and the set 
$\cB_i$ is small.
\Cref{it:meangood_refinementdef} says that each output set $B_j$ must be close to 
the true mean of its corresponding mixture component $P_i$, with error scaling 
with $\sigma_i$ as well as $\sqrt{|S_i|/|B_j|}$ --- the larger $B_j$ is, containing 
more samples in $S_i$, the closer $\mu_{B_j}$ should be to $\mu_i$.
\Cref{it:guarantee_sep_refinementdef} says that the returned subsets $\{B_j\}$ 
must themselves satisfy a mean separation akin to the one satisfied by the mixture 
components, up to a constant factor loss.
Lastly, \Cref{it:classifiedpoints_refinementdef} guarantees that at least $95\%$ 
of the samples are indeed classified and returned in one of the output sets.

\begin{remark}\label{remark:intersection}
    The guarantees of \Cref{def:refinement} imply that for every output set $B_j$ there exists a true cluster $S_i$ such that $|B_j \cap S_i| = |B_j| - |B_j \setminus S_i| \geq |B_j| - |\mathcal B_i \setminus S_i| \geq|B_j| - 0.03\alpha n \geq (1-0.03/0.92)|B_j|\geq 0.967|B_j|$, i.e.~more than $96\%$ of the points in the output set come from the true cluster $S_i$.
\end{remark}

\begin{theorem}[Simplified version of \Cref{thm:main_v1_general}]
    \label{thm:main_v1_intro}
    Consider a mixture distribution on $\R^d$, $D= \sum_{i=1}^k w_i P_i$ with unknown positive weights $w_i \geq \alpha$ for some known parameter $\alpha \in (0,1)$.
    Let $\mu_i$ and $\Sigma_i$ be the (unknown) mean and covariance for each $P_i$, and assume that $\Sigma_i \preceq \sigma_i^2 \cdot  I_d$ for all $i \in [k]$ (with $\sigma_i$ being unknown) and  $\| \mu_i - \mu_j \|_2 > C (\sigma_i + \sigma_j)/\sqrt{\alpha}$ for every $i \neq j$, for a sufficiently large constant $C$.

    There is an algorithm (\Cref{alg:main_general}) which, when given $\alpha$ and $n$ independent samples from $D$ for $n$ at least a sufficiently large multiple of $ (d \log d + \log(1/(\alpha\delta)))/\alpha^2$, runs in polynomial time and with probability at least $1-\delta$ (over the randomness of both the samples and the algorithm),  outputs an accurate refinement clustering of these samples in the sense of \Cref{def:refinement}.
\end{theorem}

As in \Cref{thm:main_uniform}, for \Cref{def:refinement} to be satisfied by the 
algorithm output, we can make the constant 0.92 in \Cref{it:Bsize_refinementdef} 
arbitrarily close to $1$, and the constants in 
\Cref{it:SminusB_refinementdef,it:BminusS_refinementdef} arbitrarily close to 0, 
if we increased the constant in the mean separation assumption in 
\Cref{thm:main_v1_intro}.

We also remark that the same algorithm 
(\Cref{alg:main_general}) can even tolerate adversarial 
corruption in an $\Omega(\alpha)$-fraction of the samples.
See the full theorem, \Cref{thm:main_v1_general}, for the 
detailed statement.

\paragraph{Clustering under ``no large sub-clusters''} We can further 
guarantee that \Cref{alg:main_general} returns only $k$ clusters (thereby 
corresponding exactly to the $k$ ground truth components), if we also assume a ``no 
large sub-cluster'' condition \`{a} la~\cite{BKK22}, stated in \Cref{def:NLSC}.
Informally, the condition says that for any large subset $S'$ of samples $S_i$ 
drawn from the $i^\text{th}$ mixture component, the standard deviation 
$\sigma_{S'}$ of $S'$ is comparable to $\sigma_{S_i}$.
This is intuitively the contrapositive of not having any large sub-clusters: a 
large sub-cluster can be understood as a large subset that is separated from the 
rest of the clusters, meaning that it has a substantially smaller covariance.
Our condition below is qualitatively the same condition as that of~\cite{BKK22}, but with a stronger quantitative 
requirement on the parameters of a sub-cluster.
In \Cref{sec:NLSC_comparison}, we show that such a stronger condition is 
information-theoretically necessary, due to our much weaker (and optimal) mixture 
separation assumption.
Afterwards, in \Cref{sec:NLSC_proof}, we also show (see \Cref{cor:NLSC_informal}, 
an informal version of \Cref{cor:NLSC}), that if the condition is satisfied, then 
there can only be one possible ground truth (i.e., there are no identifiability issues anymore) and thus
\Cref{alg:main_general} indeed returns only one output set per real mixture component.

\begin{restatable}[NLSC condition]{definition}{NLSCdef}
\label{def:NLSC}
    We say that the disjoint sets $S_1,\ldots,S_k$ of total size $n$ satisfy the ``No Large Sub-Cluster'' condition with parameter $\alpha$ if 
    for any cluster $S_i$ and any subset $S' \subset S_i$ with $|S'| \geq 0.8 \alpha n$, it holds that $\sigma_{S'} \geq 0.1 \sigma_{S_i}$, where $\sigma_{S'}$ is the square root of the largest eigenvalue of the covariance matrix of $S'$.
\end{restatable}

\begin{corollary}[Informal version of \Cref{cor:NLSC}]
\label{cor:NLSC_informal}
    If the samples $S_i$ from the $i^\text{th}$ mixture component jointly satisfy the NLSC assumption with parameter $\alpha$ across all $i\in [k]$, then \Cref{alg:main_general} returns exactly one sample set per mixture component (with high probability).
    As a consequence, if $B_i$ is the output set corresponding to the $i^\text{th}$ mixture component, then we have $\|\mu_{B_i} - \mu_i\| \le O(\sigma_i)$, just like in \Cref{thm:main_uniform}.
\end{corollary}

Later in \Cref{sec:NLSC_proof}, we also show that well-conditioned and high-dimensional log-concave distributions have samples that satisfy the NLSC condition with high probability.
We remark that the high-dimensionality assumption is necessary: the thin-shell behavior of log-concave distributions in high dimensions is critical to satisfy our NLSC condition.

\begin{proposition}[Informal version of \Cref{prop:log-concave_NLSC}]
A sample of size $\tilde{O}(d/\alpha^2)$ drawn from a well-conditioned and high-dimensional log-concave distribution satisfies the NLSC condition (\Cref{def:NLSC}) with high probability.
\end{proposition}

Before moving on to an overview of our algorithmic techniques, we emphasize that, even though we presented multiple results in multiple settings (uniform vs general weight mixtures, with and without the NLSC condition), they all apply to the same algorithm without needing any changes even in the hardcoded constants.
\Cref{alg:main_general} does not need any knowledge of whether any of the conditions hold; it achieves the corresponding results automatically whenever the corresponding assumptions are satisfied.

\subsection{Technical overview}
\label{sec:overview}

In this section, we give an overview of the components and techniques used in \Cref{alg:main_general}, our main algorithm.

Since the mixture component means are assumed to be well-separated, \Cref{alg:main_general} works by finding a list of candidate mean vectors, each of which is close to a mixture component, with the entire list ``covering'' all the components.
Once we have such a list, it suffices to consider 
the Voronoi partition of the samples; 
that is, to assign each point to the cluster 
of the closest candidate mean.
The mean separation assumption, along with the concentration 
of bounded covariance distributions, 
guarantee that such a Voronoi partition 
will be close to a refinement of the ground truth clustering.

The high-level idea of finding such a list of candidate mean vectors is to first 
generate a much larger (but still polynomially-sized) list which potentially contains 
candidates that are far from all mixture components, and then prune all the invalid 
candidate means out of the list.
The first part is relatively straightforward, since there are standard list-decodable 
mean estimation algorithms for bounded-covariance distributions 
(e.g.~\cite{diakonikolas2021list}).
The only minor complication is that, for these algorithms to return 
means with tight error guarantees, they need good upper bounds 
on the standard deviation of each mixture component.
We thus first generate a list of possible standard deviations $\hat{s}$ 
(\Cref{cl:list_covariances}), and for each $\hat{s}$, 
run the list-decodable mean estimation algorithm.
After this step, we have a list of candidate means such that, 
for each mixture component, there is at least one candidate mean close to it.

The next step is at the heart of our algorithm: to prune candidate means that are not 
sufficiently close to any mixture component 
(with distance threshold scaling with the 
standard deviation of the mixture component).
A natural way to do this would be to test each candidate mean by trying to find its 
corresponding cluster and seeing if that exists.
In particular, given a candidate mean $\muhat$ and candidate standard deviation 
$\hat{s}$, we would like to find a subset of at least an $\approx \alpha$-fraction of 
the samples whose covariance matrix is bounded by $O(\hat{s}^2)\cdot I_d$ and whose 
mean is within $O(\hat{s}/\sqrt{\alpha})$ of $\muhat$.
If we can find this, it suggests that the cluster we are looking for actually exists.

Unfortunately, the natural approach
of finding such a cluster is computationally hard,
so we need to find appropriate relaxations to make it tractable.
Immediately, to avoid computational hardness from integrality issues, we begin by 
allowing a weighted subset rather than an actual subset, which concretely is to find 
weights $w_i \in [0,1]$ over each sample $x_i$, such that $\sum_i w_i$ is at least 
$\approx \alpha n$.
This nearly turns our problem into a convex program.
In particular, if we knew the mean of the cluster exactly, the covariance would be a 
linear function of $\{w_i\}_i$, making it a convex program.
However, as we do not know the real mean, the covariance matrix centered at 
$\mu_w$ --- the mean of the weighted cluster defined by $\{w_i\}_i$ --- is no longer 
linear in $\{w_i\}_i$, and the constraint bounding its operator norm is no longer a 
convex constraint.
So, instead, we compute the second moment matrix of $\{w_i\}_i$ centered at the 
candidate mean $\muhat$ 
(i.e.~proportional to $\sum_i w_i(x_i-\muhat)(x_i-\muhat)^\top$).
This gives us a convex program, but unfortunately one that might not be satisfiable 
even by a correct cluster $C$ whose mean is indeed $O(\hat{s}/\sqrt{\alpha})$ close to 
the candidate mean $\muhat$: the second moment matrix of $C$ would actually be 
$\Cov(C) + (\muhat - \mu_C)(\muhat - \mu_C)^\top$, and the latter term might 
contribute to an eigenvector of size as large as $\Omega(\hat{s}^2/\alpha)$, which is 
too large when $\alpha$ is small.
We must therefore further relax our convex program.
Instead of finding $\{w_i\}$ whose second moment matrix centered at $\muhat$ has operator norm bounded by $O(\hat{s}^2) \cdot I_d$, we constrain its $O(1/\alpha)$-Ky-Fan norm by $O(\hat{s}^2/\alpha)$.
This new, final program (Program \eqref{program:main} in \Cref{sec:program}) is now both convex and satisfiable by a true cluster.

The next obstacle, however, is that a solution $\{w_i\}_i$ to the above convex program 
might not actually correspond to a true cluster or mixture component.
In particular, if there are other clusters with standard deviation much smaller than 
$\hat{s}$, we might have found a solution that shares bits and pieces of these smaller clusters.
This problem can only occur though if there are other clusters with standard deviation 
smaller than $\hat{s}$, but which are close to $\muhat$.
Thus, we can avoid it by searching for clusters in \emph{increasing} order of 
$\hat{s}$ and then throwing away any $\muhat$ that is within 
$O(\hat{s}/\sqrt{\alpha})$ of some previously un-pruned candidate mean.
Formally, \Cref{lemma:main} shows that if $\muhat$ is far from all clusters with 
standard deviation smaller than $\hat{s}$, and if a solution to 
Program~\eqref{program:main} exists for the pair $(\muhat, \hat{s})$, then the found 
solution must overlap substantially with a true cluster.
An induction applying \Cref{lemma:main} repeatedly then shows that, after this 
pruning, all candidate means must be close to some true cluster, and that all clusters 
have candidate means close to them.

As discussed at the beginning of the section, we can now consider the Voronoi partition of the samples based on the candidate means.
A few issues remain, that this partition does not satisfy the guarantees of \Cref{thm:main_v1_intro}.
First, if there are too many candidate means remaining at this stage, a cluster in the Voronoi partition might be too small in size (\Cref{sec:size-pruning}).
To solve this, we repeatedly remove candidate means whose Voronoi cluster is too small, 
noting that \emph{i}) this never decreases the cluster size of un-removed candidate means, and \emph{ii}) by the separation of the mixture components, 
we will never accidentally remove all candidate means close to any true cluster.
Second, due to heavy-tailed noise and adversarial corruption, 
even for the Voronoi clusters that overlap well with true clusters, their means might be very far from the candidate means we started out with.
We fix this using the standard filtering technique in robust statistics, removing at most 2\% of the samples in each Voronoi cluster.
Lastly, we need to guarantee that the returned clusters also satisfy (up to constant factors) the same separation assumption we have on our underlying mixture distribution (\Cref{sec:distance-pruning}).
We enforce this again by removing candidate means whenever we detect a pair of (filtered) clusters that are too close to each other.
Crucially, we carefully choose which corresponding candidate mean from the pair to remove, so that we never remove all the candidate means close to a true cluster.

\subsection{Related work}

Here we survey the most relevant prior work on clustering mixture models
and algorithmic robust statistics.

\paragraph{Mixture models} %
A long line of work in theoretical computer science and machine learning has focused on developing
efficient clustering methods for various mixture models (with mixtures of Gaussians being the prototypical example) 
under mean separation conditions; see, e.g.~\cite{Dasgupta:99,SanjeevK01, VempalaWang:02,AchlioptasMcSherry:05,KSV:05, kumar2010clustering,awasthi2012improved, CSV17, HL18-sos, KSS18-sos, diakonikolas2022clustering, BKK22}.

Early work~\cite{AchlioptasMcSherry:05} gave an 
efficient 
spectral algorithm for clustering mixtures of bounded 
covariance Gaussians that succeeds under mean 
separation 
$\Theta((\sigma_i+\sigma_j)/\sqrt{\alpha})$ between 
components $P_i$ and $P_j$, when the minimum mixing 
weight $\alpha$ is much smaller than $1/k$.
However, even for the special case of uniform-weight 
$k$-mixtures of Gaussians (and log-concave distributions), 
their result requires a separation of 
$(\sigma_i+\sigma_j) \Omega(k)$ --- instead of scaling 
with $\sqrt{k}$ --- and, in fact, also has additional spurious terms in the separation containing a logarithmic dependence on the sample complexity $n$.
It should be noted that the algorithm of 
\cite{AchlioptasMcSherry:05} built on an earlier algorithm 
developed in \cite{VempalaWang:02}, which only works for mixtures 
of spherical Gaussians.
They can cluster under the weaker mean 
separation condition which (roughly) scales as $(1/\alpha)^{1/4}$; their separation 
condition also has a mild logarithmic dependence on the ambient dimensionality $d$.
The works mentioned in this line all employ $k$-PCA as a core algorithmic 
technique; see the beginning of the introduction on why $k$-PCA fails 
in our heavy-tailed problem setting, under our fine-grained separation assumption.

\cite{awasthi2012improved} provided another spectral algorithm, designed to cluster mixtures of bounded covariance data.
Their algorithm is able to cluster under a separation of (roughly) $\Omega(k) (\max_i \sigma_i)$.
Their specific separation assumption can in fact be smaller than $\Omega(k) (\max_i \sigma_i)$ in certain instances, but the bound is not improvable to $o(k) (\max_i \sigma_i)$ in the worst case, contrasting the $\sqrt{k}$ dependence we achieve.
More importantly, their separation 
condition between $\mu_i, \mu_j$ scales with the maximum 
standard deviation $\max_i \sigma_i$, as opposed to the fine-grained pair-dependent sum $\sigma_i+\sigma_j$ achieved by our algorithm.

Recently,~\cite{diakonikolas2022clustering} gave an 
almost linear-time clustering algorithm for mixtures of bounded  covariance distributions.
Their techniques inherently also cluster only under a $\max_i \sigma_i$ separation 
for the following reason: their algorithm runs a list-decodable mean estimation 
routine once (with the goal to list-decode the mean of a distribution with 
covariance $\Sigma \preceq (\max_i \sigma_i^2) \cdot I_d$)
to generate a list of $O(1/\alpha)$ possible candidate cluster means. 
It then uses a coarse distance-based method to prune the means 
down to exactly $k$ of them. As a result, their approach only works under 
a uniform separation between all pairs of components.

Another recent work~\cite{BKK22} also studied efficient 
clustering of mixtures of bounded covariance 
distributions, achieving a mean separation (between 
$\mu_i, \mu_j$) scaling with  $\sigma_i+\sigma_j$. 
However, their separation assumption has a highly sub-optimal $\poly(1/\alpha)$ dependence, as well as an unnecessary logarithmic dependence on the sample complexity $n$.
More importantly, their clustering algorithm inherently requires an additional structural condition on the components (which they term ``no large sub-cluster'' condition) beyond just bounded covariance, even for the special case of uniform-weight mixtures.

A related line of work has obtained clustering algorithms 
with significantly improved separation using more 
sophisticated algorithmic tools; 
see, e.g.~\cite{DiakonikolasKS18,HL18-sos, KSS18-sos, DiakonikolasK20, Liu022}. These works apply for families of 
distributions with controlled higher moments (e.g.~sub-Gaussians), and in particular have no implication for the 
bounded covariance setting studied here.

Beyond clustering, a line of research developed efficient 
algorithms for learning mixtures of Gaussians, even
in the presence of a constant fraction of corruptions; see, e.g.~\cite{MoitraValiant:10, BelkinSinha:10, BakshiDHKKK20, kane2021robust, liu2020settling, bakshi2020robustly}. 
The aforementioned algorithms make essential use of the assumption that the mixture components are Gaussian.

\paragraph{Robust statistics and list-decodable learning} 
Our paper is also related to the field of algorithmic robust statistics
in high dimensions. Early work in the statistics community~\cite{Huber64, Tukey75} solidified the statistical foundations of this field.
Unfortunately, the underlying estimators lead to exponential time 
algorithms. A line of work in computer science, starting with 
~\cite{DKKLMS16, LaiRV16}, developed polynomial-time algorithms for a wide 
range of robust high-dimensional estimation tasks. The reader is referred  to 
the recent book~\cite{diakonikolas2023algorithmic} for an overview.

The list-decodable learning setting that we leverage in this work 
was defined, in a somewhat different context, in~\cite{BBV08}. \cite{CSV17} gave the first polynomial-time algorithm for the task of list-decodable mean estimation 
under a bounded covariance assumption. Specifically, if the clean data has 
covariance bounded by the identity, their achieved error guarantee is 
$\tilde{O}(1/\sqrt{\alpha})$. This error bound was slightly refined 
in~\cite{CherapanamjeriMY20} to $O(1/\sqrt{\alpha})$ 
with an asymptotically faster algorithm; 
a matching information-theoretic lower bound of $\Omega(1/\sqrt{\alpha})$ 
was shown in~\cite{DKS18-list}. 
We note that \cite{CSV17} also obtains a corollary for clustering mixtures, but their method requires sub-Gaussian components, and it only outputs a clustering refinement with $O(1/\alpha)$ subsets.
Finally,~\cite{diakonikolas2022clustering}, building on~\cite{DiakonikolasKK20,DiakonikolasKKLT20}, 
developed an almost-linear time algorithm for this task; 
in fact, they built their clustering result for mixtures of 
bounded covariance distributions via a reduction 
to list-decodable mean estimation.

In this work, we also use list-decodable mean estimation 
as a blackbox (\Cref{fact:list-decoding-mean} in \Cref{sec:prelim}).
An important difference compared to prior work is that 
our processing of the candidate means is substantially more involved, 
which is required due to our fine-grained separation assumption.

Finally, we point out other work which developed efficient list-decodable 
mean estimators with significantly improved error guarantees 
under much stronger distributional assumptions~\cite{DKS18-list, KSS18-sos, DiakonikolasKKP22}.

\subsection{Organization}

\Cref{sec:prelim} gives basic notations and facts that we use in the rest of the paper.
\Cref{sec:main_result} states our main algorithm (\Cref{alg:main_general}) as well as the full version of our main result (\Cref{thm:main_v1_general}).
\Cref{sec:program,sec:size-pruning,sec:distance-pruning} analyzes the three main steps of the algorithm.
\Cref{sec:wrapup} uses the guarantees from the prior three sections to prove our main result.
Finally, \Cref{sec:NLSC} discusses the implications of the no large sub-cluster condition in our problem setting.

\section{Preliminaries}
\label{sec:prelim}

In this section, we state useful notations and facts that the rest of the paper depends on.

\subsection{Notation}
For a vector $v$, we let $\|v\|_2$ denote its $\ell_2$-norm. 
We use $I_d$ to denote the $d \times d$ identity matrix; We will drop the subscript when it is clear from the context.
For a matrix $A$, we use $\|A\|_\fr$ and $\|A\|_{\op}$ to denote the Frobenius and spectral (or operator) norms, respectively. We use $\|A\|_{(k)}$ to denote the Ky-Fan norm which is defined as $\|A\|_{(k)} = \sum_{j=1}^k s_j(A)$, where $s_j(A)$ for $j=1,\ldots,k$ are the first $k$ singular values of $A$.
If $V$ is a subspace, we denote by $\mathrm{Proj}_V$ its the orthogonal projection matrix.

We use $X \sim D$ to denote that a random variable $X$ is distributed according to the distribution $D$.  We use $\cN(\mu, \Sigma)$ for the Gaussian distribution with mean $\mu$ and covariance matrix $ \Sigma$. 
For a set $S$, 
we use $X \sim S$ to denote that $X$ is distributed uniformly at random from $S$.
When $S$ is a set of points in $\R^d$, we will use the shorter notation $\mu_S := \E_{X \sim S}[X], \Cov(S) := \E_{X \sim S}[(X-\mu_S)(X-\mu_S)^\top]$, and $\sigma_S :=  \sqrt{\| \Cov(S)\|_{\op}}$.

We use $a\lesssim b$ to denote that there exists an absolute universal constant $C>0$ (independent of the variables or parameters on which $a$ and $b$ depend) such that $a \leq C  b$. We use $a \gg b$ to denote that $\alpha > C b$ for a sufficiently large absolute constant $C$.

\subsection{Deterministic conditions and useful facts}
\label{sec:stability}

\paragraph{Stability condition} Our algorithm will succeed if the following condition is 
satisfied for the samples of each true cluster. The condition, referred to as ``stability'', is 
standard in algorithmic robust statistics. Intuitively, it requires 
that any large subset of the dataset has mean and covariance that do not shift significantly. 
We provide the definition below. In the fact that follows, 
we state that large sets of points from bounded covariance distributions 
indeed satisfy the stability condition with high probability.

\begin{definition}[Stability condition]\label{def:stability}
    For $C>0$ and $\eps \in (0,1/2)$,
    a multiset $S$ of $m$ points $x_1,\ldots, x_m$ in $\R^d$  is called $(C,\eps)$-stable with respect to $\mu \in \R^d$ and $\sigma \in \R^{+}$ if, for any weights $w_1,\ldots,w_m \in [0,1]$ with $\sum_{x_i \in S} w_i \geq (1-\eps)m$ it holds:
    \begin{itemize}
        \item $\left\| \frac{1}{\sum_{x_i \in S} w_i}\sum_{x_i \in S} w_i x_i -\mu   \right\|_2 \leq C\sigma\sqrt{\eps}$
        \item $\Sigma_{w,\mu} \preceq C^2 \sigma^2 \cdot I_d$,
    \end{itemize}
    where 
    $\Sigma_{w,\mu} := \frac{1}{\sum_{x_i \in S} w_i}\sum_{x_i \in S} w_i(x_i-\mu )(x_i-\mu)^\top$.
\end{definition}

\begin{fact}[Sample complexity of stability \cite{DiaKP20fact:filtering}]\label{fact:stability-lemma}
    Let $S$ be a set of $m$ points drawn i.i.d.\ from a distribution on $\R^d$ with mean $\mu$ and covariance $\Sigma \preceq \sigma^2 \cdot I_d$.
    If $m \gg (d\log(d) + \log(1/\delta))/\min\{\eps, \alpha \}$ then, with probability $1-\delta$, there exists a $(1-0.001\alpha)m$-sized subset $S'$ of $S$ that is $(100,\eps)$-stable with respect to $\mu \in \R^d$ and $\sigma$.
\end{fact}

\paragraph{Facts from robust statistics}

We record the following facts that will be useful later on. First, we recall in \Cref{fact:filtering} a stability-based filtering algorithm that, given any stable set of samples with bounded covariance and with $4\%$ of its points arbitrarily corrupted, removes $4\%$ of the points in a way that the resulting output set is guaranteed to have bounded covariance and mean close to the true one.

\begin{definition}[Strong contamination model]\label{def:strongadv}
Given a parameter $0<\eps<1/2$, 
the strong adversary operates as follows: The algorithm specifies 
a set of $n$ samples, then the adversary inspects the samples,
removes up to $\eps n$ of them and replaces them with arbitrary points. 
The resulting set is given as input to the learning algorithm. 
We call a set $\eps$-corrupted if it has been generated by the above process.
\end{definition}

\begin{fact}[Filtering; see, e.g.~\cite{diakonikolas2023algorithmic}]\label{fact:filtering}
There exists an algorithm for which the following is true:
Let $\delta \in (0,1)$ be a parameter. Let $S$ be a set of points in $\R^d$ that is $(C,\eps)$-stable with respect to $\mu$ and $\sigma$ for some $C>0$ and $\eps \leq 0.04$. Let $T$ be an $\eps$-corrupted version of $S$ (cf. \Cref{def:strongadv})
and assume $|T| \gg \log(1/\delta)$.
Then the algorithm having as input any set $T$ of the above form and $\delta$ terminates in time $\poly(|T|,d)$ and returns a subset $T' \subseteq T$ such that, with probability at least $1-\delta$, the following hold:
\begin{itemize}
    \item $|T'| \geq (1-\eps) |T| $.
    \item $\| \mu_{T'} - \mu \|_2 \leq 10 C \sigma \sqrt{\eps}$.
    \item $\Sigma_{T',\mu} \preceq 10 C^2 \sigma^2 \cdot I_d$.
\end{itemize}
\end{fact}

The following fact states that taking subsets of a set $S$ with bounded covariance does not shifts the mean significantly. 
This (or its contrapositive version) will be used in a lot of the core arguments. 
In particular, one corollary of this fact is \Cref{lem:stability_small_sets}, stating that subsets of  stable sets are also stable with worse parameters. This will be useful for applying the aforementioned filtering algorithm at the very last step of our main algorithm to ensure that the final clusters have means and covariances that are close to what they should be.
For completeness, we provide a proof of \Cref{lem:stability_small_sets} in \Cref{appendix:prelims}.

\begin{fact}\label{fact:513}
         Let $S$ be a multiset, and denote by $\mu_S,\Sigma_S$  the mean vector and covariance matrix of the uniform distribution on $S$. If $S$ satisfies $\Sigma_S \preceq \sigma^2 \cdot I_d$ and $w_x \in [0,1]$ are weights for the points $x \in S$ that satisfy $\sum_{x \in S} w_x \geq \alpha |S|$, then we have that 
         \begin{align*}
             \left\| \frac{\sum_{x \in S}w_x x}{\sum_{x \in S}w_x} - \mu_S \right\|_2 \leq \frac{\sigma}{\sqrt{\alpha}} \;.
         \end{align*}
     \end{fact}

\begin{restatable}{lemma}{STABILITYSUBSETS}\label{lem:stability_small_sets}
    Let $S$ be a set of points that is $(C,\eps)$-stable with respect to $\mu$ and $\sigma$ for some $C\geq 1$ and $\eps<1/2$. Then, any subset $S' \subseteq S$ with $|S'| \geq \alpha |S|$ is $(1.23 C/\sqrt{0.04\alpha},0.04)$-stable with respect to $\mu$ and $\sigma$.
\end{restatable}

We finally state in  \Cref{cl:list_covariances,fact:list-decoding-mean} the subroutines that we will use for creating a list of candidate covariances and means of the true clusters.
We defer the proof of \Cref{cl:list_covariances} to \Cref{appendix:prelims}. The algorithm consists of simply returning a list with  all the values starting from $\|x-y\|_2$ down to  $\|x-y\|_2/(2|S|^2)$ in multiples of $\sqrt{2}$, for all pairs of points $x,y$. By the definition of the covariance matrix as $\Cov(S) = \frac{1}{2 |S|^2}\sum_{x,y \in S}(x-y)(x-y)^\top$, one of these quantities should be within a factor of two from $\|\Cov(S)\|_\op$.

\begin{restatable}{proposition}{LISTCOV}\label{cl:list_covariances}
    Let $T$ be a set of $m$ points in $\R^d$. There is a $\poly(m,d)$-time algorithm that outputs a list of size $O(m^2\log(m))$ that for any $S \subseteq T$ contains an estimate $\hat {s}$ such that $\| \Cov(S) \|_{\op}  \leq \hat {s}^2 \leq 2\| \Cov(S) \|_{\op}$.
\end{restatable}

\begin{fact}[List-decodable mean estimation; see, e.g.~\cite{diakonikolas2021list}]\label{fact:list-decoding-mean}
Let $S$ be a multi-set in $\R^d$ that satisfies 
$\frac{1}{m}\sum_{x \in S}(x-\mu)(x-\mu)^\top \preceq \sigma^2 \cdot I_d$ 
for some $\mu \in \R^d$ and $\sigma>0$, 
and $T$ be another multi-set in $\R^d$ such that $S \subseteq T$ and $|S| \geq \alpha |T|$. 
There exists an algorithm and absolute constant $C>1$, that on any input $T$ 
of the aforementioned form and the standard deviation parameter $\sigma$, 
the algorithm runs in polynomial time and returns a $O(1/\alpha)$-sized 
list of vectors that contains at least one vector $\hat{\mu}$ 
such that $\|\hat{\mu} - \mu \|_2 \leq C\sigma/\sqrt{\alpha}$.
\end{fact}

\section{Main algorithm and result}
\label{sec:main_result}

We present our main algorithm in the paper, \Cref{alg:main_general}, which follows the outline described in \Cref{sec:overview}.
Lines~\ref{line:listdec} and \ref{line:mean_list_dec} first generates a list of plausible component means and standard deviations.
Then, Line~\ref{line:main_loop_gen} is responsible for pruning the list such that every remaining candidate mean is indeed close to a true component. This is useful because the Voronoi partition of the samples based on such a list is an accurate refinement of the ground truth clustering. 
Lines~\ref{line:sizebased} and~\ref{line:distancebased} further prune the list, to ensure that the returned refinement have subsets that are not too small (at least $\approx \alpha n$ in size) and that they are well-separated.
Finally, Line~\ref{line:output} returns filtered versions of the final Voronoi partition, in order to filter out adversarial and heavy-tailed outliers, to make sure that the mean of each returned subset is reasonably close to its corresponding mixture component.

\begin{algorithm}[h!]
\caption{Clustering algorithm}\label{alg:main_general}
\textbf{Input}: Parameter $\alpha \in (0,1)$, and multi-set $T$ of $n$ points in $\R^d$ for which there exists a ground truth clustering $S_1,\ldots,S_k$ according to the assumptions of \Cref{thm:main_v2_general}.\\
\textbf{Output}: Disjoint subsets of $T$ that form an accurate refinement (cf.~\Cref{def:refinement}) of the ground truth clustering.
\begin{enumerate}[leftmargin=*]
    \item \label{line:listdec} Generate a list $L_{\mathrm{stdev}}$ of candidate standard deviations using the algorithm from \Cref{cl:list_covariances}.

    \item \label{line:mean_list_dec} Generate a list of candidate means, $L_{\mathrm{mean}}$, by applying the list-decoding algorithm of \Cref{fact:list-decoding-mean} for each candidate $s$ in the list $L_{\mathrm{stdev}}$, and appending the output of each run to $L_{\mathrm{mean}}$.

    \item Initialize $L \gets \emptyset$.
    
    \item \label{line:main_loop_gen} For every $s \in L_{\mathrm{stdev}}$ in increasing order of $s$:

    \begin{enumerate}
        \item \label{line:sub_loop_gen} For every $\mu \in L_{\mathrm{mean}}$: 
        \begin{enumerate}
            \item \label{line:mean_check_gen} If $\| \mu - \hat{\mu} \|_2 > 99 C  s/\sqrt{\alpha}  $ for all $\hat{\mu} \in L$, decide the satisfiability of the convex program defined in \Cref{program:main} in \Cref{sec:program}.
            \item If satisfiable, add $\mu$ to the list $L$. \label{line:add_gen}
            
        \end{enumerate}
    \end{enumerate}

        \item \label{line:sizebased} $L' \gets \sizebasedpruning(L,T,\alpha)$.   \Comment{cf.\ \Cref{alg:pruning}}
        \item \label{line:distancebased} $L'' \gets \distancebasedpruning(L',T,\alpha)$. \Comment{cf.\ \Cref{alg:final_pruning}}

    \item \label{line:output} Output $\filteredvoronoi(L'',T)$.
     
\end{enumerate}
\end{algorithm}

We will now state the full version of our main theorem (\Cref{thm:main_v1_general}).
As discussed in the introduction, our algorithm can also handle a small amount of adversarial corruption in the samples.
Recall the ``Strong Contamination Model'' from \Cref{def:strongadv}, commonly used in the robust statistics literature, capturing the powerful adversary that our algorithm can handle. In that model, a computationally unbounded adversary can inspect and  edit a small fraction of the input points however it wants.

We now give the version of our main result (\Cref{thm:main_v1_general}) that works under this adversarial corruption. 
The statement says that \Cref{alg:main_general} outputs an accurate refinement of the ground truth 
clustering of the samples: a list of sets $\{B_j\}_{j \in [m]}$ 
for some $m \in [k, O(1/\alpha)]$, each of which has size at least $0.92\alpha n$, 
such that the sets are 90\% close to a refinement of the ground truth clustering.
We also ensure that the output clusters also enjoy a mean separation guarantee 
that is qualitatively similar to the one at the distributional level (\Cref{it:concl8_new} below).
Furthermore, if the output set $B_j$ corresponds a subset of the samples $S_i$ 
drawn from component $i$, then the mean $\mu_{B_j}$ of $B_j$ is close to 
$\mu_i$ (\Cref{it:concl7_new}), by a distance bound that depends on the ratio $|S_i|/|B_j|$, 
namely that the larger the fraction that $B_j$ covers in $S_i$, the closer their means are.

\begin{theorem}[Main result, formal statement]\label{thm:main_v1_general}
    Consider a mixture distribution on $\R^d$, $D= \sum_{i=1}^k w_i P_i$ with unknown positive weights $w_i \geq \alpha$ for some known parameter $\alpha \in (0,1)$.
    Let $\mu_i$ and $\Sigma_i$ be the (unknown) mean and covariance for each $P_i$, and assume that $\Sigma_i \preceq \sigma_i^2 \cdot I_d$ for all $i \in [k]$ (with $\sigma_i$ being unknown) and  $\| \mu_i - \mu_j \|_2 > 591\, c^2 (\sigma_i + \sigma_j)/\sqrt{\alpha}$ for every $i \neq j$, for a sufficiently large constant $c$.
    
    Let a set $T_0$ of $n$ samples drawn from $D$ independently, and let ${S}_i$ be the samples from the $i^\text{th}$ mixture component.
    Let $T$ be any $0.01\alpha$-corruption of $T_0$ according to the model defined in \Cref{def:strongadv}.
    Further fix a failure probability $\delta \in (0,1)$.
    
    If $n \gg (d\log(d)+ \log(1/(\alpha\delta)))/\alpha^2$, then on input the set 
    $T$ and the parameter $\alpha$, with probability at least $1-\delta$ 
    (over the randomness of both the samples and the algorithm), 
    \Cref{alg:main_general} runs in time $\poly(nd/\alpha)$ 
    and outputs $m \leq 1/(0.92\alpha)$ disjoint sets 
    $\{B_j\}_{j \in [m]}$ such that:
    \begin{enumerate}[leftmargin=*]
        
        \item \label{it:concl1_new} The output sets $B_1,\ldots,B_{m}$ each have size  $|B_j| \geq 0.92 \alpha n$ for all $j \in [m]$.
        
        \item \label{it:concl3_new} The set of indices $[m]$ can be partitioned into $k$ subsets $H_1, \ldots, H_k$, such that if $\mathcal{B}_i$ are defined as $\mathcal{B}_i:=\cup_{j \in H_i} B_j$, the following hold:
        \begin{enumerate}
            \item \label{it:concl5_new} $|{S}_i \setminus \mathcal B_i| \leq 0.045|{S}_i|$ for every $i \in [k]$.
            \item \label{it:concl6_new} $|\mathcal{B}_i \setminus {S}_i| \leq 0.03 \alpha n$ for every $i \in [k]$.
            \item \label{it:concl7_new} For any $i \in [k]$ and any $j \in H_i$, we have that $\| \mu_{B_j} - \mu_i \|_2 \leq c\, \sigma_i\sqrt{|{S}_i|/|B_j|}$.
            \item \label{it:concl8_new} For any pair $j\neq j'$, 
            we have that $\|\mu_{B_j} - \mu_{B_{j'}} \|_2 > 366 \, c\, (\sigma_{B_{j}} + \sigma_{B_{j'}})/\sqrt{\alpha}$. \label{it:guarantee_sep_new}
        \end{enumerate}
        \item As a consequence of  \Cref{it:concl5_new}, 
        we have that $|\cup_{j \in [m]} B_j| \geq 0.95 n$, 
        namely that 95\% of the input points are classified into the output sets.
    \end{enumerate} 
\end{theorem}

Before we prove \Cref{thm:main_v1_general}, we first show \Cref{thm:main_uniform} concerning the special case of uniform-weight mixture distributions.
As we show below, \Cref{thm:main_uniform} is a direct consequence of \Cref{thm:main_v1_general}.

 \begin{proof}[Proof of \Cref{thm:main_uniform}]
     \Cref{thm:main_uniform} is a special case of \Cref{thm:main_v1_general}.
     It can be readily checked that all the assumptions of \Cref{thm:main_v1_general} are satisfied for $\alpha \in [0.6/k,1/k]$. 
     Moreover, the sizes $|S_i|$ have expected value $n/k$, and thus by the Chernoff-Hoefding bound 
     it must be the case that $0.999n/k \leq |S_i|\leq 1.001n/k$ with high probability.
     Since the sets $B_j$ ($j \in [m]$) mentioned in \Cref{thm:main_v1_general} are disjoint with sizes $|B_j| \geq 0.92 \alpha n > 0.552 n/k$ (\Cref{it:concl1_new} of the theorem statement) and their unions $\mathcal{B}_i$ corresponding to $i^\text{th}$ cluster satisfy
     $|\mathcal{B}_i \setminus S_i| \leq 0.03n/k$ (\Cref{it:concl6_new}), this means that 
     each $\cB_i$ has size $|\cB_i| \le 1.031n/k$ and thus every $\mathcal{B}_i$ 
     must consist of exactly one of the $B_j$'s.
     Thus, the algorithm outputs exactly $k$ sets $B_1,\ldots,B_k$, 
     where (up to a permutation of the labels) $B_i$ corresponds to the $i^\text{th}$ mixture component.
     Then, \Cref{it:concl5_new,it:concl6_new}  of \Cref{thm:main_v1_general} imply that 
     $|S_i \triangle B_i| \leq 0.044 \max(|S_i|,\alpha n) \leq 0.045 n/k$ since $\alpha \leq 1/k$ 
     and $|S_i|\leq 1.001n/k$. 
    \Cref{it:concl7_unif} of \Cref{thm:main_uniform} follows from \Cref{it:concl7_new} of \Cref{thm:main_v1_general} after noting that 
    $$|B_j| \geq |B_j \cap S_j| \geq |S_j| -|S_j \triangle B_j| \geq  0.999 n/k - 0.044 n/k = 0.955 n/k \geq (0.955/1.001)|S_j| \;.$$
    This completes the proof of \Cref{thm:main_uniform}.
     
 \end{proof}

It remains to analyze \Cref{alg:main_general}, which we do in \Cref{sec:program,sec:size-pruning,sec:distance-pruning,sec:wrapup}.
\Cref{sec:program} states and analyzes the convex program used in Line~\ref{line:main_loop_gen} of the algorithm, as well as the guarantees-by-induction right after Line~\ref{line:main_loop_gen} finishes.
\Cref{sec:size-pruning} gives \Cref{alg:pruning} used in Line~\ref{line:sizebased}, 
which ensures that every set in the Voronoi partition computed from the remaining candidate means is of size at least $\approx \alpha n$.
\Cref{sec:distance-pruning} gives \Cref{alg:final_pruning} used in Line~\ref{line:distancebased}, which in turn ensures that the Voronoi partition from the remaining means 
corresponds to a refinement with well-separated subsets.
Finally, in \Cref{sec:wrapup}, we prove \Cref{thm:main_v2_general} stated below, which is a version of \Cref{thm:main_v1_general} conditioned on samples satisfying deterministic stability conditions (cf.~\Cref{sec:stability}).

\begin{theorem}[Stable set version of \Cref{thm:main_v1_general}]\label{thm:main_v2_general}
    Let $d \in \Z_+$, $\delta,\alpha \in (0,1)$ be parameters, and let $C>1$ be a sufficiently large absolute constant.
    Consider a (multi-)set $T$ of $n \gg \log(1/(\alpha\delta))/\alpha$ points in $\R^d$ with $k$ disjoint subsets $S_1,\ldots, S_k \subseteq T$, where $|\cup_i S_i| \ge (1-0.02\alpha)|T|$, satisfying the following for each $i\in [k]$: (i) $|S_i| \geq 0.97\alpha n $, (ii) $S_i$ is $(C,0.04)$-stable (cf.\ \Cref{def:stability}) with respect to mean $\mu_i$ and maximum standard deviation parameter $\sigma_i$ (where $\mu_i,\sigma_i$ are unknown), (iii) for every pair $i \neq j$ we have $\|\mu_i - \mu_j \|_2 > 10^5  C^2  (\sigma_i + \sigma_j)/\sqrt{\alpha}$.
    Then \Cref{alg:main_general} on input $T,\alpha$, runs in $\poly(nd/\alpha)$-time and with probability at least $1-\delta$ (over the internal randomness of the algorithm) outputs $m \leq 1.07/\alpha$ disjoint sets $\{B_j\}_{j \in [m]}$ that satisfy the following:
    \begin{enumerate}[leftmargin=*]
        \item \label{it:concl1} The output sets $B_1,\ldots,B_{m}$ are disjoint and have size  $|B_j| \geq 0.92 \alpha n$ for all $j \in [m]$.
        \item \label{it:concl3} The  set $[m]$ can be partitioned into $k$ sets $H_1, \ldots, H_k$, such that if $\mathcal{B}_i$ are defined as $\mathcal{B}_i:=\cup_{j \in H_i} B_j$, the following hold:
        \begin{enumerate}
            \item \label{it:concl4} $\mathcal{B}_i \neq \emptyset$ for $i \in [k]$.
            \item \label{it:concl5} $|S_i \setminus \mathcal B_i| \leq 0.033|S_i|$ for every $i \in [k]$.
            \item \label{it:concl6}$|\mathcal{B}_i \setminus S_i| \leq 0.03 \alpha n$ for every $i \in [k]$.
            \item \label{it:concl7} For any $i \in [k]$ and any $j \in H_i$ we have that $\| \mu_{B_j} - \mu_i \|_2 \leq 13 C \sigma_i\sqrt{|S_i|/|B_j|}$.
            \item \label{it:concl8} For any pair $j\neq j'$ we have that $\|\mu_{B_j} - \mu_{B_{j'}} \|_2 > 4761 C (\sigma_{B_{j}} + \sigma_{B_{j'}})/\sqrt{\alpha}$. \label{it:guarantee_sep}
        \end{enumerate}
    \end{enumerate}
\end{theorem}

To end this section, we prove that \Cref{thm:main_v1_general} does indeed follow from \Cref{thm:main_v2_general}.

\begin{proof}[Proof of \Cref{thm:main_v1_general}] 
    Before we begin the proof, we note that, despite the notation $S_i$ appearing in both \Cref{thm:main_v1_general,thm:main_v2_general}, they mean slightly different sets in the context.
    In \Cref{thm:main_v1_general}, the $S_i$ sets refer to all the samples generated from the $i^\text{th}$ mixture component, prior to any corruptions.
    On the other hand, when applying \Cref{thm:main_v2_general}, we will instead consider large subsets of the samples that are stable.
    For this proof, we will use the notation $\tilde S_1, \ldots, \tilde S_k$ to denote the samples from the $i^\text{th}$ component, and we will later choose $S_i$ in the context of \Cref{thm:main_v2_general} to be large subsets of $\tilde{S}_i$ that are stable, essentially guaranteed by \Cref{fact:stability-lemma}.

     We now check explicitly that with high probability (i.e.~at least $1-\delta/2$), 
     the set $T$ in \Cref{thm:main_v1_general} has subsets $S_1,\ldots, S_k$ 
     satisfying the assumptions of \Cref{thm:main_v2_general}.
     We choose the constant $c$ that appears in the statement 
     of \Cref{thm:main_v1_general} to be the same as $13C$ in \Cref{thm:main_v2_general}.

     We can think of the mixture model as first deciding the number of samples drawn from each component, and then generating each set of samples by drawing i.i.d.~samples from the component.
     Since each component has weight at least $\alpha$ and the number of samples is $n \gg (d\log(d)+ \log(1/(\alpha \delta)))/\alpha^2$, by Chernoff-Hoeffding bounds and a union bound, with probability at least $1-\delta/100$, $|\tilde{S}_i| \geq 0.999 \alpha n \gg (d\log(d)+ \log(1/(\alpha \delta)))/\alpha$ for all $i \in [k]$. 
     Then, by \Cref{fact:stability-lemma} applied to the samples $\tilde{S}_i$ from each component,
     and a union bound over all components, we have that with probability at least $1-\delta/100$, there exist subsets $S_i' \subseteq \tilde{S}_i$ for $i \in [k]$ with $|S_i'| \geq (1-0.001\alpha)|\tilde{S}_i| $ that are $(C/2,0.05)$-stable with respect to $\mu_i$ and $\sigma_i$. This, combined with the fact that the adversary can corrupt only $0.01\alpha n$ points, means that if we let $S_i$ for $i \in [k]$ be the sets $S_i' \cap T$ (i.e.~parts of $S_i'$ that are not corrupted by the adversary), the assumptions of \Cref{thm:main_v2_general} that $|\cup_i S_i| \geq (1- 0.02\alpha)|T|$, $|S_i| \geq 0.97\alpha n$ and $S_i$ being $(C,0.04)$-stable are all satisfied with probability at least $1-\delta/2$.

     Continuing our check of the assumptions of \Cref{thm:main_v2_general}, 
     the separation assumption $\| \mu_i - \mu_j \|_2> 10^5 C^2(\sigma_i + \sigma_j)/\sqrt{\alpha}$ trivially follows from the corresponding assumption 
     in \Cref{thm:main_v1_general} (and the fact that we have chosen $c=13C$). 

     The conclusion of \Cref{thm:main_v2_general} is guaranteed to hold with probability $1-\delta/2$ 
     over the randomness of the algorithm. By a union bound over the failure event of the 
     \Cref{thm:main_v2_general} and the failure event of \Cref{fact:stability-lemma} (which are both at 
     most $\delta/2$), we get that the conclusion holds with probability 
     at least $1-\delta$ over both the randomness of the samples and the randomness of the algorithm. 
     
     We finally check that the conclusion of \Cref{thm:main_v2_general} implies the conclusion in \Cref{thm:main_v1_general}.
     \Cref{it:concl1_new} of \Cref{thm:main_v1_general}, stating that $|B_j| \geq 0.92 \alpha n$, is the same as in \Cref{thm:main_v2_general}. 
     \Cref{it:concl5_new} of \Cref{thm:main_v1_general}, stating that $|\tilde{S}_i \setminus \mathcal{B}_i | \leq 0.034|\tilde{S}_i|$ is derived from \Cref{it:concl5} of \Cref{thm:main_v2_general} as follows:
     $|\tilde{S}_i \setminus \mathcal{B}_i | \leq |S_i \setminus \mathcal{B}_i | + | \tilde{S}_i \setminus S_i | \leq 0.033|S_i|+ | \tilde{S}_i \setminus S_i | \leq 0.033|\tilde{S}_i| + 0.001|\tilde{S}_i| + 0.01\alpha n = 0.045|\tilde{S}_i|$, where the second step used \Cref{it:concl5} of \Cref{thm:main_v2_general}, the third step used that $S_i \subseteq S_i' \subseteq \tilde S_i$,  $|S_i'| \geq (1-0.001\alpha)|\tilde{S}_i| $ and that the adversary can edit at most $0.01\alpha n$ points. The last step used that $|\tilde{S}_i| \geq 0.999\alpha n$.
     \Cref{it:concl6_new} of \Cref{thm:main_v1_general}, stating that $|\mathcal{B}_i \setminus  \tilde{S}_i| \leq 0.03\alpha n$ can be derived from \Cref{it:concl6} of \Cref{thm:main_v2_general} as follows: $|\mathcal{B}_i \setminus  \tilde{S}_i| \leq |\mathcal{B}_i \setminus  S_i|  \leq 0.03\alpha n$, the first step is because $S_i \subseteq  \tilde{S}_i $ and the second step uses the guarantee from \Cref{thm:main_v2_general}. The last two parts of the conclusion of \Cref{thm:main_v1_general} follow similarly.

 \end{proof}

\section{Candidate mean pruning via convex programming}
\label{sec:program}

This section states and analyzes the convex program (in~\eqref{program:main} below) used in Line~\ref{line:main_loop_gen} of \Cref{alg:main_general}.
Line~\ref{line:main_loop_gen} assumes that for all mixture components $P_i$ and its stable subset of samples $S_i$, the list $L_\mathrm{stdev}$ contains an $\hat{s}  \in [\sigma_{S_i}, \sqrt{2}\sigma_{S_i}]$ by \Cref{cl:list_covariances}, and the list $L_\mathrm{mean}$ contains a $\muhat$ with $\|\muhat - \mu_{S_i}\| \le O(\sigma_{S_i}/\sqrt{\alpha})$ by \Cref{fact:list-decoding-mean}---recall that we denote by $\sigma_{S_i} = \sqrt{\|\Cov(S_i)\|_\op}$ the maximum standard deviation of the points in $S_i$.
At the end of the section, we will then guarantee that, after the double-loop of Line~\ref{line:main_loop_gen} finishes, the list $L \subset L_\mathrm{mean}$ also contains mean estimates close to every $S_i$, and moreover, every $\muhat \in L$ is close to some $S_i$.

We will use the notation of \Cref{thm:main_v2_general} in the following.
Recall that we denote by $T$ the input set of samples.
For every vector $\mu \in \R^d$ and $s>0$, we define the convex program below,
where the constant $C$ is the same constant appearing in \Cref{fact:list-decoding-mean}.

\begin{alignat}{2}
    & \text{Find:} &\quad &\text{$w_x \in [0,1]$ for all $x \in T$}\notag\\
   & \text{s.t.: }& \quad & \begin{aligned}[t]
                   &\left\| \sum_{x \in T} w_{x} (x-\mu)(x-\mu)^\top \right\|_{(1/\alpha)} \leq     \frac{2C^2 s^2}{\alpha} \sum_{x \in T }w_x, \\[1ex]  
                    &0.97 \alpha n \leq \sum_{x \in T} w_x
                \end{aligned}\label{program:main}
\end{alignat}

The following lemma (\Cref{lemma:main}) analyzes the convex program~\eqref{program:main}.
If for some standard deviation candidate $s$ and candidate mean $\mu$, we are guaranteed that $\mu$ is 
far from all $S_j$ with $\sigma_{S_j} \ll s$, and furthermore, there is a solution for the 
program~\eqref{program:main}, then every $S_j$ whose mean is far away from $\mu$ has negligible overlap with the solution $\{w_x\}_x$.
The first assumption corresponds to the check in 
Line~\ref{line:mean_check_gen}---\Cref{lemma:main} will be used in the context of an induction 
over the outer loop, where we assume that all clusters $S_j$ with $\sigma_{S_j} \le s$ 
have some ``representative'' candidate mean in $L$ that is close to $\mu_{S_j}$.
The conclusion of \Cref{lemma:main} certifies that $\mu$ must be close 
to some true cluster $S_i$ if Line~\ref{line:mean_check_gen} passes, 
thus allowing us to safely add this $\mu$ to the list $L$.

\begin{lemma}\label{lemma:main}
    Consider the setting of \Cref{thm:main_v2_general} and consider 
    an arbitrary pair of parameters $\mu \in \R^d$ and $s>0$.
    Suppose that: (i) for every cluster $S_j$ with $\sigma_{S_j} < s/100$ it holds 
    that $\| \mu - \mu_{S_j}\|_2 \geq 46 C s /\sqrt{\alpha}$, 
    and (ii) a solution $w_x$ for $x \in T$  to the program  defined in \Cref{program:main} exists.
    Then there exists a unique true cluster $S_i$ with $\sigma_{S_i}\geq s/100$ such that 
    $\|\mu_{S_i} - \mu \|_2 \leq 4600 C   \sigma_{S_i}/\sqrt{\alpha}$.
\end{lemma}
\begin{proof}
    By the constraint $0.97 \alpha n \leq \sum_{x \in T} w_x$ of the program, 
    it suffices to show that all clusters $S_j$ with $\|\mu_{S_i} - \mu \|_2 > 4600 C   \sigma_{S_i}/\sqrt{\alpha}$ have (in the aggregate) small overlap 
    with the solution of the program $\{w_x\}_x$, namely, 
    that $\sum_{j: \|\mu_{S_i} - \mu \|_2 > 4600 C   \sigma_{S_i}/\sqrt{\alpha}} \sum_{x\in S_j } w_{x} \le 0.01\sum_{x \in T} w_x$. 
    In order to show this, we consider a number of cases.
    We first consider clusters that have standard deviation at most $s/100$ 
    (which satisfy assumption (i) in the lemma statement), 
    and then clusters with bigger standard deviation. 
    At the end, we combine the two analyses to conclude the proof of the lemma.

    \bigskip

    \textbf{For clusters $S_j$ with $\sigma_{S_j} < s/100$:} 
    We first show that across cluster indices $j$ with $\sigma_{S_j} < s/100$, 
    we must have that $\sum_{j \, : \, \sigma_{S_j} < s/100} \sum_{x \in S_j} w_x \leq 0.003 \sum_{x \in T} w_x$. For every cluster index $j \in [k]$, we denote by $v_j$ the unit vector in the direction of $\mu_{S_j}-\mu$ and consider the partition $S_j = S_j^{\leq} \cup S_j^{>}$, where  $S_j^{>} = \{x \in S_j :  v_j^\top(x - \mu) > 45 C s /\sqrt{\alpha} \}$ and $S_j^{\leq} = S \setminus S_j^{>}$.
    That is, $S_j^{>}$ is the part of the cluster $S_j$ that is far away from $\mu$ in the direction $\mu_{S_j}-\mu$ and $S_j^{\leq}$ the points that are close. We bound the overlap of the solution $\{w_x\}_{x \in T}$ with each kind of points individually in \Cref{cl:points_far,cl:points_close} that follow. The argument for the points that are far away is that a large number of them would cause a violation of the Ky-Fan norm constraint of the program defined in \eqref{program:main}. For the points that are close to $\mu$, the  argument is that a large number of these points would move the mean of the cluster close to $\mu$ and violate our assumption that
    $\| \mu - \mu_{S_j}\|_2 \geq 46 C s /\sqrt{\alpha}$
    for every cluster $S_j$ with $\sigma_{S_j} < s/100$.
    
    \begin{claim} \label{cl:points_far}
        $\sum_{j=1}^k \sum_{x \in S_j^{>}} w_x \leq 0.001 \sum_{x \in T} w_x$.
    \end{claim}
    \begin{proof}
        This follows by the Ky-Fan norm constraint of the the program defined in \Cref{program:main}. Let $V$ be an arbitrary $(1/\alpha)$-dimensional subspace containing the span of $v_1,\ldots,v_k$ (where the $v_i$'s are defined as the unit vectors in the directions $\mu_{S_i} - \mu$ for $i \in [k]$). Then we have that:
        \begin{align*}
            \frac{2C^2 s^2}{\alpha} \sum_{x \in T }w_x 
            &\geq \left\|  \sum_{x \in T} w_x (x-\mu)(x-\mu)^\top \right\|_{(1/\alpha}) \tag{by the Ky-Fan norm constraint}\\
            &\ge  \tr\left(\sum_{x \in T} w_x \proj_{V}(x - \mu)(x-\mu)^\top\proj_V^\top \right) \tag{by def. of the Ky-Fan norm}\\
            &=  \sum_{x \in T} w_x \left\| \proj_{V}(x - \mu)  \right\|_2^2\\
            &\geq  \sum_{j=1}^{k} \sum_{x \in S_j^{>}} w_x \left\| \proj_{V}(x - \mu)  \right\|_2^2\\
            &\geq  \sum_{j=1}^{k} \sum_{x \in S_j^{>}} w_x (v_j(x-\mu))^2 \tag{since $v_j \in V$}\\
            &\geq 2000 C^2 \sum_{j=1}^{k} \sum_{x \in S_j^{>}}w_x   \frac{1}{\alpha} s^2 \;. \tag{by definition of set $S_j^{>}$}
        \end{align*}
        The above implies that $\sum_{j=1}^{k} \sum_{x \in S_j^{>}}w_x \leq 0.001 \sum_{x \in T} w_x$.
    \end{proof}
    \begin{claim}\label{cl:points_close}
     We have that  
     $\sum_{j \,:\, \sigma_{S_j} <s/100} \sum_{x \in S_j^{\leq}} w_x \leq  0.002 \sum_{x \in T} w_x$.
    \end{claim}
    \begin{proof}
        Let $\alpha_j^{\leq} := \left(\sum_{x \in S_j^{\leq}} w_x\right)/|S_j^{\leq}|$ denote the intersection of the solution with $ S_j^{\leq}$; the part of the $j$-th cluster that is close to $\mu$. We will show that $\sum_{j \in [k]: \sigma_{S_j} <s/100} \alpha_j^{\leq} \leq 0.001$.

        Since $S_j^{\leq}$ contains by definition the points $x\in S_j$ that $v_j^\top( x - \mu)\leq 45 C s /\sqrt{\alpha}$, then their mean satisfies $v_j^\top( \mu_{S_j^{\leq}} - \mu) \leq 45 C s /\sqrt{\alpha}$. 
        Then we can write 
        \begin{align*}
            v_j^\top( \mu_{S_j} - \mu_{S_j^{\leq}} ) =   v_j^\top( \mu_{S_j} - \mu) - v_j^\top( \mu_{S_j^{\leq}} - \mu) \geq 46C  s/\sqrt{\alpha} - 45 C  s/\sqrt{\alpha} \geq C s /\sqrt{\alpha} > 100 C \sigma_{S_j} /\sqrt{\alpha} \;,
        \end{align*}
        where the first inequality used the assumption that $\|\mu_{S_j} - \mu\|_2 \ge 46 C  s/\sqrt{\alpha}$ for $\sigma_{S_j} < s/100$ (and that $v_j$ is the unit vector in the direction of $\mu_{S_j} - \mu$), and the last inequality used that we consider only clusters with $\sigma_{S_j} < s/100$.
        
        The above implies that $\|\mu_{S_j} - \mu_{S_j^\le}\|_2 > 100 C \sigma_{S_j} /\sqrt{\alpha}$.
        If, for the sake of contradiction, we had $\alpha_j^{\leq} \geq 0.001 \alpha$, then \Cref{fact:513} (and the fact that $C>1$) implies $\| \mu_{S_j^{\leq}} - \mu_{S_j} \|_2 \leq 100 \sigma_{S_j} /\sqrt{\alpha} \leq 100C \sigma_{S_j} /\sqrt{\alpha} $, which is a contradiction. Thus, it must be the case that $\alpha_j^{\leq} < 0.001\alpha$.

        The above implies that $\sum_{j \,:\, \sigma_{S_j} <s/100} \sum_{x \in S_j^{\leq}} w_x \leq 0.001 \alpha  \sum_{j \,:\, \sigma_{S_j} <s/100} |S_j^{\leq}| $ $\leq 0.001 \alpha n \leq (0.001/0.97) \sum_{x \in T} w_x < 0.002 \sum_{x \in T} w_x$, where the last inequality used that $\sum_{x \in T} w_x\geq 0.97\alpha n$ is a constraint in the program \eqref{program:main}.
    \end{proof}

    \textbf{For clusters $S_j$ with $\sigma_{S_j} \geq s/100$:}
    For every cluster $S_j$, we define a similar notation as in the previous case $\alpha_j := \left(\sum_{x \in S_j} w_x\right)/|S_j|$, which quantifies the overlap of the cluster with the solution of the program. 
    As explained in the beginning, the goal is to show that all clusters $S_j$ 
    with mean far away from $\mu$ have (in the aggregate) small overlap 
    with the solution $\{w_x\}_x$ of the program. In the previous paragraph 
    (the one analyzing clusters with  $\sigma_{S_j} \geq s/100$), 
    we did not have to use that the means are far from $\mu$ 
    because we could argue separately for the points that are close to $\mu$; 
    but here considering only clusters with mean far away from $\mu$ will become crucial. 
    We will furthermore only consider clusters for which 
    $\alpha_j > 0.001\alpha$---since our goal is to show small overlap 
    in the aggregate, it suffices to do so for the clusters that individually 
    have non-trivial overlap. In summary, the clusters that we consider 
    in this paragraph are ones from the set 
    $\mathrm{Bad}:= \{ j \in [k] : \sigma_{S_j} \geq s/100, \alpha_j > 0.001\alpha, \| \mu_{S_j} - \mu\|_2 > 4600 C \sigma_{S_j} /\sqrt{\alpha}\}$, and the goal is to show that 
    $\sum_{j \in \mathrm{Bad}} \sum_{x \in S_j} w_x \leq 0.001 \sum_{x \in T}  w_x$. 
    To do this, we will show that the part of the solution coming from clusters 
    in the set $\mathrm{Bad}$ causes large variance in the subspace 
    connecting the $\mu_{S_j}$'s with $\mu$; thus, by the Ky-Fan norm constraint, 
    such contributions should be limited.

     Recall that for any cluster $S_j$, the notation $v_j$ denotes the unit vector in the direction of $\mu_{S_j} - \mu$, and $V$ denotes a subspace of dimension $1/\alpha$ that includes the span of $v_1,\ldots,v_k$ (recall $k\leq 1/\alpha$).
     Using calculations similar to \Cref{cl:points_far}, we have that
     \begin{align}
         \sum_{j=1}^k  \sum_{x \in S_j} w_x \left( v_j(x - \mu) \right)^2 
         &\leq \sum_{j=1}^k  \sum_{x \in S_j} w_x \left\| \proj_{V}(x - \mu)  \right\|_2^2 \notag\\
         &=  \sum_{x \in T } w_{x} \left\| \proj_{V}(x - \mu)  \right\|_2^2 \notag\\
         &\leq  \left\|  \sum_{x \in T} w_x (x-\mu)(x-\mu)^\top \right\|_{(1/\alpha)} \notag\\
         &\leq  \frac{2C^2 s^2}{\alpha} \sum_{x \in T }w_x\;, \label{eq:eq0}
     \end{align}
     where the last inequality is, again, by definition of the the program \eqref{program:main}.

     Now consider a cluster $S_j$ with $j \in \mathrm{Bad}$, 
     i.e.~a cluster for which $\sigma_{S_j} > s/100$,  $\alpha_j > 0.001 \alpha$ 
     and $\| \mu_{S_j} - \mu\|_2 > 4600 C   \sigma_{S_j}/\sqrt{\alpha}$. 
     Let $\mu'_j := \left(\sum_{x \in S_j} w_x x\right)/\left(\sum_{x \in S_j} w_x  \right)$. 
     We have the following by \Cref{fact:513}:
     \begin{align}
         \| \mu_j' - \mu_{S_j} \|_2 \leq \sigma_{S_j}/\sqrt{\alpha_j}\leq 100 \sigma_{S_j}/\sqrt{\alpha} \leq 100C \sigma_{S_j}/\sqrt{\alpha} \;. \label{eq:eq1}
     \end{align}
     The above implies that $v_j^\top(\mu_j' - \mu) \geq  4500 C  \sigma_{S_j}/\sqrt{\alpha}$, 
     because otherwise we would have
    \begin{align*}
        \| \mu_{S_j} - \mu \|_2 
        &= v_j^\top(\mu_{S_j} - \mu) \\
        &= v_j^\top(\mu_{S_j} - \mu'_j) + v_j^\top(\mu_j' - \mu) \\
        &\leq \| \mu_{S_j} - \mu_j' \|_2 +  4500 C  \sigma_{S_j}/\sqrt{\alpha} \\
        &\leq  100C \sigma_{S_j} /\sqrt{\alpha} +  4500 C  \sigma_{S_j}/\sqrt{\alpha} \tag{by \Cref{eq:eq1}}\\
        &\leq  4600 C   \sigma_{S_j}/\sqrt{\alpha} \;,
    \end{align*}
    which is a contradiction to $j \in \mathrm{Bad}$. Thus,
     \begin{align*}
          \sum_{x \in S_j} w_x (v_j^\top(x-\mu))^2
         &= |S_j| \frac{\alpha_j}{\sum_{x \in S_j}w_x} \sum_{x \in S_j} w_x (v_j^\top(x-\mu))^2 \\
         &\geq |S_j| \alpha_j \left(\frac{1}{\sum_{x \in S_j}w_x} \sum_{x \in S_j}w_x  v_j^\top(x-\mu) \right)^2\tag{by Jensen's inequality} \\
         &= |S_j| \alpha_j \left( v_j^\top(\mu_j' - \mu) \right)^2 \\
         &\geq |S_j| \alpha_j \cdot 2 \cdot 10^7\cdot C^2 \alpha^{-1} \sigma_{S_j}^2 \\
         &\geq |S_j| \alpha_j  \cdot 2000C^2 \cdot \alpha^{-1} s^2  \tag{since $\sigma_{S_j} \geq s/100$}\\
         &\geq \left(\sum_{x \in S_j} w_x\right)  2000C^2 \alpha^{-1} s^2 \;.
     \end{align*}
     Combining with \Cref{eq:eq0} the above shows that $\sum_{j \in \mathrm{Bad}} \sum_{x \in S_j} w_x \leq 0.001 \sum_{x \in T}  w_x$.

    \bigskip
    
     \textbf{Putting everything together: }
    We now show how the two previous analyses for the clusters $j$ with $\sigma_{S_j} < s/100$ 
    and $\sigma_{S_j} \geq s/100$ for $j \in \mathrm{Bad}$ 
    can be combined to conclude the proof of  \Cref{lemma:main}.
       
    We first argue that there exists exactly one cluster $i$ 
    with $\|\mu_{S_j} - \mu \|_2 \leq  4600 C   \sigma_{S_i}/\sqrt{\alpha}$: 
    Indeed, there cannot be more than one such clusters because if there were two clusters $i\neq j$ then by the triangle inequality and stability condition we would have 
    \begin{align*}
        \| \mu_i - \mu_j \|_2 &\leq \| \mu_{S_i} - \mu  \|_2 + \| \mu_{S_j} - \mu  \|_2 + \| \mu_i - \mu_{S_i} \|_2 + \| \mu_j - \mu_{S_j} \|_2 \tag{by the triangle inequality}\\
        &\leq \| \mu_{S_i} - \mu  \|_2 + \| \mu_{S_j} - \mu  \|_2+  C(\sigma_i+\sigma_j) \tag{by stability condition for means}\\
        &\leq 4600 C   (\sigma_{S_i}+\sigma_{S_j})/\sqrt{\alpha} + C(\sigma_i+\sigma_j) \\
        &\leq 4600 C^2   (\sigma_{i}+\sigma_{j})/\sqrt{\alpha} + C(\sigma_i+\sigma_j) \tag{by stability condition for covariances}\\
        &\leq 4601 C^2   (\sigma_{i}+\sigma_{j})/\sqrt{\alpha} \;,  \tag{using $C>1$}
    \end{align*}
    which would violate our separation assumption in \Cref{thm:main_v2_general}. It also cannot be the case that none of the clusters satisfy the condition that $\|\mu_{S_i}-\mu\|_2 \le 4600 C   \sigma_{S_i}/\sqrt{\alpha}$, because in that case we will show that we could also obtain a contradiction. Recall that in our notation $T$ is the entire dataset and $S_j$'s are the stable sets (which we often call ``clusters''). The contradiction can be derived as follows  (step by step explanations are provided in the next paragraph):
    \begin{align}
        \sum_{x \in T} w_x 
        &= \sum_{j:  \sigma_{S_j}<s/100} \sum_{x \in S_j} w_x + \sum_{j:  \sigma_{S_j} \geq s/100} \sum_{x \in S_j} w_x + \sum_{x \in T \setminus \cup_j S_j} w_x\label{eq:tmp11}\\
        &= \sum_{j:  \sigma_{S_j}<s/100} \sum_{x \in S_j } w_x + \sum_{j \in \mathrm{Bad}} \sum_{x \in S_j} w_x + \sum_{j: \sigma_{S_j}\geq s/100, j \not\in \mathrm{Bad}} \sum_{x \in S_j} w_x + \sum_{x \in T \setminus \cup_j S_j} w_x\label{eq:tmp12}\\
        &\leq 0.003 \sum_{x \in T} w_x + 0.002 \sum_{x \in T} w_x + 0.001 \sum_{x \in T} w_x + 0.02\alpha n \label{eq:tmp13}\\
        &\leq 0.003 \sum_{x \in T} w_x + 0.002 \sum_{x \in T} w_x + 0.001 \sum_{x \in T} w_x + 0.021  \sum_{x \in T} w_x \leq 0.05 \sum_{x \in T} w_x \;. \label{eq:tmp14} 
    \end{align}
    We explain the steps here: \eqref{eq:tmp11} splits the summation into a part for the large covariance clusters and one for the small covariance ones, and the part of the dataset that does not belong to any of the clusters. \eqref{eq:tmp12} further splits the sum due to large variance clusters into two parts: the clusters that belong in the set $\mathrm{Bad}$ and the rest of them. \eqref{eq:tmp13} bounds each one of the resulting terms as follows: The bound of the first term uses the analysis of small covariance clusters.
    The bound of the second term uses the analysis of large covariance clusters.
    The bound of the third term uses that, since we have assumed that $\|\mu_{S_j}-\mu\|_2 > 4600 C  \sigma_{S_j}/\sqrt{\alpha}$ for all clusters, the only way that $j \not\in \mathrm{Bad}$ can happen is because of $\alpha_j < 0.001\alpha$.
    The bound of the last term comes from the assumption in \Cref{thm:main_v2_general} that $\cup_i S_i$ contains most of the points in $T$ (this is one of the assumptions in \Cref{thm:main_v2_general}).
    Finally, \eqref{eq:tmp14} uses the fact that $\sum_{x \in T} w_x \ge 0.97\alpha n$ by construction of the the program constraints in \eqref{program:main}.

    Equation \eqref{eq:tmp14} yields the desired contradiction, thus there must be exactly one cluster $S_i$ with $\|\mu_{S_i} - \mu \|_2 \leq  4600 C  \sigma_{S_i}/\sqrt{\alpha}$. 
    This completes the proof of \Cref{lemma:main}.

\end{proof}

Having shown \Cref{lemma:main} which gives guarantees about solutions of the convex program~\eqref{program:main}, we can now state and prove the induction (\Cref{lem:induction}) which guarantees that throughout the execution of the double loop in Line~\ref{line:main_loop_gen}, every candidate mean added to the list $L$ must be close to some true cluster $S_i$, and every true cluster $S_i$ with standard deviation at most $s$ must have a corresponding candidate mean in $L$.

\begin{lemma}[Induction]\label{lem:induction}
    Consider the setting of \Cref{thm:main_v2_general} and \Cref{alg:main_general}. The first statement below holds throughout the execution and the
    second statement  holds at the start of every iteration of the loop of line~\ref{line:main_loop_gen}:
    
    \begin{enumerate}[leftmargin=*]
        \item \label{it:hypothesis0} (Every element from the list is being mapped to a true cluster): For every element $\hat{\mu}_i$ in the list $L$ there exists a true cluster $S_j$ such that 
        $\| \hat \mu_i - \mu_{S_j} \| \leq  4600 C \sigma_{S_j}/\sqrt{\alpha}$.
        
        \item\label{it:hypothesis1} (Every cluster of smaller covariance has already been found): For every true cluster $S_i$ with $\sigma_{S_i} \leq s$, there exists $\hat{\mu}_j$  in the list $L$ such that  
        $\| \hat{\mu}_j  - \mu_{S_i} \|_2 \leq 4600 C \sigma_{S_i}/\sqrt{\alpha}$.
    \end{enumerate}
\end{lemma}

Before we prove the lemma, we note that the guarantee of the lemma involves the empirical quantities $\mu_{S_i}$ and $\sigma_{S_i}$ as opposed to the ``true'' means and standard deviations $\mu_i,\sigma_i$ of the mixture components, which are the parameters that each $S_i$ is stable with respect to.
Later on in the paper, we will use the following straightforward corollary of \Cref{lem:induction}, which can be derived directly by the two stability conditions $\sigma_{S_i} \leq C \sigma_i$ and $\| \mu_{S_i} - \mu_i \|_2 \leq C \sigma_i$.
\begin{corollary}\label{cor:induction}
    In the setting of \Cref{lem:induction}, the first statement holds throughout the execution of the algorithm and the second holds at the start of every iteration of the loop of line \ref{line:main_loop_gen}:
    \begin{enumerate}[leftmargin=*]
        \item \label{it:hypothesis00} For every element $\hat{\mu}_i$ in the list $L$, there exists a true cluster $S_j$ such that 
        $\| \hat \mu_i - \mu_{j} \| \leq  4601 C^2 \sigma_{j}/\sqrt{\alpha}$.
        
        \item\label{it:hypothesis11} For every true cluster $S_i$ with $\sigma_{S_i} \leq s$, there exists $\hat{\mu}_j$  in the list $L$ such that  
        $\| \hat{\mu}_j  - \mu_{i} \|_2 \leq 4601 C^2 \sigma_{i}/\sqrt{\alpha}$.
    \end{enumerate}
\end{corollary}

We now prove \Cref{lem:induction}.

\begin{proof}[Proof of \Cref{lem:induction}]

    In everything that follows, we will informally use the phrase that  ``cluster $S_i$ has been found'' as a shorthand to the statement that there exists $\hat{\mu}_j$ in the list $L$ such that  $\| \hat{\mu}_j  - \mu_{S_i} \|_2 \leq  4600 C \sigma_i/\sqrt{\alpha}$.

    \vspace{10pt}
 
    We prove the lemma by induction. Suppose the algorithm enters a new iteration of the outer loop (line \ref{line:main_loop_gen}), and suppose that \Cref{it:hypothesis0,it:hypothesis1} (our inductive hypothesis) hold for all prior steps of the algorithm. We will show that \Cref{it:hypothesis0}  remains true each time a new element is inserted into the list $L$ in iterations of the inner loop and that  \Cref{it:hypothesis1} remains true in the next iteration of the outer loop. 
    Since showing \Cref{it:hypothesis1} is more involved, we will start with that.

    \item \paragraph{Proof of \Cref{it:hypothesis1}:}
    For \Cref{it:hypothesis1} we want to show that every cluster $S_j$ 
    with $\sigma_{S_j}\leq s$ will be found. We consider two cases: 
    The first case is $\sigma_{S_j} < s/100$. In that case, 
    by the guarantee of list-decoding for the covariances (\Cref{cl:list_covariances}), 
    there must exist a candidate standard deviation $\hat{s}$ in the list $L_{\mathrm{stdev}}$ 
    such that $\sigma_{S_j} \leq \hat{s} \leq \sqrt{2}  \sigma_{S_j}$. 
    Note that combining with $\sigma_{S_j} < s/100$ this implies that $\hat{s} < s$. 
    This means that, as the algorithm has gone through the list $L_{\mathrm{stdev}}$, 
    it must have examined that candidate covariance $\hat{s}$ in an earlier step. 
    For that step, the inductive hypothesis along with the fact that $\sigma_{S_j} \leq \hat{s}$ 
    implies that the cluster $S_j$ must have already been found at that earlier step.

    Now let us consider the case  $s/100 \leq \sigma_{S_j} \leq s$.
   We will show that, if the cluster has not been already found, then it will be found at the current iteration of the loop of line \ref{line:main_loop_gen}. We will do this by showing that there exists a candidate mean $\mu \in L_{\mathrm{mean}}$ such that: 
   \vspace{10pt}
    \begin{enumerate}[label=(\alph*)]
        \item $\| \mu - \mu_{S_j} \|_2 \leq C \sigma_{S_j}/\sqrt{\alpha}$. \label{it:listdec}
        \item \label{it:meanchecks_gen} $\| \mu - \hat\mu_i \|_2 > 99 C s/\sqrt{\alpha}$ for every $\hat\mu_i$ in the list $L$.
    
        \item \label{it:satisfiable_gen}  The program defined by \Cref{program:main} 
        is satisfiable.

    \end{enumerate}
    \vspace{10pt}
    
    Before establishing the individual claims, we point out that they indeed imply that the cluster $j$ will be found at the current iteration.
    To see this, first note that claim \ref{it:meanchecks_gen} above implies that the algorithmic check in line \ref{line:mean_check_gen} will go through when the algorithm uses the candidate mean $\mu$. Then, because of claim \ref{it:satisfiable_gen}, the program will be  satisfiable, and an application of  \Cref{lemma:main} combined with claim \ref{it:listdec} will yield that $\| \mu - \mu_{S_j} \|_2 \leq 4600 C \sigma_{S_j}/\sqrt{\alpha}$, i.e.~the cluster $S_j$ is indeed found. 
    We explain the application of  \Cref{lemma:main} in detail in the next two paragraphs.
    
    First, we check that the preconditions of \Cref{lemma:main} are established, 
    i.e.~we will check that for every cluster $\ell$ with $\sigma_{S_\ell} < s/100$ 
    it holds that $\| \mu - \mu_{S_\ell} \|_2 \geq 46 C s/\sqrt{\alpha}$ 
    and that a solution to the program exists. The satisfiability of the program 
    is due to claim \ref{it:satisfiable_gen}. In the reminder of the paragraph, 
    we show the part that $\| \mu - \mu_{S_\ell} \|_2 \geq 46 C s/\sqrt{\alpha}$ 
    for all  clusters $\ell$ with $\sigma_{S_\ell} < s/100$: 
    By the inductive hypothesis, all clusters with standard deviation at most $s/100$ 
    have already been found, meaning that if $S_\ell$ is  a cluster with $\sigma_{S_\ell}<s/100$, 
    then there is a $\hat \mu_t$ in the list with 
    $\| \hat \mu_t - \mu_{S_\ell} \|_2 \leq 4600 C \sigma_{S_\ell}/\sqrt{\alpha}$. 
    Putting everything together, if $S_\ell$ is  a cluster with $\sigma_{S_\ell}<s/100$, 
    then $\|\mu-\mu_{S_\ell} \|_2 \geq  \|\mu-\hat{\mu}_t \|_2 - \| \hat \mu_t - \mu_{S_\ell} \|_2 \geq 99 C  s/\sqrt{\alpha} - 4600 C \sigma_{S_\ell}/\sqrt{\alpha} \geq 99 C s/\sqrt{\alpha} -  46 C s/\sqrt{\alpha} \geq 46 C  s/\sqrt{\alpha}$ 
    (where the first step uses the reverse triangle inequality, 
    the second step uses claim \ref{it:meanchecks_gen} 
    for the first term and $\| \hat \mu_t - \mu_\ell \|_2 \leq 4600 C \sigma_{S_\ell}/\sqrt{\alpha}$ 
    for the second term and the next step uses that $\sigma_{S_\ell} < s/100$). 
    
    We have thus checked that \Cref{lemma:main} is applicable. 
    We now check that the conclusion of the lemma indeed implies 
    that cluster $S_j$ will be found. The conclusion of the lemma (after a renaming of the index) 
    is that there exists a unique true cluster $S_t$ with $\sigma_{S_t} \geq s/100$ 
    such that $\| \mu - \mu_{S_t}\|_2 \leq 4600 C \sigma_{S_t}/\sqrt{\alpha}$.
    Note the ``unique'' part: there cannot be any other cluster $S_{t'}$ 
    for which $\|\mu - \mu_{S_{t'}}\|_2 \leq 4600 C \sigma_{S_{t'}}/ \sqrt{\alpha}$ 
    (otherwise the separation assumption between clusters is violated). 
    This combined with claim \ref{it:listdec} means that the cluster 
    $S_t$ from the conclusion of \Cref{lemma:main} must be the same cluster 
    that we originally denoted by $S_j$. 
    Thus, we showed that cluster $S_j$ is found, as desired.

    We now show that the  claims \ref{it:listdec}, \ref{it:satisfiable_gen}, 
    and \ref{it:meanchecks_gen} hold for  $\mu$ being the mean candidate 
    for which it holds $\| \mu - \mu_{S_j} \|_2 \leq C  \sigma_{S_j}/\sqrt{\alpha}$ 
    by the list-decoding guarantee (\Cref{fact:list-decoding-mean}). 
    Thus, \ref{it:listdec} is satisfied by that fact. 
    We now show that this $\mu$ also satisfies \ref{it:satisfiable_gen}: 
    Using \ref{it:listdec} and that the standard deviation of $S_j$ in every direction 
    is at most $\sigma_{S_j}$ (by definition), we can show the following 
    for the Ky-Fan norm of the centered around $\mu$ second moment of that true cluster: 
    \begin{align*}
        \left\| \frac{1}{|S_j|} \sum_{x \in S_j}(x-\mu)(x-\mu)^\top \right\|_{(1/\alpha)}
        &\leq  \left\| \frac{1}{|S_j|} \sum_{x \in S_j}(x-\mu_{S_j})(x-\mu_{S_j})^\top \right\|_{(1/\alpha)} + \| \mu-\mu_{S_j}\|_2^2 \\
        &\leq \frac{1}{\alpha} \left\| \frac{1}{|S_j|} \sum_{x \in S_j}(x-\mu_{S_j})(x-\mu_{S_j})^\top \right\|_{(\op)} + C^2 \frac{1}{\alpha} \sigma_{S_j}^2 \\
        &\leq \frac{1}{\alpha}\left( \sigma_{S_j}^2 + C^2 \sigma_{S_j}^2 \right)\\
        &\leq 2C^2  \frac{s^2}{\alpha} \;,
    \end{align*}
     where the first step uses the inverse triangle inequality 
     and the last step uses that we only consider true clusters 
     with $\sigma_{S_j} \leq s$. Thus, the program is satisfiable 
     by the binary weights $w_x = \1(x \in S_j)$.

    We now move to establishing the claim \ref{it:meanchecks_gen}, 
    i.e.~that $\| \mu - \hat\mu_i \|_2 > 99 C s/\sqrt{\alpha}$ for every $\hat\mu_i$ in the list $L$. 
    Consider an arbitrary $ \hat\mu_i$ from the list $L$ corresponding to a previously found cluster. 
    By the inductive hypothesis, for every $ \hat\mu_i \in L$, 
    there exists a true cluster $S_\ell$ for which 
    $\| \hat\mu_i - \mu_{S_\ell} \|_2 \leq 4600 C \sigma_{S_\ell}/\sqrt{\alpha}$.
    By assumption in the context of the claim, cluster $j$ has not been found, 
    and thus $\ell \neq j$.
    Then, by the reverse triangle inequality, we obtain: 
    \begin{align*}
        \| \mu - \hat\mu_i \|_2 &\geq \| \mu_j - \mu_\ell \|_2 - \| \mu_j - \mu_{S_j}\|_2-\| \mu_\ell - \mu_{S_\ell}\|_2  - \| \mu_{S_\ell} - \hat \mu_i \|_2 - \| \mu - \mu_{S_j} \|_2 \\
        &> 10^4 C^2  (\sigma_{\ell} + \sigma_{j})/\sqrt{\alpha} - C \sigma_{j} - C \sigma_{\ell}   -  4600 C \sigma_{S_\ell}/\sqrt{\alpha} - C\sigma_{S_j}/\sqrt{\alpha} \\
        &\geq (10^4-1) C^2  (\sigma_{j}+\sigma_{\ell})/\sqrt{\alpha}  -4600 C\sigma_{S_\ell}/\sqrt{\alpha} -C\sigma_{S_j}/\sqrt{\alpha} \\
        &\geq   (10^4-1) C  (\sigma_{S_j}+\sigma_{S_\ell})/\sqrt{\alpha} -4600 C\sigma_{S_\ell}/\sqrt{\alpha} -C\sigma_{S_j}/\sqrt{\alpha} \tag{$\sigma_{S_j} \leq C \sigma_j$ by stability condition for covariances}\\
        &\geq   (10^4-2) C   \sigma_{S_j} \\
        &\geq 99 C s/\sqrt{\alpha} \;, \tag{using $s/100 < \sigma_{S_j}$}
    \end{align*}
    where the second line uses the separation assumption between clusters $\ell,j$ 
    to bound below the first term, the stability condition to bound the next two terms, 
    and the facts that
    $\| \mu - \mu_j \|_2 \leq C\sigma_{S_j}/\sqrt{\alpha}$ and 
    $\| \hat\mu_i - \mu_{S_\ell} \|_2 \leq 4600 C \sigma_\ell/\sqrt{\alpha}$ 
    that we had already established in the previous paragraph.
    The last line uses that we are analyzing only the case $s/100 < \sigma_{S_j}$.

   \item \paragraph{Proof of \Cref{it:hypothesis0}:} Consider an iteration of the (inner) loop of the algorithm. We assume that the inductive hypothesis holds for the past iterations and we will show that \Cref{it:hypothesis0} continues to be true after the current one is finished. It suffices to only consider an iteration where a new element $\hat{\mu}$ gets inserted to the list $L$ in line \ref{line:add_gen} (otherwise the claim is trivial). The fact that $\hat{\mu}$ corresponds to a true cluster will be a direct consequence of \Cref{lemma:main}.

     It remains to check that \Cref{lemma:main} is applicable, i.e.~we will check that for every
     cluster $\ell$ with $\sigma_{S_\ell} < s/100$ it holds that 
     $\| \mu - \mu_{S_\ell} \|_2 \geq 46 C s/\sqrt{\alpha}$ and that a solution to the program exists. 
     The satisfiablitity of the program is due to the fact that the algorithm 
     has reached line \ref{line:add_gen}. In the reminder of the paragraph, we show the part that 
     $\| \mu - \mu_{S_\ell} \|_2 \geq 46 C s/\sqrt{\alpha}$ for all  clusters $\ell$ 
     with $\sigma_{S_\ell} < s/100$: 
    By the inductive hypothesis, all clusters with standard deviation at most $s/100$ 
    have already been found, meaning that if $S_\ell$ is  a cluster with $\sigma_{S_\ell}<s/100$, 
    then there is a $\hat \mu_t$ in the list with $\| \hat \mu_t - \mu_{S_\ell} \|_2 \leq 4600 C \sigma_{S_\ell}/\sqrt{\alpha}$. Putting everything together, 
    if $S_\ell$ is  a cluster with $\sigma_{S_\ell}<s/100$, then 
    $\|\mu-\mu_{S_\ell} \|_2 \geq  \|\mu-\hat{\mu}_t \|_2 - \| \hat \mu_t - \mu_{S_\ell} \|_2 \geq 99 C  s/\sqrt{\alpha} - 4600 C \sigma_{S_\ell}/\sqrt{\alpha} \geq 99 C s/\sqrt{\alpha} -  46C s/\sqrt{\alpha} \geq 46C  s/\sqrt{\alpha}$, 
    where the inequalities used are the following:
    The first step uses the reverse triangle inequality, 
    the second step uses the condition in line \ref{line:mean_check_gen} of the pseudocode 
    line \ref{line:add_gen} in order to bound the first term and 
    $\| \hat \mu_t - \mu_{S_\ell}  \|_2 \leq 4600 C \sigma_{S_\ell}/\sqrt{\alpha}$ 
    for the second term, and the next inequality uses that $\sigma_{S_\ell} < s/100$.

\end{proof}

\section{Cardinality-based pruning of candidate means}\label{sec:size-pruning}

This section concerns Line~\ref{line:sizebased} of \Cref{alg:main_general}.
Right before Line~\ref{line:sizebased} is executed, 
we are guaranteed that the list $L$ of candidate means consists 
only of candidates close to one of the $S_i$ sets. Concretely, 
every $\muhat \in L$ is close to some $S_i$ with distance at most 
$O(\sigma_{S_i}/\sqrt{\alpha})$, and that every $S_i$ has some $\muhat \in L$ close to it.
At this point, the Voronoi partition of the samples is already 
an accurate refinement of the ground truth clustering (\Cref{lem:final_clustering} below).
However, we want to further ensure that the returned clustering ``looks like'' 
what we assume of our underlying mixture distribution; 
namely, that each subset has at least $\approx \alpha$ mass, 
and that the subsets are pairwise well-separated.
Line~\ref{line:sizebased} prunes candidate means, via \Cref{alg:pruning} stated below, 
to ensure that the corresponding Voronoi cell has sufficient mass.

We first show  \Cref{lem:final_clustering}, which states that the Voronoi partition based on the candidate means in $L$ does form an accurate refinement to the ground truth clustering.

\begin{lemma}[Voronoi clustering properties]\label{lem:final_clustering}
Consider the notation and assumptions of \Cref{thm:main_v2_general}.
Let $L$ be an $m$-sized list of vectors $\hat\mu_1,\ldots, \hat \mu_m$ with $m\geq k$.
Suppose the list $L$ can be partitioned into sets $H_1,\ldots,H_k$ such that for every $i \in [k]$, $H_i$ consists of the vectors $\hat \mu_j$ with  $\| \hat \mu_j - \mu_i \|_2 \leq  4601 C^2 \sigma_i/\sqrt{\alpha}$, and further assume that $H_i \neq \emptyset$ for all $i \in [k]$.  
Let $A_j = \{x \in T : \arg\min_{j' \in [m]}\|x - \hat \mu_{j'} \|_2 = j \}$ for $j \in [m]$ be the Voronoi partition (recall that $T$ denotes the entire dataset). For each $i \in [k]$ define $\mathcal{A}_i := \cup_{j: \hat \mu_j \in H_i} A_j$.
Then, the following hold:
\begin{enumerate}[leftmargin=*]
    \item \label{it:firstclaim} (Points from $S_i$ assigned to sub-clusters associated with the wrong true cluster are few)\\ $| S_i \setminus \mathcal{A}_i | \leq 0.011  |S_i|$ for every $i \in [k]$, and
    \item \label{it:secondclaim}(Points from the sub-clusters associated with a true cluster mostly include points from that true cluster) $|\mathcal{A}_i \setminus S_i | \leq 0.03 \alpha n$ for every $i \in [k]$.
    \item \label{it:thirdclaim} $|\mathcal{A}_i| \geq 0.959 \alpha n$ for $i \in [k]$.
\end{enumerate}
\end{lemma}

\begin{proof}

First, observe that \Cref{it:thirdclaim} in the lemma follows directly from \Cref{it:firstclaim} and the assumption $|S_i|\geq 0.97\alpha n$. Namely, 
\begin{align}
    | \mathcal{A}_i | \geq | \mathcal{A}_i \cap S_i | 
    \geq |S_i| - |S_i \setminus \mathcal{A}_i |
    \geq 0.989|S_i| \ge 0.959\alpha n \;.
\end{align}
Thus it suffices to prove \Cref{it:firstclaim,it:secondclaim}.

For $i \in [k]$ and for every $i'\neq i$ define the intersection of the true cluster $i$ with the union of the sub-clusters associated with cluster $i'$ as $S_{i,i'}' := S_i \cap \mathcal{A}_{i'}$.
We claim that it suffices to show that $|S_{i,{i'}}'| < (0.01 \alpha ) |S_i|$ for every $i' \neq i$, that \Cref{it:firstclaim,it:secondclaim} follow.

For the first part of the lemma statement (\Cref{it:firstclaim}), we have that 
\begin{align*}
    |S_i \setminus \mathcal{A}_i | = \sum_{i' \neq i} |S_i \cap \mathcal{A}_{i'}| = \sum_{i' \neq i} | S_{i,i'}'| \leq 0.01 |S_i| \alpha k \leq 0.011 |S_i|\;,
\end{align*}
where we used that the sets $A_1,\ldots,A_m$ form a partition of $T$, and the number of true clusters is $k\leq 1/(0.97\alpha)$ (since we assumed $|S_i| \geq 0.97\alpha n$).

Similarly, for the second part of the lemma statement (\Cref{it:secondclaim}), %
\begin{align*}
|\mathcal{A}_i \setminus S_i |   \leq \sum_{i' \neq i} |\mathcal{A}_i \cap S_{i'}| + 0.02 \alpha n \leq 0.01\alpha \sum_{i' \in [k]}|S_{i'}| + 0.02 \alpha n \leq 0.01 \alpha n + 0.02 \alpha n \leq 0.03\alpha n \;,    
\end{align*}
where the first inequality uses the assumption from \Cref{thm:main_v2_general}, that there are at most $0.02\alpha n$ points that do not belong to any of the sets $S_1,\ldots,S_n$.

We now show the claim that $|S_{i,{i'}}'| < (0.01 \alpha ) |S_i|$ for every $i,i' \in [k]$ with $i'\neq i$.
Recall our notation $\mu_i$ (for $i \in [k]$) representing the vectors that each true cluster $S_i$ is stable for (see setup of \Cref{thm:main_v2_general}).
These vectors should not be confused with the $\hat \mu_j$ ones (for $j \in [m]$), which are the approximate centers used to produce the Voronoi partition.
Since we have assumed that the $\mu_i$'s  are separated from each other and $H_i$  contains (by definition) the candidate means that are close to $\mu_i$, every pair of vectors $\hat{\mu} \in H_i$ and   $ \hat{\mu}' \in H_{i'}$ for $i\neq i'$ must also be separated:
\begin{align}
    \| \hat{\mu} - \hat{\mu}' \|_2 &\geq \| \mu_i - \mu_{i'} \|_2 - \| \hat\mu - \mu_i \|_2 - \| \hat\mu' - \mu_{i'} \|_2 \tag{by reverse triangle inequality} \\
    &\geq  10^5 C^2  (\sigma_i + \sigma_{i'})/\sqrt{\alpha} - 4601 C^2  \sigma_i/\sqrt{\alpha} - 4601 C^2 \sigma_{i'}/\sqrt{\alpha} \notag \\
    &\geq   95399 C^2  (\sigma_i + \sigma_{i'})/\sqrt{\alpha} \;. \label{eq:sep}
\end{align}

Given that every point in $S_{i,{i'}}'$ is closer to some $\muhat' \in H_{i'}$ than every $\muhat \in H_i$, and furthermore given that $\muhat$ and $\muhat'$ are far from each other according to \eqref{eq:sep}, we now show that $\| \mu_{S_{i,i'}'} - \mu_i \|_2 > 10 C^2  \sigma_i/\sqrt{\alpha}$.
Combining this with \Cref{fact:513}, we can extract that $|S_{i,i'}'| < (0.01 \alpha ) |S_i|$.
To see that by contradiction, assume that $|S_{i,i'}'| \geq  (0.01 \alpha ) |S_i|$. Then, \Cref{fact:513} ensures that $\| \mu_{S_{i,i'}'} - \mu_i \|_2 \leq 10  \sigma_{S_i}/\sqrt{\alpha} \leq 10 C  \sigma_{S_i}/\sqrt{\alpha} \leq 10 C^2   \sigma_{i}/\sqrt{\alpha}$, where we used $C>1$ as well as the stability condition for the covariance (the fact that $\sigma_{S_i}\leq C \sigma_i$).

To see that $\| \mu_{S_{i,i'}'} - \mu_i \|_2 > 10 C^2  \sigma_i/\sqrt{\alpha}$, consider an arbitrary point 

 $x \in S_{i,i'}'$ and let $\hat{\mu}' \in H_{i'}$ be the center from $L$ that is the closest one to $x$ (by definition of $S_{i,i'}'$ that closest center belongs in $H_{i'}$).
 Letting $\hat \mu$ again be an arbitrary center from $H_i$, since $x$ is closer to $\hat{\mu}'$ than $\hat \mu$, we have $\| x - \hat \mu \|_2 \geq \frac{1}{2}\| \hat \mu - \hat \mu' \|_2$. Finally, 
\begin{align}
    \| x - \mu_i \|_2 &\geq \| x - \hat \mu \|_2 - \| \hat \mu - \mu_i \|_2 \tag{by reverse triangle inequality} \\
    &\geq \frac{1}{2}\| \hat \mu - \hat \mu' \|_2 -  \| \hat \mu - \mu_i \|_2 \notag \\
    &\geq \frac{1}{2} \cdot 95399 C^2  (\sigma_i + \sigma_{i'})/\sqrt{\alpha} - 4601 C^2 \sigma_i/\sqrt{\alpha} \tag{by \eqref{eq:sep} and $\hat \mu \in H_i$}\\
    &> 10 C^2 \sigma_i/\sqrt{\alpha} \;. \notag
\end{align}
Since the above holds for every $x \in S_{i,{i'}}'$, it also holds for the mean of that set, i.e.~$
    \| \mu_{S_{i,i'}'} - \mu_i \|_2 > 10 C^2 \sigma_i/\sqrt{\alpha}$.
As we mentioned above, combining this with \Cref{fact:513} shows 
that $|S_{i,i'}'| < (0.01 \alpha ) |S_i|$, as desired.

\end{proof}

We now state \Cref{alg:pruning}, which is used in Line~\ref{line:sizebased} of \Cref{alg:main_general}.

\begin{algorithm}[h!]
\caption{Pruning of sub-clusters based on cardinality.}\label{alg:pruning}
\textbf{Input}: Dataset $T$ of $n$ points, centers $\hat \mu_1, \ldots, \hat\mu_m$ and parameter $\alpha \in (0,1)$.\\
\textbf{Output}: A subset $\hat \mu_1, \ldots, \hat\mu_{m'}$ of the input centers.
\begin{enumerate}[leftmargin=*]
    \item $J_{\mathrm{deleted}} \gets \emptyset$ .
    \item Construct the Voronoi partition $A_j = \{x : \arg\min_{j' \in [m]}\|x - \hat \mu_{j'} \|_2 = j \}$ for  $j \in [m]$.
    \item \label{line:stoping_cond} While there exists $j \in [m] \setminus J_{\mathrm{deleted}}$ with $|A_j| < 0.96 \alpha n$ do:
    \begin{enumerate}
        \item Update $J_{\mathrm{deleted}} \gets J_{\mathrm{deleted}} \cup \{ j\}$.
        \item For all $j \notin J_{\mathrm{deleted}}$, update $A_j = \{x : \arg\min_{j' \in [m] \setminus J_{\mathrm{deleted}}}\|x - \hat \mu_{j'} \|_2 = j \}$. 
    \end{enumerate}
    \item Return $\{\hat \mu_j\}_{j \in [m] \setminus J_{\mathrm{deleted}}}$.
\end{enumerate}
 
\end{algorithm}

\Cref{lem:pruning} below analyzes \Cref{alg:pruning}.

\begin{lemma}[Pruning of sub-clusters based on cardinality]\label{lem:pruning}
    Consider the notation and assumptions of \Cref{thm:main_v2_general}.
Let $L$ be an $m$-sized list of vectors $\hat\mu_1,\ldots, \hat \mu_m$ with $m\geq k$.
Suppose the list $L$ can be partitioned into sets $H_1,\ldots,H_k$ such that for every $i \in [k]$, $H_i$ consists of the vectors $\hat \mu_j$ with  $\| \hat \mu_j - \mu_i \|_2 \leq 4061 C^2 \sigma_i/\sqrt{\alpha}$, and further assume that $H_i \neq \emptyset$ for all $i \in [k]$.  
    
    Suppose that we run \Cref{alg:pruning} on $L$ as the input and denote by $\hat \mu_1, \ldots, \hat \mu_{m'}$ the sublist of centers output by the algorithm.
    Then, if we define the sets  $H_i' := \{ \hat \mu_j $ for  $j \in [m']  : \| \hat \mu_j - \mu_i \|_2 \leq 4061 C^2 \sigma_i/\sqrt{\alpha} \}$ for $i \in [k]$, then   $H_1',\ldots,H_k'$ is a partition of $\{ \hat \mu_1, \ldots, \hat \mu_{m'}\}$ and it also holds that $H_i' \neq \emptyset$ for all $i \in [k]$. 
    Moreover, in the final Voronoi clustering that corresponds to these output centers, $A_j := \{x : \arg\min_{j' \in [m']}\|x - \hat \mu_{j'} \|_2 = j \}$ for  $j \in [m']$, it holds true that $|A_j| \geq 0.96 \alpha n$.

\end{lemma}
\begin{proof}
    Consider the notation $A_j$ for the Voronoi clusters as in the pseudocode of \Cref{alg:pruning}.
    The claim that $|A_j| \geq 0.96\alpha n$ for all $j \in [m']$ follows by construction of the algorithm (line \ref{line:stoping_cond}). We thus focus on the remaining part of the lemma conclusion (the one about the sets $H_i'$).

    To show the remaining parts of the lemma conclusion,  
    it suffices to show that at any point during the algorithm's execution, if we define the sets $H_i' := \{ \hat \mu_j $ for $j \in [m] \setminus J_{\mathrm{deleted}} : \| \hat \mu_j - \mu_i \|_2 \leq 4061 C^2 \sigma_j/\sqrt{\alpha} \}$, then $H_i' \neq \emptyset$ for all $i \in [k]$ (the fact that $H_1',\ldots,H_k'$ is a partition of $L$ holds trivially by our assumption on the input).%

    In order to show that $H_i' \neq \emptyset$ for all $i \in [k]$, suppose that at some point during the algorithm's execution there exists $i \in [k]$ for which we are left with only a single center  $\hat \mu_j$ satisfying $\| \hat \mu_j - \mu_i \|_2 \leq 4061 C^2 \sigma_i/\sqrt{\alpha}$. Then, we will show that this $\hat \mu_j$ will never get deleted.
    To do so, we claim that at least $0.99 |S_i|$ points of $S_i$  have $\hat \mu_j$ as their closest center among the non-deleted centers $\{ \hat \mu_t\}_{t \in [m] \setminus  J_{\mathrm{deleted}}}$.
    From this claim, it follows that the set $A_j$ in the Voronoi partition corresponding to that center will have size  $|A_j| \geq 0.99 |S_i| \geq 0.99\cdot 0.97 \alpha n\geq 0.96 \alpha n$ (using our assumption $|S_i| > 0.97\alpha n$) and therefore $\hat \mu_j$ will never be deleted because of the deletion condition in line \ref{line:stoping_cond}. 
    
    We now prove the above claim that at least $0.99 |S_i|$ points of $S_i$  have $\hat \mu_j$ as their closest non-deleted center.
    Denote by $S_{i,i'}' := \{x \in S_i : \arg\max_{t \in [m] \setminus  J_{\mathrm{deleted}} } \|x - \hat \mu_{t} \|_2 \in H_{i'} \}$, i.e.~the part of $S_i$ consisting of the points that are closer to centers belonging in $H_{i'}$ than $H_i$.
    First we argue that it suffices to show that $|S_{i,i'}'| < 0.01\alpha|S_i|$.  This implies $\sum_{i' \neq i }|S_{i,i'}'| \leq 0.01k \alpha |S_i| \leq 0.01|S_i|$, which means that, at least $0.99|S_i|$ of the points from $S_i$ must have $\arg\max_{j' \in [m] \setminus  J_{\mathrm{deleted}} } \|x - \hat \mu_{j'} \|_2 \in H_{i}$. Finally, since we are under the assumption that $\hat \mu_j$ is the only center in $H_{i}$ from the non-deleted ones ($j \in [m] \setminus  J_{\mathrm{deleted}}$), the previous implies that at least $0.99 |S_i|$ points of $S_i$  have $\hat \mu_j$ as their closest center.

     In order to show $|S_{i,i'}'| < 0.01\alpha|S_i|$ for any $i' \neq i$, we will show that $\| \mu_{S_{i,i'}'} - \mu_i \|_2 > 10  C^2 \sigma_i/\sqrt{\alpha}$; this is enough because of \Cref{fact:513} and the fact that $S_{i,i'} \subseteq S_i$.

     It thus remains to show that $\| \mu_{S_{i,i'}'} - \mu_i \|_2 > 10 C^2 \sigma_i/\sqrt{\alpha}$.
     To do so, consider any center $\hat \mu_\ell$ that satisfies $\| \hat \mu_{\ell} - \mu_{i'} \|_2 \leq 4061 C^2 \sigma_{i'}/\sqrt{\alpha}$ and observe the following (recall that in our notation $\hat \mu_j$ is the only center from $\{ \hat \mu_t\}_{t \in [m] \setminus  J_{\mathrm{deleted}}}$ that satisfies $\| \hat \mu_j - \mu_i \|_2 \leq 4061 C^2 \sigma_i/\sqrt{\alpha}$):
      \begin{align}
        \| \hat{\mu}_j - \hat{\mu}_\ell \|_2 &\geq \| \mu_i - \mu_{i'} \|_2 - \|  \hat \mu_{\ell} - \mu_{i'}  \|_2 - \| \hat\mu_j - \mu_i \|_2 \tag{by reverse triangle inequality} \\
        &\geq  10^{5}C^2 (\sigma_i + \sigma_{i'})/\sqrt{\alpha} -  4061 C^2 \sigma_{i'}/\sqrt{\alpha} -4061 C^2 \sigma_i/\sqrt{\alpha} \notag \\
        &\geq   95399 C^2(\sigma_i + \sigma_{i'})/\sqrt{\alpha} \;. \label{eq:sepp}
    \end{align}
    Now, consider $S_{i,i'}' := \{x \in S_i : \arg\max_{t \in [m]  \setminus J_{\mathrm{deleted}} } \|x - \hat \mu_t \|_2 \in H_{i'}  \}$ and fix an $x \in S_{i,i'}'$. If   $\ell$ denotes the $ \arg\max_{t \in [m]  \setminus J_{\mathrm{deleted}} } \|x - \hat \mu_t \|_2$, then it holds $\|x - \hat \mu_j \|_2 \geq \frac{1}{2} \|  \hat \mu_j - \hat \mu_\ell \|_2 $. Then,
    \begin{align}
        \| x - \mu_i \|_2 &\geq \| x - \hat \mu_j \|_2 - \|  \hat \mu_j - \mu_i \|_2 \notag \\
        &\geq \frac{1}{2} \|  \hat \mu_j - \hat \mu_\ell \|_2- \|  \hat \mu_j - \mu_i \|_2  \notag \\
        &\geq \frac{1}{2}\cdot 95399 C^2 (\sigma_i + \sigma_{i'})/\sqrt{\alpha} - 4061 C^2 \sigma_i/\sqrt{\alpha} \tag{by \eqref{eq:sepp}}\\
        &> 10  C^2 \sigma_i/\sqrt{\alpha} \;. \notag 
    \end{align}
    Since, the above holds for every $x \in S_{i,i'}'$, then it must also hold for their mean of the set, i.e.~$\| \mu_{S_{i,i'}'} - \mu_i \|_2 > 10  C^2\sigma_i/\sqrt{\alpha}$.
 
\end{proof}

\section{Distance-based pruning of candidate means}\label{sec:distance-pruning}

In the previous section, we gave \Cref{alg:pruning} used in Line~\ref{line:sizebased} of \Cref{alg:main_general}, which ensures that the list $L$ of candidate means corresponds to a Voronoi partition that is an accurate refinement of the true clustering $\{S_i\}_i$, and furthermore, that each subset in the partition has size at least $\approx \alpha n$.

This section concerns Line~\ref{line:distancebased} of \Cref{alg:main_general}, which additionally prunes the list $L$ so that the Voronoi cells are in fact far apart from each other, satisfying a pairwise separation that is qualitatively identical to the separation assumption we impose on the underlying mixture distribution.

Due to the existence of adversarial corruptions and heavy-tailed noise in the data set, we first need to use filtering on each Voronoi cell (\Cref{alg:filteredvoronoi}), in order to make sure that the mean of the filtered Voronoi cell is actually close to the mean of the $S_i$ that the cell corresponds to.
\Cref{cor:filteredvoronoi} states the guarantees after such filtering.

\begin{algorithm}
\caption{Filtered Voronoi partitioning}\label{alg:filteredvoronoi}
\textbf{Input}: Dataset $T$ of $n$ points and centers $\hat \mu_1, \ldots, \hat\mu_m$.\\
\textbf{Output}: Disjoint subsets $B_1,\ldots,B_m$ of $T$.
\begin{enumerate}[leftmargin=*]
    \item Construct the Voronoi partition $A_j = \{x \in T: \arg\min_{j' \in [m]}\|x - \hat \mu_{j'} \|_2 = j \}$.
    \item $B_j \gets \textsc{Filter}(A_j)$ for $j \in [m]$, where $\textsc{Filter}$ denotes the filtering algorithm from \Cref{fact:filtering}.
    \item Output $B_1,\ldots, B_m$.
\end{enumerate}
\end{algorithm}

\begin{corollary}[Filtered Voronoi clustering properties]\label{cor:filteredvoronoi}
    Consider the setting of \Cref{lem:final_clustering} and furthermore assume that the Voronoi sets have size $|A_j|\geq 0.96\alpha n$ for every $j \in [m]$. Then the algorithm $\filteredvoronoi(T, \{ \hat \mu_i \}_{i\in [m]})$
    outputs disjoint sets $B_1,\ldots,B_m$ such that with probability $1-\alpha \delta/10$, the following are true (denote $\mathcal{B}_i = \cup_{j: \hat \mu_j \in H_i} B_j$, where $H_i$'s are defined as in \Cref{lem:final_clustering}):
    
    \begin{enumerate}[leftmargin=*]

    \item  \label{it:02}$| S_i \setminus \mathcal{B}_i | \leq 0.033  |S_i|$ for every $i \in [k]$.
    \item \label{it:03}$|\mathcal{B}_i \setminus S_i | \leq 0.03 \alpha n$ for every $i \in [k]$ and $|A_j \setminus B_j| \leq 0.04|A_j|$ for every $j \in [m]$.
    \item \label{it:04} For any $j \in [m]$ such that $\hat \mu_j \in H_i$, it holds $\| \mu_{B_j} - \mu_i\|_2 \leq 13 C \sigma_i\sqrt{|S_i|/|B_j|}$ and $\sigma_{B_j} \leq 20 C \sigma_i\sqrt{|S_i|/|B_j|}$.
    \item \label{it:01} $|\mathcal{B}_i| \geq 0.93\alpha n$ for $i \in [k]$.
\end{enumerate}
\end{corollary}

\begin{proof}
As in the previous lemma, we first note that \Cref{it:01} follows directly from \Cref{it:02}.
    \begin{align*}
    | \mathcal{B}_i | \geq | \mathcal{B}_i \cap S_i |
    \geq 0.967|S_i| \ge 0.93\alpha n \;,
    \end{align*}
where the second inquality uses \Cref{it:02} and the last inequality uses $|S_i| \ge 0.97\alpha n$ by the setup in \Cref{thm:main_v2_general}.

    If $A_1,\ldots,A_m$ is the Voronoi clustering before filtering and $\mathcal{A}_1,\ldots, \mathcal{A}_k$ as in \Cref{lem:final_clustering}, then by that lemma: $| S_i \setminus \mathcal{A}_i | \leq 0.011  |S_i|$, $|\mathcal{A}_i \setminus S_i | \leq 0.03 \alpha n$ and $|\mathcal{A}_i| \geq 0.959 \alpha n$ for all $i\in [k]$. 
    In everything that follows we assume $|A_j| \geq 0.96\alpha n$.
    Let $B_j$ denote the filtered sets output by the algorithm of \Cref{fact:filtering} on input $A_j$.

    \item 
    \paragraph{Proof of \Cref{it:04}:}
    Recall that the outputs $B_j$ of \Cref{alg:filteredvoronoi} are filtered versions of the sets $A_j$ from the Voronoi partition.
    \Cref{it:04} states that the filtered version $B_j \subseteq \cB_i$ must have mean close to $\mu_i$ and covariance not too large.
    We check this by showing the preconditions of \Cref{fact:filtering} (applied with $\eps=0.04$), and then \Cref{it:04} follows from applying the fact
     with $A_j$ as the set $T$ from the fact statement and $A_j \cap S_i$ as the set $S$ in that statement, where $i$ here is the index for which $A_j \subseteq \mathcal{A}_i$.
     
     We will apply \Cref{fact:filtering} with $\eps=0.04$. For this to be applicable, we need to ensure that $|T \setminus S| \leq 0.04|T|$, which using $A_j$ in place of $T$ and $A_j \cap S_i$ in place of $S$ becomes $|A_j \setminus S_i| \leq 0.04|A_j|$. Applying \Cref{fact:filtering} also requires that $A_j \cap S_i$ is stable (\Cref{def:stability}). We start by establishing the first requirement, that $|A_j \setminus S_i| \leq 0.04|A_j|$:
     \begin{align}
         |A_j \cap S_i| &= |A_j| - |A_j \setminus S_i| \notag\\
         &\geq |A_j| - |\cA_i \setminus S_i| \tag{since $A_j \subseteq \mathcal{A}_i$}\\
         &\geq |A_j| - 0.03\alpha n \tag{$|\mathcal{A}_i \setminus S_i | \leq 0.03 \alpha n$ by \Cref{lem:final_clustering}}\\
         &\geq 0.96|A_j| \;, \label{eq:tmp3545}
     \end{align}
     where the last line uses that we have assumed $|A_j| \geq 0.96\alpha n$.
    Using the above $|A_j \setminus S_i| = |A_j| - |A_j \cap S_i| \leq 0.04|A_j|$, as desired.

    We now establish the second requirement, that  $A_j \cap S_i$ is stable (\Cref{def:stability}). 
    To this end, since $S_i$ was assumed to be $(C,0.04)$-stable with respect to $\mu_i$ and $\sigma_i$, then using  \Cref{lem:stability_small_sets} we have that $A_j \cap S_i$ is $(1.23 C\sqrt{|S_i|}/\sqrt{0.04 |A_j \cap S_i|},0.04)$-stable with respect to $\mu_i,\sigma_i$.
    
    The first part of the conclusion of \Cref{fact:filtering} is that if $B_j$ denotes the output of the filtering algorithm run on $A_j$, it holds $|B_j| \geq 0.96 |A_j|$, the second part states that
    \begin{align*}
        \| \mu_{B_j} - \mu_i \|_2 \leq 12.3\, C \sigma_i \sqrt{ \frac{|S_i|}{0.04 |A_j \cap S_i|}}\sqrt{0.04} \leq 13 C \sigma_i \sqrt{ \frac{|S_i|}{|B_j |}}
    \end{align*}
    where the last inequality above is because $|B_j| \leq |A_j| \leq  |A_j \cap S_i|/0.96$, where the last step here is because of \eqref{eq:tmp3545}.

    Similarly, the third part of the conclusion of \Cref{fact:filtering} is that $\sigma_{B_j} \leq 20 C \sigma_i \sqrt{  |S_i|/|B_j |}$.
     Lastly, we check that the condition on the size of the sets from \Cref{fact:filtering} is indeed satisfied because  $|A_j| \geq 0.96\alpha n \gg \log(1/(\alpha \delta))$, where we used the assumption on the size of $n$ from  \Cref{thm:main_v2_general}.

    \paragraph{Proof of \Cref{it:02}:} We have already shown that \Cref{fact:filtering} is applicable for analyzing the effect of the filtering algorithm on input $A_j$ and thus $|A_j \setminus B_j| \leq 0.04|A_j|$ (first part of the conclusion of \Cref{fact:filtering}). Then, 
    \begin{align}
        | S_i \setminus \mathcal{B}_i | &= | S_i \setminus \mathcal{A}_i | + \sum_{j: A_j \subseteq \cA_i } |A_j \setminus B_j| \notag\\
        &\leq 0.011|S_i| + 0.04\sum_{j: A_j \subseteq \cA_i }|A_j| \tag{by \Cref{lem:final_clustering} and \Cref{fact:filtering}} \\
        &= 0.011|S_i| + 0.04 |\cA_i| \tag{$A_j$'s are disjoint}\\
        &= 0.011|S_i| + 0.04 (|\cA_i \cap S_i| + |\cA_i \setminus S_i|) \notag \\
        &\leq 0.011|S_i| + 0.04 (|S_i| + 0.03\alpha n) \tag{ $|\cA_i \setminus S_i| \leq 0.012\alpha n$ by \Cref{lem:final_clustering}}\\
        &\leq 0.011|S_i| + 0.04 \left(|S_i| + \frac{0.03}{0.97 }|S_i| \right) \leq  0.033 |S_i| \tag{by assumption that $|S_i| \geq 0.97\alpha n$}
    \end{align}

    \paragraph{Proof of \Cref{it:03}:} We have that $|\mathcal{A}_i \setminus S_i | \leq 0.03 \alpha n$ before the filtering takes place. Since filtering only removes points, $\mathcal{B}_i \subseteq \mathcal{A}_i$ and thus $|\mathcal{B}_i \setminus S_i | \leq 0.03 \alpha n$ continues to hold after the filtering.

\end{proof}

Having shown guarantees on the filtered Voronoi cells, we now give \Cref{alg:final_pruning}, used in Line~\ref{line:distancebased} of \Cref{alg:main_general}, which is responsible for further pruning the candidate means in $L$ such that the resulting filtered Voronoi cells are well-separated.
\Cref{lem:distance-based-pruning} gives the guarantees of \Cref{alg:final_pruning}.

\medskip

\begin{algorithm}
\caption{Distance-based pruning of sub-clusters}\label{alg:final_pruning}
\textbf{Input}: Dataset $T$ of $n$ points,  centers $\hat \mu_1, \ldots, \hat\mu_m$, and parameter $\alpha \in (0,1)$.\\
\textbf{Output}: A subset $\hat \mu_1, \ldots, \hat\mu_{m'}$ of the input centers.
\begin{enumerate}[leftmargin=*]
    \item \label{line:filteredVor}$\{B_1,\ldots,B_m\} \gets \filteredvoronoi(\{\hat \mu_1, \ldots, \hat\mu_m \},T)$.
    \item $J_{\mathrm{deleted}} \gets \emptyset$.
    \item \label{line:flag} While there exist $j,j'$ with $\| \mu_{B_j} - \mu_{B_{j'}} \|_2 \leq 4761 C (\sigma_{B_j} + \sigma_{B_{j'}} )/\sqrt{\alpha}$:
    \begin{enumerate}
        \item \label{line:dist1} Calculate $d = \min_{t \in [m] } \|\mu_{B_j} - \mu_{B_t} \|_2/\sigma_{B_j}$ and  $d' = \min_{t \in [m] } \|\mu_{B_{j'}} - \mu_{B_t} \|_2/\sigma_{B_{j'}}$.
        \item If $d < d'$: \label{line:distance_check1}
        \begin{enumerate}
            \item \label{line:check1} $j_{\mathrm{deleted}} \gets  j $.
        \end{enumerate}
        \item Else:\label{line:distance_check2}
        \begin{enumerate}
            \item \label{line:check2} $j_{\mathrm{deleted}} \gets   j'$.
        \end{enumerate}
        \item Update $J_{\mathrm{deleted}} \gets J_{\mathrm{deleted}} \cap \{ j_{\mathrm{deleted}} \}$
        \item \label{line:updateVor} Update $\{B_j \}_{j \in [m] \setminus J_{\mathrm{deleted}}} \gets \filteredvoronoi( \{\hat \mu_j \}_{j \in [m] \setminus J_{\mathrm{deleted}} },T)$.
    \end{enumerate}

    \item Output $\hat \mu_j$ for $j \in [m] \setminus J_{\mathrm{deleted}}$ after relabeling the indices so that they are from $1$ to $m - |J_{\mathrm{deleted}}|$.
\end{enumerate}
 
\end{algorithm}

\newpage

\begin{lemma}[Distance-based pruning of sub-clusters]\label{lem:distance-based-pruning}
   Consider the setting and notation of \Cref{thm:main_v2_general}.
   Let $L = \{\hat \mu_1,\ldots, \hat \mu_m\}$ be a list of vectors for some $m \geq k$.
   Suppose the list $L$ can be partitioned into sets $H_1,\ldots,H_k$ such that for every $i \in [k]$, $H_i$ consists of the vectors $\hat \mu_j$ with  $\| \hat \mu_j - \mu_i \|_2 \leq 4061 C^2   \sigma_i/\sqrt{\alpha}$, and that $H_i \neq \emptyset$ for all $i \in [k]$.
   Also assume that every set in the Voronoi partition $A_j = \{x : \arg\min_{j'} \| x - \hat \mu_{j'} \|_2 = j \}$ for $j \in [m]$ has size $|A_j| \geq 0.96\alpha n$.
   Consider an execution of $\distancebasedpruning(L,T,\alpha)$ algorithm (\Cref{alg:final_pruning}) with the list $L$, the entire dataset of points $T$ and the parameter $\alpha$ as input.
       
   After the algorithm terminates, let $\hat \mu_1' ,\ldots, \hat \mu_{m'}'$ be the output list (where we denote by $m'$ its size). Then the following three statements hold with probability at least $1-\delta/2$: 
   \begin{enumerate}[leftmargin=*]
       \item \label{it:conclusion1} The output list $\{\hat \mu'_j \}_{j \in [m'] }$ can be partitioned into sets $H_1',\ldots,H_k'$ such that for every $i \in [k]$,  $H_i'$ consists of the vectors of $\hat \mu_j'$ with  $\| \hat \mu_j' - \mu_i \|_2 \leq 4061 C^2  \sigma_i/\sqrt{\alpha}$  and it holds $H_i' \neq \emptyset$ for all $i \in [k]$.
       \item \label{it:conclusion2} Every set in the Voronoi partition corresponding to the output centers $A_j' = \{x : \arg\min_{j' \in [m']} \| x - \hat \mu_{j'} \|_2 = j \}$ for $j \in [m']$ has size $|A_j'| \geq 0.96\alpha n$.
       \item \label{it:conclusion3} If $B_1',\ldots,B'_{m'}$ denote the output of $\filteredvoronoi(\{\hat \mu_1' ,\ldots, \hat \mu'_{m'}\},T)$ for the non-deleted centers, then it holds that $\| \mu_{B'_j} - \mu_{B'_{j'}} \|_2 \geq 4761 C (\sigma_{B'_j} + \sigma_{B'_{j'}})/\sqrt{\alpha}$ for every $j,j' \in [m']$ with $j \neq j'$.
   \end{enumerate}

\end{lemma}
\begin{proof}

    The final part of the lemma conclusion, \Cref{it:conclusion3}, holds by design of the stopping condition of our algorithm (line \ref{line:flag}).

    We show the remaining parts (\Cref{it:conclusion1,it:conclusion2}) by induction. 
    That is, we will fix an iteration of the algorithm, assume that \Cref{it:conclusion1,it:conclusion2} hold just before the iteration starts, and prove that they continue to hold after the iteration ends.
    More specifically, since in each iteration we use $\filteredvoronoi$, which is randomized, we may allow a probability of failure for each step in our inductive hypothesis, in particular, we will use probability of failure $(\delta/2)$ divided by the maximum number of iterations (so that by \Cref{fact:condunionbound}, the conclusion holds after all iterations end with probability at least $1-\delta/2$). 
    \begin{fact}
\label{fact:condunionbound}
If event $A$ happens with probability $1-\tau_1$ and event $B$ happens with probability $1-\tau_2$ conditioned on event $A$, then the probability of both $A$ and $B$ happening is at least $1-\tau_1-\tau_2$.
\end{fact}
    The upper bound on the number of iterations can be trivially seen to be $1/(0.96\alpha)$. 
    This is because we assumed that every Voronoi set in the beginning has size $|A_j|\geq 0.96 \alpha n$ and the algorithm only deletes one of the candidate means at a time, thus the algorithm will trivially terminate after $1/(0.96\alpha)$ steps. It therefore suffices to show that the inductive step of our proof holds with probability at least $1-0.1\alpha\delta$.

    Since the iteration under consideration alters the list of vectors and some associated quantities, we must ensure that our notation reflects the specific moment within the algorithm. To achieve this, we will use unprimed letters to represent quantities at the moment just before the iteration begins ($ J_{\mathrm{deleted}}, H_i, A_i, B_i$) for the set of deleted indices appearing in the pseudocode, the partition, the Voronoi clustering, and the filtered Voronoi clustering), and primes to denote the quantities ($ J_{\mathrm{deleted}}', H_i', A_i', B_i'$) after the iteration ends. That is, our inductive hypothesis is that
    \begin{enumerate}[label=(\alph*),leftmargin=*]
       \item \label{it:inductive1} The list $\{\hat \mu_j \}_{j \in [m] \setminus J_{\mathrm{deleted}}}$ can be partitioned into sets $H_1,\ldots,H_k$ such that for every $i \in [k]$,  $H_i$ consists of the vectors of $\hat \mu_j$ with  $\| \hat \mu_j - \mu_i \|_2 \leq 4061 C^2 \sigma_i/\sqrt{\alpha}$  and it holds $H_i \neq \emptyset$ for all $i \in [k]$.
       \item \label{it:inductive2} Every set in the Voronoi partition corresponding to the  centers $A_j = \{x : \arg\min_{j \in [m]\setminus J_{\mathrm{deleted}}} \| x - \hat \mu_{j'} \|_2 = j \}$ for $j \in [m]\setminus J_{\mathrm{deleted}}$ has size $|A_j| \geq 0.96\alpha n$.
   \end{enumerate}
   And we will show that after the iteration ends, if $J_{\mathrm{deleted}}'$ denotes the updated set of deleted indices (i.e.~the set that also includes the index that was deleted during the current iteration), and $A_j',B_j'$ denote the Voronoi sets and filtered Voronoi sets corresponding to the centers $\hat \mu_j$ for $j \in [m] \setminus J_{\mathrm{deleted}}'$, the following hold with probability at least $1-0.1\alpha \delta$:
    \begin{enumerate}[leftmargin=*]
       \item \label{it:conclusion1_new} The updated list $\{\hat \mu_j \}_{j \in [m] \setminus J_{\mathrm{deleted}}'}$ can be partitioned into sets $H_1',\ldots,H_k'$ such that for every $i \in [k]$,  $H_i'$ consists of the vectors of $\hat \mu_j$ with  $\| \hat \mu_j - \mu_i \|_2 \leq 4061 C^2 \sigma_i/\sqrt{\alpha}$  and it holds $H_i' \neq \emptyset$ for all $i \in [k]$.
       \item \label{it:conclusion2_new} Every set in the Voronoi partition corresponding to the updated centers $A_j'$, where $A_j' = \{x : \arg\min_{j \in [m]\setminus J'_{\mathrm{deleted}}} \| x - \hat \mu_{j'} \|_2 = j \}$ for $j \in [m]\setminus J'_{\mathrm{deleted}}$ has size $|A_j'| \geq 0.96\alpha n$.
   \end{enumerate}
    Now observe that, by construction, every iteration only deletes a vector from the 
    list, and therefore the list $\{\hat \mu_j \}_{j \in [m] \setminus J_{\mathrm{deleted}}'}$ 
    can  be partitioned into the sets $H'_1, \ldots, H'_k$ satisfying the first part of 
    \Cref{it:conclusion1_new} (that $H_i'$ consists of the vectors of 
    $\hat \mu_j$ with  $\| \hat \mu_j - \mu_i \|_2 \leq 4061 C^2 \sigma_i/\sqrt{\alpha}$). 
    Regarding \Cref{it:conclusion2_new}, this trivially holds because deleting a point, can 
    only make the Voronoi clusters bigger in size.
    The only nontrivial condition to check is that $H'_i$ remains non-empty for all $i\in[k]$.
    Equivalently, we need to show that, if at the beginning of an iteration, $H_i$ consists of 
    only a single vector, then it will never be removed in the iteration.

    By our inductive hypothesis that the partition $H_1,\ldots, H_k$ with the aforementioned properties exists (\Cref{it:inductive1}) and our assumption that $|A_j| \geq 0.96\alpha n$ (\Cref{it:inductive2} of inductive hypothesis),  \Cref{cor:filteredvoronoi} is applicable.
    The application of that implies that the following holds with probability at least $1-0.1\alpha \delta$: Denote by 
    $\mathcal{B}_i = \cup_{j \in [m] \setminus J_{\mathrm{deleted}}: \hat \mu_j  \in H_i} B_j$ for $i\in [k]$, i.e.~$\mathcal{B}_i$ is the union of all  Voronoi clusters corresponding to (non-deleted) centers in $H_i$. Then,
        \begin{enumerate}[label=(\roman*),leftmargin=*]
            \item $\mathcal{B}_i \neq \emptyset$ for $i \in [k]$.
            \item \label{it:mean_close} For any $j\in [m] \setminus J_{\mathrm{deleted}}$ such that $ \hat \mu_j  \in H_i$ it holds $\| \mu_{B_j} - \mu_{i} \|_2 \leq 14 C \sigma_i/\sqrt{\alpha}$ and $\sigma_{B_j} \leq 21 C \sigma_i/\sqrt{\alpha}$.
            \item \label{it:large_sub} For every $i \in [k]$ with $|H_i| = 1$, if $j \in [m] \setminus J_{\mathrm{deleted}}$ denotes the unique index for which $B_j = \mathcal{B}_i$, then it holds $\sigma_{B_j} \leq 21 C \sigma_i$.
        \end{enumerate}
        The second statement above can be extracted from \Cref{it:04} of  \Cref{cor:filteredvoronoi} after noting that $|B_j| \geq |A_j|-|A_j \setminus B_j| \geq 0.96 |A_j| \geq  0.92 \alpha n \geq  0.94 \alpha |S_i|$, where we used $|A_j \setminus B_j| \leq 0.04|A_j|$ (\Cref{it:04} of  \Cref{cor:filteredvoronoi}) and the assumption that $|A_j|\geq 0.96 \alpha n$. The third statement (\ref{it:large_sub} above) can be extracted from \Cref{it:04} of  \Cref{cor:filteredvoronoi} after noting that $|B_j| \geq |B_j \cap S_i| \geq |S_i| - |S_i \setminus B_j| \geq 0.967 |S_i|$, where we used that $|S_i \setminus B_j| \leq 0.033|S_i|$ by \Cref{it:02} of \Cref{cor:filteredvoronoi}.

    We will also use the notation $\mathrm{par}(B_j)$ to denote the index $i \in [k]$ for which it holds
    $\|\mu_{B_j} - \mu_i \|_2 \leq 35 C \sigma_i/\sqrt{\alpha} $ (by the fact that $\mathcal{B}_i \neq \emptyset$ mentioned above and the separation assumption for the $\mu_i$'s, such an index indeed exists and it is unique). 
    We will call $\mathrm{par}(B_j)$ the ``parent'' of $B_j$. By slightly overloading this notation, we will also use $\mathrm{par}(\hat \mu_j)$ to denote the index $i$ for which it holds $\hat \mu_j \in H_i$, i.e.~$\|\hat \mu_{j} - \mu_i \|_2 \leq 4061 C^2 \sigma_i/\sqrt{\alpha} $ . We will also informally call the $B_j$'s ``sub-clusters'' (as opposed to the sets $S_i$ that we call ``true'' or ``parent'' clusters).

    Using this notation, and further denoting by $j_{\mathrm{deleted}}$ the index of the vector deleted in the current iteration, what remains to check is equivalent to the statement that $|H_{\mathrm{par}(j_{\mathrm{deleted}})}| > 1$.

    To show this, we need \Cref{cl:temp} below, which states the straightforward fact that sub-clusters with the same parent cluster will have means close to each other, and sub-clusters with different parents necessarily have means much farther.
    This in particular implies that, given a sub-cluster $B_j$, the closest sub-cluster must share the same parent if $|B_j| > 1$.
    We will now use \Cref{cl:temp} to show the statement that the deleted vector $\muhat_{j_{\mathrm{deleted}}}$ must have $|H'_{\mathrm{par}(j_{\mathrm{deleted}} )}| > 1$, and provide the simple proof of \Cref{cl:temp} at the end, which follows from straightforward applications of the reverse triangle inequality.
    \begin{claim}\label{cl:temp}
        The following holds for every  for $j,j' \in [m] \setminus J_{\mathrm{deleted}}$ with $j\neq  j'$: Denote by $\ell := \mathrm{par}(B_j)$, $\ell':= \mathrm{par}(B_{j'})$. 
        If $\ell=\ell'$, then $\| \mu_{B_j} -  \mu_{B_{j'}} \|_2 \leq 28 C \sigma_\ell /\sqrt{\alpha}$, otherwise, $\| \mu_{B_j} -  \mu_{B_{j'}} \|_2 > 10^4 C^2 (\sigma_\ell + \sigma_{\ell'})/\sqrt{\alpha}$.
    \end{claim}

    We will now show our end goal using a case analysis (and \Cref{cl:temp}).
    Denote by $B_{j},B_{j'}$ the sub-clusters that are identified in line \ref{line:flag} of \Cref{alg:final_pruning} (i.e.~one of $j$ or $j'$ will eventually be what we called $j_{\mathrm{deleted}}$ before).
    We need to show that, if the index $j$ is the one that gets deleted, then $|H'_{\mathrm{par}(j)}| > 1$, and similarly for $j'$.
    We check each of the following cases:
    \begin{enumerate}[leftmargin=*]       
        
        \item (Case where $|H_{\mathrm{par}(B_j)}| = 1, |H_{\mathrm{par}(B_{j'})}| > 1$) 
        Let $\ell,\ell'$ be the parents of $B_j$ and $B_{j'}$ respectively. 
        We first note that $\sigma_{B_{j}} \leq 21 C \sigma_\ell$ by the third property of $B_j$ (\Cref{it:large_sub}).
        Now we argue that, since $j$ and $j'$ are flagged by line \ref{line:flag} of \Cref{alg:final_pruning}, it must be the case that $\sigma_{B_{j'}} > 21C \sigma_{\ell'}$, for otherwise: 
        \begin{align}
            \| \mu_{B_{j}} - \mu_{B_{j'}} \| &\geq \| \mu_\ell - \mu_{\ell'} \| - \| \mu_{B_{j'}} - \mu_{\ell'} \| - \| \mu_{B_{j}} - \mu_\ell \| \tag{reverse triangle inequality}\\
            &\geq 10^{5} C^2 (\sigma_\ell + \sigma_{\ell'})/\sqrt{\alpha} -14C \sigma_{\ell'}/\sqrt{\alpha}-14C \sigma_\ell/\sqrt{\alpha} \tag{by separation assumption and \Cref{it:mean_close}} \\
            &\geq (10^{5}-14) C^2 (\sigma_\ell + \sigma_{\ell'})/\sqrt{\alpha} \label{eq:tmp243}\\
            &\geq  4761 C(\sigma_{B_{j}} + \sigma_{B_{j'}})/\sqrt{\alpha} \tag{using $\sigma_{B_{j}} \leq  21 C \sigma_\ell$, $\sigma_{B_{j'}} \leq 21 C \sigma_{\ell'}$}
        \end{align}
        Having shown that $\sigma_{B_{j}} \leq 21 C \sigma_\ell$ and $\sigma_{B_{j'}} > 21 C \sigma_{\ell'}$, we will now show that the center $\hat \mu_j$ corresponding to $B_j$ will not be the one deleted in this loop iteration, and instead the center $\hat \mu_{j'}$ corresponding to $B_{j'}$ will be the one that will get deleted.
        To see that, denote by $d$ and $d'$ the same quantities as in line \ref{line:dist1} of the pseudocode, i.e.~the normalized distances of the sub-clusters from their closest other sub-clusters. 
        
        On the one hand, we have that
        \begin{align*}
            d:= \frac{\min_{t \in [m] \setminus J_{\mathrm{deleted}}} \|\mu_{B_j} - \mu_{B_t} \|_2}{\sigma_{B_j}} \geq \frac{(10^5-14) C^2 \sigma_\ell}{\sqrt{\alpha} \sigma_{B_{j}}} \geq  \frac{4761 C}{\sqrt{\alpha}} \;,
        \end{align*}
        where the first step follows by the fact that the closest sub-cluster to $B_j$ must have as parent a different true cluster (because $|\mathrm{par}(B_j)| = 1$), and since true clusters are sufficiently separated, the closest sub-cluster to $B_j$ must be at least $(10^5-14) C^2  \sigma_\ell/\sqrt{\alpha}$-away (see the derivation of \eqref{eq:tmp243} for an identical proof).
        The last step uses that $\sigma_{B_{j}} \leq 21 C \sigma_\ell$.

        On the other hand, for the (normalized) distance of $B_{j'}$ to its closest sub-cluster (denote that sub-cluster by $B_{t^*}$) we have the following: First note that $B_{t^*}$ must have the same parent as $B_{j'}$ due to \Cref{cl:temp} and $|\mathrm{par}(B_{j'})| > 1$. Then, since both have $\ell'$ as their parent,
        \begin{align*}
            d' &:= \frac{\min_{t \in [m] \setminus J_{\mathrm{deleted}}} \|\mu_{B_{j'}} - \mu_{B_t} \|_2}{\sigma_{B_{j'}}} \\
            &= \frac{\| \mu_{B_{j'}} - \mu_{B_{t^*}} \|_2}{\sigma_{B_{j'}}}\\
            &\leq \frac{\| \mu_{B_{j'}} - \mu_{\ell'} \|_2}{\sigma_{B_{j'}}} + \frac{\|  \mu_{\ell'} - \mu_{B_{t^*}} \|_2}{\sigma_{B_{j'}}} \\
            &\leq 2 \cdot \frac{14 C^2 \sigma_{\ell'}}{\sqrt{\alpha} \sigma_{B_{j'}}} \tag{by \Cref{it:mean_close} and $C>1$}\\
            &\leq 2C/\sqrt{\alpha} \;. \tag{using $\sigma_{B_{j'}} \geq 21 C \sigma_{\ell'}$ }
        \end{align*}
        This means that $d' < d$ and line \ref{line:check2} of the algorithm will delete $\hat \mu_{j'}$, i.e.~the eliminated center is not the only center of its parent.

        \item (Case $|H_{\mathrm{par}(B_j)}| > 1, |H_{\mathrm{par}(B_{j'})}| = 1$) Symmetric to the previous case.

        \item (Case $|H_{\mathrm{par}(B_j)}| > 1, |H_{\mathrm{par}(B_{j'})}| > 1$) This case is straightforward. In this case, both parents have more than one centers, thus no matter which center the algorithm deletes, the eliminated center is not the only center of its parent.

        \item ($|H_{\mathrm{par}(B_j)}| = 1, |H_{\mathrm{par}(B_{j'})}| = 1$) In this case we argue that  $B_{j},B_{j'}$ could not have been identified in line \ref{line:flag} of \Cref{alg:final_pruning}, meaning that this is not a valid case to consider. To show this, let $\ell,\ell'$ be the parents of $B_j$ and $B_{j'}$ respectively.
        By the second and third properties of $B_j$, $\sigma_{B_j} \leq 21C \sigma_\ell$ and $\| \mu_{B_j} - \mu_\ell \| \leq 14 C \sigma_\ell/\sqrt{\alpha}$. Similarly, $\sigma_{B_{j'}} \leq 21C \sigma_{\ell'}$ and $\| \mu_{B_{j'}} - \mu_{\ell'} \| \leq 14 C\sigma_{\ell'}/\sqrt{\alpha}$. 
        \begin{align*}
            \| \mu_{B_{j}} - \mu_{B_{j'}} \| &\geq \| \mu_{\ell} - \mu_{\ell'} \| - \| \mu_{B_j} - \mu_\ell \| - \| \mu_{B_{j'}} - \mu_{\ell'} \| \\
            &\geq 10^{5}C^2(\sigma_\ell + \sigma_{\ell'})/\sqrt{\alpha} -14 C \sigma_\ell/\sqrt{\alpha}-14C \sigma_{\ell'}/\sqrt{\alpha} \\
            &\geq (10^{5}-14)C^2(\sigma_\ell + \sigma_{\ell'})/\sqrt{\alpha} \\
            &>  4761 C (\sigma_{B_{j}} + \sigma_{B_{j'}})/\sqrt{\alpha} \;.
        \end{align*}
        The above means that the check of line \ref{line:flag} in \Cref{alg:final_pruning} could not be satisfied for $B_j,B_{j'}$.
    \end{enumerate}

    It only remains to prove \Cref{cl:temp}.
    \begin{proof}[Proof of \Cref{cl:temp}]
        Let $B_j,B_{j'}$ be sub-clusters with the same parent $\ell$. Then by \Cref{it:mean_close} and a triangle inequality, $\| \mu_{B_j} - \mu_{B_{j'}} \| \leq \| \mu_{B_j} - \mu_{\ell} \|+\| \mu_{B_{j'}} - \mu_{\ell} \| \leq  28 C \sigma_\ell/\sqrt{\alpha}$.

    Now, if $B_j$ has parent $\ell$ and $B_{j'}$ has parent $\ell'$, then by \Cref{it:mean_close} and reverse triangle inequality:
    \begin{align*}
        \| \mu_{B_j} - \mu_{B_{j'}} \| &\geq \| \mu_\ell - \mu_{\ell'} \| - \| \mu_\ell - \mu_{B_j}\|
        -\| \mu_{\ell'} - \mu_{B_{j'}}\| \\
        &\geq 10^{5} C^2 (\sigma_\ell + \sigma_{\ell'})/\sqrt{\alpha} - 14 C \sigma_\ell/\sqrt{\alpha} -  14 C\sigma_{\ell'}/\sqrt{\alpha} \\
        &> 10^4 C^2 (\sigma_\ell + \sigma_{\ell'})/\sqrt{\alpha} \;. \tag{$C>1$}
    \end{align*}
    \end{proof}
\end{proof}
    
\section{Overall analysis of \Cref{alg:main_general}}
\label{sec:wrapup}

In this brief section, we combine the results and analyses in \Cref{sec:program,sec:size-pruning,sec:distance-pruning} to prove \Cref{thm:main_v2_general}.

    \begin{proof}[Proof of \Cref{thm:main_v2_general}]

    Let $s_{\max}$ denote the maximum element of the list $L_{\mathrm{stdev}}$ created in line \ref{line:listdec}. By \Cref{cor:induction}, after the loop of line \ref{line:main_loop_gen} ends, the list $L$ of candidate mean vectors that the algorithm has created is such that (i) for every element $\hat \mu_j \in L$ there exists a true cluster $S_i$ such that $\| \hat \mu_j - \mu_i \|_2  \leq 4061 C^2 \sigma_j/\sqrt{\alpha}$, and (ii) for every true cluster $S_i$ with $\sigma_{S_i} \leq s_{\max}$, there exists a $\hat \mu_j\in L$ such that $\| \hat \mu_j - \mu_i \|_2  \leq 4061 C^2 \sigma_j/\sqrt{\alpha}$.
    Furthermore, we also know by \Cref{cl:list_covariances} that, for every true cluster $S_i$, there exists an $\hat s$ in the list such that $\sigma_{S_{i}} \leq \hat s$.
    This implies that $s_{\max} \ge \max_i \sigma_{S_i}$, and guarantee (ii) above applies to every true cluster $S_i$.

    Following the structure of the algorithm, we use \Cref{lem:pruning} to reason about line~\ref{line:sizebased} of \Cref{alg:main_general}.
    To check that the lemma is indeed applicable, we need to show that $L$ can be partitioned into disjoint sets $H_1,\ldots,H_k$ such that for every $i \in [k]$, $H_i$ consists of the vectors $\hat \mu_j$ satisfying $\| \hat \mu_j - \mu_i \|_2 \leq 4061 C^2 \sigma_i/\sqrt{\alpha}$, and that $H_i \neq \emptyset$ for all $i \in [k]$.
    This is indeed true for the sets $H_i :=\{ \hat \mu \in L : \| \hat \mu - \mu_i \|_2 \leq 4061 C^2 \sigma_i/\sqrt{\alpha}\}$.
    The sets are disjoint because of our assumption that $\| \mu_i - \mu_{i'}\|_2 >10^{5}C^2(\sigma_i + \sigma_{i'})/\sqrt{\alpha}$ for every $i\neq i'$, and their union is equal to the entire $L$ because of the guarantee (i) from the previous paragraph. Finally, the fact that $H_i \neq \emptyset$ for all $i \in [k]$ holds because of the guarantee (ii) of the previous paragraph.

    The conclusion of \Cref{lem:pruning} is that, after we apply the $\sizebasedpruning$ algorithm in line \ref{line:sizebased} of \Cref{alg:main_general}, the resulting list $L'$  will admit a partition  $H'_1,\ldots,H'_k$ with the same properties as before, but also with the added property that every Voronoi cluster $A_j' := \{x \in T : \arg\min_{\hat \mu_{j'} \in L'}\|x - \hat \mu_{j'} \|_2 = j \}$ for  $j \in [|L'|]$ that corresponds to the centers of the output list $L'$, satisfies $|A_j'| \geq 0.96 \alpha n$.

    Next we use \Cref{lem:distance-based-pruning} to analyze the application of $\distancebasedpruning$ to the list $L'$ in line \ref{line:distancebased} of \Cref{alg:main_general}.
    Let us use $L''$ to denote the output of $\distancebasedpruning(L',T,\alpha)$.
    The lemma is applicable because of the conclusion of the previous paragraph.
    In turn, the conclusion of \Cref{lem:distance-based-pruning} is that with probability at least $1-\delta/2$ (over the randomness of the algorithm, in particular, the uses of filtering from \Cref{fact:filtering}),
    \begin{enumerate}[label=(\alph*),leftmargin=*]
        \item \label{it:sp1} The list $L''$ of centers admits a partition $H''_1,\ldots,H''_k$ with the same properties as before.
        \item \label{it:sp2} Every set in the Voronoi partition corresponding to these centers $A_j'' = \{x : \arg\min_{\hat \mu_{j'} \in L''} \| x - \hat \mu_{j'} \|_2 = j \}$ have sizes $|A_j''| \geq 0.96\alpha n$.
        \item \label{it:sp3} If $B_1'',\ldots,B''_{|L''|}$ denote the output of $\filteredvoronoi(L'',T)$ then it holds that $\| \mu_{B_j''} - \mu_{B_{j'}''} \|_2 \geq 4761 C (\sigma_{B_j''} + \sigma_{B_{j'}''})/\sqrt{\alpha}$ for every $j \neq j'$.
    \end{enumerate}

    Note that $\filteredvoronoi(L'',T)$ is the last step of \Cref{alg:main_general}.
    We will show that all the guarantees of the output \Cref{thm:main_v2_general} follow by \Cref{it:sp1,it:sp2,it:sp3} and a final application of \Cref{cor:filteredvoronoi} (which is applicable because of \Cref{it:sp1,it:sp2} above):

    \Cref{it:concl1} in the conclusion of \Cref{thm:main_v2_general} is true by the fact that $|A_j''| \geq 0.96\alpha n$ from \Cref{it:sp2} and the fact that the filtering in $\filteredvoronoi(L,T)$ only removes $4\%$ of the points in $A_j''$ (see \Cref{it:03} in \Cref{cor:filteredvoronoi}) with probability $1-\delta/2$.

    For \Cref{it:concl3} in the conclusion of \Cref{thm:main_v2_general}, we have the following:
    \Cref{it:concl4} holds by \Cref{it:01} in \Cref{cor:filteredvoronoi}.
    \Cref{it:concl5} holds by \Cref{it:02} in \Cref{cor:filteredvoronoi}.
    \Cref{it:concl6} holds by \Cref{it:03} in \Cref{cor:filteredvoronoi}.
    \Cref{it:concl7} follows from \Cref{it:04} in \Cref{cor:filteredvoronoi}.
    \Cref{it:concl8} holds by \Cref{it:sp3} above.

    Moreover, the number $m$ of the output sets $B_1,\ldots, B_m$ is at most $1/(0.92\alpha)$ since each set has at least $0.94 \alpha n$ points and the sets are disjoint.

    Finally, the algorithm runs in time $\poly(nd/\alpha)$-time because the size of the lists $L_{\mathrm{mean}},L_{\mathrm{stdev}}$ is polynomial in $n$ and $1/\alpha$, which means that the size of $L$ is also polynomial, and finally since the two pruning algorithms in lines \ref{line:sizebased} and \ref{line:distancebased} delete one element of $L$ at each step until termination, the overall number of steps is polynomial.
    It can also be checked that each step involves calculations that can be implemented in $\poly(nd/\alpha)$-time.

    \end{proof}

\section{Clustering under the no large sub-cluster condition}
\label{sec:NLSC}

The previous section analyzes \Cref{alg:main_general} in the general case, where the underlying mixture satisfies information-theoretically optimal separation, and the algorithm only knows a lower bound $\alpha$ to the mixing weight.
As we have shown in the introduction, we cannot aim to return an accurate clustering close to the ground truth, but instead, we return an accurate \emph{refinement} of the ground truth clustering.

In this section, we study the no large sub-cluster (NLSC) condition (\Cref{def:NLSC}), which is a deterministic condition on the sample set that guarantees that \Cref{alg:main_general} in fact returns an accurate clustering instead of just a refinement.
We first compare our NLSC condition with that proposed by~\cite{BKK22}.
Even though the conditions are qualitatively similar, our choice of parameters makes our NLSC condition a stronger assumption.
We explain in \Cref{sec:NLSC_comparison} why the stronger NLSC assumption is necessary due to our weaker separation assumption.

We then show in \Cref{sec:NLSC_proof} that, under the NLSC condition (\Cref{def:NLSC}), \Cref{alg:main_general} will return exactly $k$ sets, one per mixture component, despite not knowing $k$.
This is stated as \Cref{cor:NLSC}, the formal version \Cref{cor:NLSC_informal}.
Afterwards, we also show that the general class of well-conditioned log-concave distributions yield samples that satisfy this condition with high probability, as long as the dimensionality is large and the sample complexity is polynomially large.

\subsection{Comparison with the NLSC condition from \cite{BKK22}}
\label{sec:NLSC_comparison}

In this subsection, we compare our NLSC condition (\Cref{def:NLSC}) with the NLSC condition proposed by~\cite{BKK22}.
For the reader's convenience, we restate \Cref{def:NLSC} below.

\NLSCdef*

For contrast, the NLSC condition of~\cite{BKK22} is weaker.
Instead of $\sigma_{S'}$ being within a constant factor of $\sigma_{S_i}$, their requirement can be as small as an $\alpha$ factor of $\sigma_{S_i}$.

\begin{definition}[NLSC condition of~\cite{BKK22}]
\label{def:NLSC_BKK22}
    We say that the disjoint sets $S_1,\ldots,S_k$ of total size $n$ satisfy the ``No Large Sub-Cluster'' condition of~\cite{BKK22} with parameter $\alpha$ if 
    for any cluster $S_i$ and any subset $S' \subset S_i$ with $|S'| \geq 0.01 \sqrt{n}\log n$, it holds that $\sigma_{S'} \geq \frac{1}{5\sqrt{5}}\frac{|S'|}{|S_i|} \sigma_{S_i}$, where $\sigma_{S'}$ is the square root of the largest eigenvalue of the covariance matrix of $S'$.
\end{definition}

The two differences between the definitions are \emph{(i}) the minimum size of the subset $S'$ being considered, which is an insignificant difference, and more importantly \emph{(ii}) the lower bound of $\sigma_{S'}$.
In our definition, the lower bound is a small constant factor of $\sigma_{S_i}$, but their definition uses a factor that scales with the ratio of the set sizes, potentially interpolating between $\Theta(\alpha)$ and $\Theta(1)$.

We will now show that, under the separation assumption of $C\cdot(\sigma_i+\sigma_j)/\sqrt{\alpha}$ between pairs of clusters, for any large constant $C$, there is an explicit construction of a sample set where the NLSC condition of~\cite{BKK22} allows for two substantially different clusterings satisfying the separation assumption, whereas our NLSC condition (by the result of \Cref{cor:NLSC} below) only allow clusterings that are essentially the same as each other.

First consider a 1-dimensional set of points $U$ of $\alpha n$ points, distributed as a uniform grid over the interval $[-\frac{1}{2},\frac{1}{2}]$.
Its mean is 0, and its variance is $\frac{1}{12}-o(1)$, where the $o(1)$ term goes to 0 as $\alpha n \to \infty$.
It is straightforward to check via a ``swapping'' argument that, for any subset $U'$ of size at least $0.8\alpha n$ (which is thus at least a 0.8-fraction of $U$), we have $\sigma_{U'} \ge 0.8 \sigma_U \ge \frac{1}{5\sqrt{5}}\frac{|U'|}{|U|}\sigma_{U}$.

We then use $U$ to construct a high-dimensional sample set.
Set the ambient dimensionality to be $d = 1/(2\alpha)$.
We will embed a set $U$ along each Euclidean axis, symmetrically in the positive and negative coordinates.
For $i\in [d]$, construct the set $S_i^{+} = \{(x + (C/\sqrt{\alpha}))e_i : x \in U  \}$, which is a set of points that are non-zero only in the $i^\text{th}$ coordinate, embedded on the positive side of the $i^\text{th}$ axis, and similarly construct $S_i^{-} = \{(x - (C/\sqrt{\alpha}))e_i : x \in U  \}$.
Let the set $S$ be the union of all these $S_i^+$ and $S_i^-$ across $i \in [d]$, giving a total of $n$ points.

We claim that, according to the NLSC condition of~\cite{BKK22}, there are two very different but both valid clusterings of $S$: \emph{i}) treating every $S_i^+$ and $S_i^-$ as a separate cluster, and \emph{ii}) treating the entire $S$ as a single cluster.
We now verify both clusterings.

Recall that, to verify the validity of a clustering, we need to check that \emph{a}) each cluster has size at least $\alpha n$, \emph{b}) the clusters are well-separated, and \emph{c}) the NLSC condition of~\cite{BKK22} is satisfied.

For the clustering treating each $S_i^+$ and $S_i^-$ as separate clusters, point (\emph{a}) is trivial, and (\emph{c}) is true by construction of $U$.
It remains to check the cluster separation assumption (point (\emph{b}) above).
The minimum distance between (the means of) a pair of clusters is $\sqrt{2}C/\sqrt{\alpha}$, and each cluster has variance upper bounded by $1/12$.
On the other hand, the required separation is $C \cdot(1/\sqrt{12} + 1/\sqrt{12})/\sqrt{\alpha} < \sqrt{2}C/\sqrt{\alpha}$.
Thus the separation assumption is indeed satisfied.

Now consider the clustering treating the entire set $S$ as a single cluster.
Point (\emph{a}) is again trivial, and so is point (\emph{b}).
It remains to check point (\emph{c}), which is the NLSC condition of~\cite{BKK22}.

By construction of the set $S$, its mean is 0 and its covariance matrix is a multiple of the identity.
We bound above its variance along an axis direction, in order to establish the NLSC condition of~\cite{BKK22}.
By the law of total variance, we can write 
\begin{align*}
    \Cov(S)_{ii} = 2\alpha\Var(U) + 2\alpha (C/\sqrt{\alpha})^2 \le \alpha/6 + 2C^2,
\end{align*}
since $\Var(U) \le 1/12$.
As long as $\alpha\ll 1$ and $C > 1$, we have that $\|\Cov(S)\|_\op = \Cov(S)_{ii} \le 2.1C^2$.

Now consider an arbitrary subset $S' \subseteq S$ (in fact, we will not need to lower bound its size for the analysis).
By an averaging argument, there must exist some dimension $i$ such that at least $2\alpha|S'|$ points of $S'$ lie in $S_i^+ \cup S_i^-$.
We will lower bound the variance of $S'$ in direction $e_i$.

Either at least 50\% of the points in $|S'|$ lie in $S_i^+ \cup S_i^-$ or at least 50\% of the points lie at the origin in direction $e_i$.

In the former case, since $S_i^+ \cup S_i^-$ has size $2\alpha n$, we know that $|S'| \le 2\alpha n/0.5 = 4 \alpha n$.
Moreover, by an averaging argument, there are at least $|S'|/4$ points in one of $S'\cap S_i^+$ or $S'\cap S_i^-$.
Without loss of generality, we assume it is the $+$ side.
By construction of $U$, the variance of $S' \cap S_i^+$ in the $e_i$ direction is at least $\frac{1}{13}\frac{|S'\cap S_i^+|}{\alpha n} \ge 0.01\frac{|S'|}{\alpha n}$, where the first lower bound follows from having a sufficiently large $\alpha n$.
This in turn lower bounds the variance of $S'$ in the $e_i$ direction by $0.01\frac{|S'|}{\alpha n} \cdot |S'\cap S_i^+|/|S'| \ge 0.002 \frac{|S'|}{\alpha n}$.
The variance lower bound for $S'$ required by the NLSC condition of~\cite{BKK22} is at most $\frac{1}{125} (\frac{|S'|}{n})^2 \cdot 2.1C^2 \le 0.07 \alpha \frac{|S'|}{n} \cdot C^2$.
Thus, as long as $\alpha$ is upper bounded by some constant much smaller than $1/C$, we will have $0.07 \alpha \frac{|S'|}{n} \cdot C^2 \le 0.002 \frac{|S'|}{\alpha n}$, and the NLSC condition of~\cite{BKK22} is satisfied in this case.

In the latter case, we know that there are at least $0.5|S'|$ points that project to the origin in dimension $i$, and we also showed previously that there are at least $2\alpha|S'|$ points in $S' \cap (S_i^+ \cup S_i^-)$.
Further observe that points in $S' \cap (S_i^+ \cup S_i^-)$ have distance at least $C/\sqrt{\alpha}-\frac{1}{2}$ from the origin in the direction $e_i$.
Using the formula that the variance of $S'$ in direction $e_i$ is equal to
\[ \frac{1}{2|S'|^2} \sum_{x \in S'} \sum_{y \in S'} (x_i - y_i)^2\, ,\]
we can thus lower bound this directional variance by
\[ \frac{1}{2|S'|^2} 2(0.5|S'|)(2\alpha|S'|) \cdot \left(C/\sqrt{\alpha} - \frac{1}{2}\right)^2 \ge 0.25C^2\]
whenever $C > 1$ and $\alpha < 1$.
Finally, we note that $0.25 C^2 \ge \frac{1}{125}\left(\frac{|S'|}{|S|}\right)^2 \cdot 2.1C^2$, where the right hand side is the NLSC condition (of~\cite{BKK22}) variance lower bound, meaning that the NLSC condition is also satisfied in this case.

To summarize, we have exhibited a set $S$ such that, under the separation assumption of $C \cdot(\sigma_i+\sigma_j)/\sqrt{\alpha}$, the NLSC condition of~\cite{BKK22} still allows for two very different clusterings of the set $S$ as long as we choose $\alpha$ sufficiently small as a function of the assumed constant $C$.

On the other hand, our stronger NLSC condition lets us prove \Cref{cor:NLSC} below, which shows that the algorithm will always output a clustering close to the ground truth.
As such, there cannot be two substantially-different ground truth clusterings under our stronger assumption.

\subsection{NLSC implies accurate clustering}
\label{sec:NLSC_proof}

We prove \Cref{cor:NLSC}, which states that, if we assume the NLSC condition (\Cref{def:NLSC}), then \Cref{alg:main_general} returns a clustering instead of just a refinement.
That is, it returns exactly $k$ sets.
After that, we show that well-conditioned high-dimensional log-concave distributions give samples that satisfy the NLSC condition with high probability.

\begin{corollary}\label{cor:NLSC}
    If in the setting of \Cref{thm:main_v2_general} we additionally assume that the sets $S_i$ jointly satisfy the NLSC assumption with parameter $\alpha$ across all $i\in [k]$, then the algorithm returns exactly one sample set per mixture component.
    More precisely, for all $i \in [k]$, the set $H_i$ mentioned in the statement of \Cref{thm:main_v2_general} is a singleton.
    As a consequence, if $j$ is the unique index in $H_i$ in the context of \Cref{thm:main_v2_general}, then we have $\|\mu_{B_j} - \mu_i\| \le O(\sigma_i)$.
\end{corollary}

    \begin{proof}[Proof of \Cref{cor:NLSC}]
    To show this by contradiction, suppose that there are two output sets $B,B'$ that correspond to the same cluster $S_i$, i.e.~$B \subseteq \cB_i$ and $B' \subseteq \cB_i$ according to \Cref{it:concl3_new} of \Cref{thm:main_v1_general}.

    For the set $B$, observe that we have $|B \cap S_i| \geq 0.96|B| \geq 0.88 \alpha |S_i|$.
    The first inequality is a consequence of \Cref{it:concl6_new,it:concl1_new} (see \Cref{remark:intersection}), and the second inequality uses $|B| \ge 0.92 \alpha n$ and $|S_i| \le n$.
    
    By an application of \Cref{fact:513}, and a subsequent usage of the NLSC assumption, we have that
    \begin{align}
        \| \mu_{B \cap S_i} - \mu_{S_i} \|_2 \leq \frac{\sigma_{S_i}}{\sqrt{0.88\alpha}} 
        \leq \frac{10\sigma_{B \cap S_i}}{\sqrt{0.88\alpha}}
        \leq \frac{10\sigma_{B}}{0.96\sqrt{0.88\alpha}}
        \leq \frac{12\sigma_{B}}{\sqrt{\alpha}}\;,
    \end{align}
    where the first inequality is an application of \Cref{fact:513} using $|B \cap S_i| \geq 0.88 \alpha |S_i|$, the second inequality uses the NLSC assumption, and the third inequality uses the fact that $|B \cap S_i| \geq 0.96|B|$ implies $\sigma_{B \cap S_i} \leq \sigma_B/0.96$.
    Moreover, since $|B \cap S_i| \geq 0.96|B|$, by another application of \Cref{fact:513},
    \begin{align}
        \| \mu_{B \cap S_i} - \mu_B \|_2 \leq  \sigma_B/\sqrt{0.96} \;.
    \end{align}
    The above two inequalities together imply $\| \mu_B - \mu_{S_i} \| \leq 13 \sigma_{B}/\sqrt{\alpha}$.
    By symmetric arguments for $B'$, we also have  $\| \mu_{B'} - \mu_{S_i} \| \leq 13 \sigma_{B'}/\sqrt{\alpha}$.
    This then implies $\| \mu_B - \mu_{B'} \|_2 \leq 13( \sigma_{B} +  \sigma_{B'})/\sqrt{\alpha}$.
    We have thus contradicted \Cref{it:guarantee_sep} of the theorem (because the constant $C$ there is $C>1$).
\end{proof}

We now show that well-conditioned log-concave distributions yield samples that satisfy the no large sub-cluster condition with high probability.
We first start with isotropic distributions.

\begin{proposition}\label{prop:log-concave}
    Consider an arbitrary $d$-dimensional isotropic log-concave distribution $D$.
    If $d \ge c \cdot \log^8\frac{1}{\alpha}$ for some sufficiently large constant $c$, then it suffices to take a set $S$ of $\tilde O((d+\log\frac{1}{\delta})/\alpha^2)$ samples from $D$ so that, with probability at least $1-\delta$, for any subset $S' \subseteq S$ with $|S'| \geq 0.8\alpha$, we have $\| \Cov(S') \|_\op \geq 0.7$.
\end{proposition}

\begin{proof}
    First observe that isotropic log-concave distributions $D$ concentrate around a thin spherical shell.
    Specifically, a result of \cite{fleury2010concentration} shows that
    \[ \Pr_{X \sim D}\left[\left(1-\frac{t}{d^{1/8}}\right)\sqrt{d} \le \|X\|_2 \le \left(1+\frac{t}{d^{1/8}}\right)\sqrt{d} \right] \ge 1-O\left(e^{-\Omega(t)}\right)\]
    for all $t \in [0, d^{1/8}]$.
    Taking $t = \Theta(\log\frac{1}{\alpha})$ and using the assumption that $d \gg \log^8\frac{1}{\alpha}$, this implies that with probability at least $1-\alpha/1000$, we have $\|X\|_2 \ge \sqrt{0.99d}$.
    Thus, by standard Chernoff bounds, if we take at least $O((1/\alpha)\log\frac{1}{\delta})$ many samples for some sufficiently large hidden constant, then with probability at least $1-\delta/2$, at most an $\alpha/100$ fraction of the samples have $\|X\|_2 < \sqrt{0.99d}$.

    We will further show that the following claim that with high probability over the entire sample set, any $\alpha$-fraction of the samples must have mean not too far from the origin.
    
    \begin{claim}
    Suppose $S$ is a set of samples drawn from distribution $D$, of size at least a large constant multiple of $(d+\log\frac{1}{\delta})/\alpha^2$.
    Then, with probability at least $1-\delta/2$ over the randomness of $S$, for any arbitrary subset $S' \subset S$ of size at least $0.9 \alpha |S|$, we have $\|\mu_{S'}\|_2 \le O(\log\frac{1}{\alpha})$.
    \end{claim}
    
    \begin{proof}
    We will use the standard fact that isotropic log-concave distributions are sub-exponential, whose samples are in turn stable with high probability, as long as the sample size is sufficiently large (see Exercise 3.1 in~\cite{diakonikolas2023algorithmic} for example).
    In particular, with probability at least $1-\delta/2$ over a set $S$ of $\tilde O((d+\log\frac{1}{\delta})/\alpha^2)$ samples from a log-concave distribution $D$ with unit covariance $D$ and mean $0$, it holds that for every subset $\tilde{S} \subseteq S$ of size at least $(1-\alpha)|S|$, we have $\|\mu_{\tilde{S}}\|_2 \le O(\alpha \log\frac{1}{\alpha})$.

    Now consider any subset $S'' \subseteq S$ of size between $(\alpha/2)|S|$ and $\alpha|S|$.
    Its complement $\tilde{S} = S \setminus S''$ satisfies $\|\mu_{\tilde{S}}\|_2 \le O(\alpha \log\frac{1}{\alpha})$.
    Furthermore, by alternatively taking $\tilde{S} = S$, we have $\|\mu_S\|_2 = O(\alpha \log\frac{1}{\alpha})$.
    Thus, $\|\mu_{S''}\|_2 \le \frac{2}{\alpha}\|\mu_S - (1-\alpha) \mu_{\tilde{S}}\|_2 \le \frac{1}{\alpha} O(\alpha \log\frac{1}{\alpha}) = O(\log\frac{1}{\alpha})$ by the triangle inequality.

    Finally, consider any subset $S'$ of size at least $0.8\alpha|S|$.
    Observe that this set $S'$ can always be partitioned into sets $S''$ of sizes between $(\alpha/2)|S|$ and $\alpha|S|$, each of which satisfies $\|\mu_{S''}\|_2 \le O(\log\frac{1}{\alpha})$.
    Moreover, the mean $\mu_{S'}$ of $S'$ is just the convex combination of the means of these smaller disjoint subsets.
    This implies that $\|\mu_{S'}\|_2 \le O(\log\frac{1}{\alpha})$.
    \end{proof}

    To summarize, we have shown that, with probability at least $1-\delta$ over the randomness of the samples $S$, we have (a) at most an $\alpha/100$ fraction of the samples $x$ have $\|x\|_2 < \sqrt{0.99d}$, and (b) for any subset $S' \subseteq S$ of size at least $0.9\alpha n \ge 0.9\alpha |S|$, $\|\mu_{S'}\|_2 \le O(\log\frac{1}{\alpha})$.
    We are now ready to show the NLSC condition for $S$ conditioned on these two facts.

    First, take any subset $S'$ of size at least $0.8 \alpha n$.
    By condition (a) above, there are at least $0.75|S'|$ many points $x \in S'$ with $\|x\|_2 \ge \sqrt{0.99d}$.
    Thus, we have
    \[ \tr\left(\frac{1}{|S'|}\sum_{x \in S'} xx^\top \right) \ge 0.75d \;. \]
    Second, observe that the covariance of $S'$ is
    \[ \Cov(S') = \frac{1}{|S'|}\sum_{x \in S'} (x-\mu_{S'})(x-\mu_{S'})^\top = \frac{1}{|S'|}\sum_{x \in S'} xx^\top - \mu_{S'}\mu_{S'}^\top \;.\]
    Thus, we have
    \begin{align}
      \tr(\Cov(S')) = \tr\left(\frac{1}{|S'|}\sum_{x \in S'} xx^\top \right) - \tr(\mu_{S'}\mu_{S'}^\top) \ge 0.75d - O\left(\log^2\frac{1}{\alpha}\right) \ge 0.7d \;, \label{eq:trace_ineq}  
    \end{align}
    where the last inequality uses $d \gg \log^8\frac{1}{\alpha} \gg \log^2\frac{1}{\alpha}$.
    Since the trace is equal to the sum of all eigenvalues, \eqref{eq:trace_ineq} states that the average eigenvalue is at least $0.7$, thus the largest one should be $\|\Cov(S')\|_\op \ge 0.7$.
    
\end{proof}

Now we use the above proposition to show that samples from well-conditioned log-concave distributions satisfy \Cref{def:NLSC} with high probability.
In fact, the guarantees apply even to log-concave distributions for which there is a high-dimensional subspace $V$ that both \emph{i}) contains the largest variance direction and \emph{ii}) is well-conditioned in the projection onto $V$.

\begin{proposition}
\label{prop:log-concave_NLSC}
    Consider an arbitrary $d$-dimensional  log-concave distribution $D$ with covariance matrix $\Sigma$ such that there exists a subspace $V$ of dimension $\mathrm{dim}(V) \gg c \cdot \log^8\frac{1}{\alpha}$ which: (i) contains the top eigenvector of $\Sigma$ and (ii) the covariance matrix of the projected distribution $\mathrm{Proj}_{V}\, \Sigma \, \mathrm{Proj}_{V}^\top$ has condition number at most 2. 
    Then $\tilde{O}((d+\log\frac{1}{\delta})/\alpha^2)$ samples from $D$ suffice for the sample set to satisfy the NLSC condition (\Cref{def:NLSC}) with probability at least $1-\delta$.
\end{proposition}

\begin{proof}
    We first show the special case when $V$ is the entire $\R^d$, and by assumption, the condition number of $\Sigma$ is $\kappa \leq 2$.
    Observe that, by the property of the operator norm, for every matrix $A$ it holds $\|A\|_\op = \|\Sigma^{-1/2} \Sigma^{1/2} A \Sigma^{1/2} \Sigma^{-1/2} \|_\op \leq \| \Sigma^{-1/2} \|_\op^2 \|\Sigma^{1/2}A\Sigma^{1/2}\|_\op$, which rearranging gives 
    \begin{align}\label{eq:generic_eq}
        \|\Sigma^{1/2}A\Sigma^{1/2}\|_\op \geq \|A\|_\op/\| \Sigma^{-1/2} \|_\op^2 \;.
    \end{align}

    Let $S$ be the samples from $D$, and consider an arbitrary subset $S'\subset S$ with $|S'| \geq 0.8\alpha n$.
    Moreover, let the normalized versions of these sets be $\tilde S = \{ \Sigma^{-1/2}x : x \in S\}$ and $\tilde S' = \{ \Sigma^{-1/2}x : x \in S'\}$.
    The normalization means that the samples in $\tilde S,\tilde S'$ come from an isotropic log-concave distribution.
    Thus, by \Cref{prop:log-concave} we know that with  probability at least $1-\delta/2$, it holds 
    \begin{align}\label{eq:application}
        \| \Cov(\tilde S') \|_\op \geq 0.7 \;.
    \end{align}
    Putting everything together, we have
    \begin{align}
        \| \Cov( S') \|_\op &\geq \frac{1}{\|\Sigma^{-1/2}\|_\op^2} \| \Cov(\tilde S') \|_\op \tag{using \eqref{eq:generic_eq} with $A=\Cov(\tilde S')$}\\
        &\geq   \frac{0.7}{\|\Sigma^{-1/2}\|_\op^2} \tag{by \Cref{eq:application}}\\
        &\geq  0.35  \|\Sigma\|_\op  \tag{since condition number of $\Sigma$ is at most $\kappa \leq 2$}
    \end{align}
    Finally, we again note that isotropic log-concave distributions are sub-exponential and thus by standard arguments (see, e.g.~Exercise 3.1 in \cite{diakonikolas2023algorithmic}), $\tilde{O}((d+\log\frac{1}{\delta}))$ samples suffice to have that $\| \Cov(\tilde S) \|_\op \leq 1.001$ with probability at least $1-\delta/2$. This means that $\| \Cov(S) \|_\op = \| \Sigma^{1/2} \Cov(\tilde S) \Sigma^{1/2} \|_\op \leq \|\Sigma\|_{\op} \| \Cov(\tilde S) \|_\op \leq 1.001 \|\Sigma\|_{\op}$. Combining this with the fact that $\| \Cov( S') \|_\op \geq 0.35 \|\Sigma\|_\op$ (that we showed earlier), we obtain that $\| \Cov( S') \|_\op \geq 0.1 \| \Cov( S) \|_\op $, i.e.~that the NLSC condition holds.

    It is easy to extend the argument for a general subspace $V$ in the corollary statement.
    To see this, note that orthogonal projections preserve log-concavity, thus if we restrict everything to the subspace $V$, we could first show that NLSC holds in that subspace.
    That is, for any subset $S' \subseteq S$ with $|S'| \geq 0.8\alpha|S|$ of the data points, if $S_{V}:= \{ \mathrm{Proj}_{V}x : x \in S  \}$ and  $S'_{V}:= \{ \mathrm{Proj}_{V}x : x \in S'  \}$ denote the projected versions of the sets onto $V$ then $\sigma_{S'_{V}} \geq 0.1 \sigma_{S_{V}}$.
    Then, the two inequalities $\sigma_{S'} \geq \sigma_{S'_{V}}$ and $\sigma_{S_{V}} \geq 0.99 \sigma_{S}$ would imply that NLSC holds in the full-dimensional space. The first inequality is due to the fact that orthogonal projections can only decrease the variance, and the second inequality is because both $\sigma_{S_V}$ and $\sigma_S$ are with high probability close to $\sqrt{\|\Sigma\|_\op}$, by concentration of the empirical covariance matrix in every direction (Exercise 3.1 in \cite{diakonikolas2023algorithmic}).
\end{proof}

\section*{Acknowledgements} 

I.D.~is grateful to Ravi Kannan for numerous technical 
conversations during the Simons Institute program on the
Computational Complexity of Statistical Inference. 
His insights served as an inspiration for this work.


\bibliographystyle{alphaabbr}
\bibliography{allrefs.bib}

\newcommand{\etalchar}[1]{$^{#1}$}
\begin{thebibliography}{GEGMMI10}

\bibitem[AK01]{SanjeevK01}
S.~Arora and R.~Kannan.
\newblock Learning mixtures of arbitrary gaussians.
\newblock In {\em Proceedings on 33rd Annual {ACM} Symposium on Theory of
  Computing, 2001}, pages 247--257. {ACM}, 2001.

\bibitem[AM05]{AchlioptasMcSherry:05}
D.~Achlioptas and F.~McSherry.
\newblock On spectral learning of mixtures of distributions.
\newblock In {\em Proceedings of the Eighteenth Annual Conference on Learning
  Theory (COLT)}, pages 458--469, 2005.

\bibitem[AS12]{awasthi2012improved}
P.~Awasthi and O.~Sheffet.
\newblock Improved spectral-norm bounds for clustering.
\newblock In {\em Approximation, Randomization, and Combinatorial Optimization.
  Algorithms and Techniques}, pages 37--49. Springer, 2012.

\bibitem[BBV08]{BBV08}
M.-F. Balcan, A.~Blum, and S.~Vempala.
\newblock A discriminative framework for clustering via similarity functions.
\newblock In {\em STOC}, pages 671--680, 2008.

\bibitem[BDH{\etalchar{+}}20]{BakshiDHKKK20}
A.~Bakshi, I.~Diakonikolas, S.~B. Hopkins, D.~Kane, S.~Karmalkar, and P.~K.
  Kothari.
\newblock Outlier-robust clustering of gaussians and other non-spherical
  mixtures.
\newblock In {\em 61st {IEEE} Annual Symposium on Foundations of Computer
  Science, {FOCS} 2020}, pages 149--159. {IEEE}, 2020.

\bibitem[BDJ{\etalchar{+}}20]{bakshi2020robustly}
A.~Bakshi, I.~Diakonikolas, H.~Jia, D.~M. Kane, P.~K. Kothari, and S.~S.
  Vempala.
\newblock Robustly learning mixtures of $ k $ arbitrary gaussians.
\newblock {\em arXiv preprint arXiv:2012.02119}, 2020.
\newblock Conference version in STOC'22.

\bibitem[BKK22]{BKK22}
C.~Bhattacharyya, R.~Kannan, and A.~Kumar.
\newblock How many clusters? - an algorithmic answer.
\newblock In {\em Proceedings of the 2022 Annual ACM-SIAM Symposium on Discrete
  Algorithms (SODA)}, pages 2607--2640, 2022.

\bibitem[Bru09]{Brubaker:09}
S.~C. Brubaker.
\newblock {\em Extensions of Principle Components Analysis}.
\newblock PhD thesis, Georgia Institute of Technology, 2009.

\bibitem[BS10]{BelkinSinha:10}
M.~Belkin and K.~Sinha.
\newblock Polynomial learning of distribution families.
\newblock In {\em FOCS}, pages 103--112, 2010.

\bibitem[CMY20]{CherapanamjeriMY20}
Y.~Cherapanamjeri, S.~Mohanty, and M.~Yau.
\newblock List decodable mean estimation in nearly linear time.
\newblock In {\em 61st IEEE Annual Symposium on Foundations of Computer
  Science, FOCS 2020}, 2020.

\bibitem[CSV17]{CSV17}
M.~Charikar, J.~Steinhardt, and G.~Valiant.
\newblock Learning from untrusted data.
\newblock In {\em Proceedings of STOC 2017}, pages 47--60, 2017.

\bibitem[Das99]{Dasgupta:99}
S.~Dasgupta.
\newblock {Learning mixtures of Gaussians}.
\newblock In {\em Proceedings of the 40th Annual Symposium on Foundations of
  Computer Science}, pages 634--644, 1999.

\bibitem[DK20]{DiakonikolasK20}
I.~Diakonikolas and D.~M. Kane.
\newblock Small covers for near-zero sets of polynomials and learning latent
  variable models.
\newblock In {\em 61st {IEEE} Annual Symposium on Foundations of Computer
  Science, {FOCS} 2020}, pages 184--195. {IEEE}, 2020.

\bibitem[DK23]{diakonikolas2023algorithmic}
I.~Diakonikolas and D.~M. Kane.
\newblock {\em Algorithmic high-dimensional robust statistics}.
\newblock Cambridge University Press, 2023.

\bibitem[DKK{\etalchar{+}}16]{DKKLMS16}
I.~Diakonikolas, G.~Kamath, D.~M. Kane, J.~Li, A.~Moitra, and A.~Stewart.
\newblock Robust estimators in high dimensions without the computational
  intractability.
\newblock In {\em Proceedings of FOCS'16}, pages 655--664, 2016.

\bibitem[DKK20a]{DiakonikolasKK20}
I.~Diakonikolas, D.~Kane, and D.~Kongsgaard.
\newblock List-decodable mean estimation via iterative multi-filtering.
\newblock In {\em Advances in Neural Information Processing Systems 33: Annual
  Conference on Neural Information Processing Systems 2020, NeurIPS 2020,
  December 6-12, 2020, virtual}, 2020.

\bibitem[DKK{\etalchar{+}}20b]{DiakonikolasKKLT20}
I.~Diakonikolas, D.~M. Kane, D.~Kongsgaard, J.~Li, and K.~Tian.
\newblock List-decodable mean estimation in nearly-pca time.
\newblock {\em CoRR}, abs/2011.09973, 2020.
\newblock Conference version in NeurIPS'21.

\bibitem[DKK{\etalchar{+}}21]{diakonikolas2021list}
I.~Diakonikolas, D.~Kane, D.~Kongsgaard, J.~Li, and K.~Tian.
\newblock List-decodable mean estimation in nearly-pca time.
\newblock {\em Advances in Neural Information Processing Systems},
  34:10195--10208, 2021.

\bibitem[DKK{\etalchar{+}}22a]{DiakonikolasKKP22}
I.~Diakonikolas, D.~Kane, S.~Karmalkar, A.~Pensia, and T.~Pittas.
\newblock List-decodable sparse mean estimation via difference-of-pairs
  filtering.
\newblock In {\em NeurIPS}, 2022.

\bibitem[DKK{\etalchar{+}}22b]{diakonikolas2022clustering}
I.~Diakonikolas, D.~M. Kane, D.~Kongsgaard, J.~Li, and K.~Tian.
\newblock Clustering mixture models in almost-linear time via list-decodable
  mean estimation.
\newblock In {\em Proceedings of the 54th Annual ACM SIGACT Symposium on Theory
  of Computing}, pages 1262--1275, 2022.

\bibitem[DKP20]{DiaKP20fact:filtering}
I.~Diakonikolas, D.~M. Kane, and A.~Pensia.
\newblock Outlier {{Robust Mean Estimation}} with {{Subgaussian Rates}} via
  {{Stability}}.
\newblock In {\em Advances in {{Neural Information Processing Systems}} 33,
  {{NeurIPS}} 2020}, 2020.

\bibitem[DKS18a]{DKS18-list}
I.~Diakonikolas, D.~M. Kane, and A.~Stewart.
\newblock List-decodable robust mean estimation and learning mixtures of
  spherical gaussians.
\newblock In {\em Proceedings of the 50th Annual {ACM} {SIGACT} Symposium on
  Theory of Computing, {STOC} 2018}, pages 1047--1060, 2018.
\newblock Full version available at https://arxiv.org/abs/1711.07211.

\bibitem[DKS18b]{DiakonikolasKS18}
I.~Diakonikolas, D.~M. Kane, and A.~Stewart.
\newblock Sharp bounds for generalized uniformity testing.
\newblock In {\em Advances in Neural Information Processing Systems 31: Annual
  Conference on Neural Information Processing Systems 2018, NeurIPS 2018},
  pages 6204--6213, 2018.

\bibitem[Fle10]{fleury2010concentration}
B.~Fleury.
\newblock Concentration in a thin euclidean shell for log-concave measures.
\newblock {\em Journal of Functional Analysis}, 259(4):832--841, 2010.

\bibitem[GEGMMI10]{GGMM10}
L.~García-Escudero, A.~Gordaliza, C.~Matrán, and A.~Mayo-Iscar.
\newblock {A review of robust clustering methods}.
\newblock {\em Advances in Data Analysis and Classification}, 4(2):89--109,
  2010.

\bibitem[HL18]{HL18-sos}
S.~B. Hopkins and J.~Li.
\newblock Mixture models, robustness, and sum of squares proofs.
\newblock In {\em Proceedings of the 50th Annual {ACM} {SIGACT} Symposium on
  Theory of Computing, {STOC} 2018}, pages 1021--1034, 2018.

\bibitem[Hub64]{Huber64}
P.~J. Huber.
\newblock Robust estimation of a location parameter.
\newblock {\em Ann. Math. Statist.}, 35(1):73--101, 03 1964.

\bibitem[Kan21]{kane2021robust}
D.~M. Kane.
\newblock Robust learning of mixtures of gaussians.
\newblock In {\em Proceedings of the 2021 ACM-SIAM Symposium on Discrete
  Algorithms (SODA)}, pages 1246--1258. SIAM, 2021.

\bibitem[KK10]{kumar2010clustering}
A.~Kumar and R.~Kannan.
\newblock Clustering with spectral norm and the k-means algorithm.
\newblock In {\em 2010 IEEE 51st Annual Symposium on Foundations of Computer
  Science}, pages 299--308. IEEE, 2010.

\bibitem[KSS18]{KSS18-sos}
P.~K. Kothari, J.~Steinhardt, and D.~Steurer.
\newblock Robust moment estimation and improved clustering via sum of squares.
\newblock In {\em Proceedings of the 50th Annual {ACM} {SIGACT} Symposium on
  Theory of Computing, {STOC} 2018}, pages 1035--1046, 2018.

\bibitem[KSV05]{KSV:05}
R.~Kannan, H.~Salmasian, and S.~Vempala.
\newblock The spectral method for general mixture models.
\newblock In {\em Proceedings of the Eighteenth Annual Conference on Learning
  Theory (COLT)}, pages 444--457, 2005.

\bibitem[Lin95]{Lindsay:95}
B.~Lindsay.
\newblock {\em Mixture models: theory, geometry and applications}.
\newblock Institute for Mathematical Statistics, 1995.

\bibitem[LL22]{Liu022}
A.~Liu and J.~Li.
\newblock Clustering mixtures with almost optimal separation in polynomial
  time.
\newblock In S.~Leonardi and A.~Gupta, editors, {\em {STOC} '22: 54th Annual
  {ACM} {SIGACT} Symposium on Theory of Computing, 2022}, pages 1248--1261.
  {ACM}, 2022.

\bibitem[LM20]{liu2020settling}
A.~Liu and A.~Moitra.
\newblock Settling the robust learnability of mixtures of gaussians.
\newblock {\em arXiv preprint arXiv:2011.03622}, 2020.
\newblock Conference version in STOC'21.

\bibitem[LRV16]{LaiRV16}
K.~A. Lai, A.~B. Rao, and S.~Vempala.
\newblock Agnostic estimation of mean and covariance.
\newblock In {\em Proceedings of FOCS'16}, 2016.

\bibitem[MV10]{MoitraValiant:10}
A.~Moitra and G.~Valiant.
\newblock {Settling the polynomial learnability of mixtures of Gaussians}.
\newblock In {\em FOCS}, pages 93--102, 2010.

\bibitem[TSM85]{titterington_85}
D.~Titterington, A.~Smith, and U.~Makov.
\newblock {\em Statistical Analysis of Finite Mixture Distributions}.
\newblock Wiley, New York, 1985.

\bibitem[Tuk75]{Tukey75}
J.~Tukey.
\newblock Mathematics and picturing of data.
\newblock In {\em Proceedings of ICM}, volume~6, pages 523--531, 1975.

\bibitem[VW02]{VempalaWang:02}
S.~Vempala and G.~Wang.
\newblock A spectral algorithm for learning mixtures of distributions.
\newblock In {\em Proceedings of the 43rd Annual Symposium on Foundations of
  Computer Science}, pages 113--122, 2002.

\end{thebibliography}

\newpage
\appendix

\section*{Appendix}

\section{Omitted Proofs from \Cref{sec:prelim}}\label{appendix:prelims}
We restate and prove the following statements. 

\STABILITYSUBSETS*
\begin{proof}[Proof]
    Let $S'$ be a subset of $S$ with $|S'| \geq \alpha |S|$. According to \Cref{def:stability}, in order to show that $S'$ is $(1.23C/\sqrt{0.04\alpha},0.04)$-stable, we have to show that for any weight function $w : S' \to [0,1]$ with $\sum_{x \in S'} w_x \geq (1- 0.04)|S'|$, the weighted mean and second moment centered around $\mu$ are at most $1.23C/\sqrt{\alpha}$ and $38 C^2\sigma^2/\alpha$ respectively.

    For the mean, by an application of \Cref{fact:513}, we have that
    \begin{align*}
        \left\| \frac{\sum_{x \in S'}w_x x}{\sum_{x \in S'}w_x} - \mu \right\|_2
        &\leq \left\| \frac{\sum_{x \in S'}w_x x}{\sum_{x \in S'}w_x} - \mu_{S} \right\|_2 + \| \mu_{S} - \mu \|_2\\
        &\leq \frac{\sigma_S}{\sqrt{(1-0.04)\alpha}} + C\sigma \sqrt{0.04}\\
        &\leq \frac{C \sigma}{\sqrt{(1-0.04)\alpha}} + C\sigma \sqrt{0.04}\\
        &\leq \frac{1.23 C \sigma}{\sqrt{\alpha}} \;,
    \end{align*}
    where the first step is triangle inequality, the second step uses \Cref{fact:513} for the first term and stability of $S$ for the second term, and the next step uses stability condition for the covariance.

    For the second moment, we have that
    \begin{align*}
        \frac{1}{\sum_{x \in S'}w_x} \sum_{x \in S'} w_x (x-\mu)(x-\mu)^\top 
        &\preceq \frac{1}{(1-0.04)\alpha}  \frac{1}{|S|}\sum_{x \in S} w_x (x-\mu)(x-\mu)^\top \\
        &\preceq \frac{1}{(1-0.04)\alpha} \frac{1}{|S|}\sum_{x \in S}(x-\mu)(x-\mu)^\top \\
        &\preceq \frac{1}{(1-0.04)\alpha}  C^2 \sigma^2 I
        \preceq 38 \frac{C^2 \sigma^2}{\alpha} I \;,
    \end{align*}
    where the first step uses that $\sum_{x \in S'} w_x \geq (1-0.04)\alpha |S|$, and the last line uses stability for $S$.
\end{proof}

\LISTCOV*

\begin{proof}
    The algorithm is the following:
    \begin{enumerate}[leftmargin=*]
        \item $L \gets \emptyset$.
        \item For every pair $x,y \in T$:
        \begin{enumerate}
            \item Add $\sqrt{2^{-j}\|x-y\|_2^2} $ to the list $L$ for every $j=0,1,\ldots,  \log(2 m^2)$.
        \end{enumerate}
        \item Let $L' \gets \{ \sqrt{2} s: s \in L\}$.
        \item Return $L'$.
    \end{enumerate}

    Using the definition of the covariance matrix $\Cov(S) = \frac{1}{2 |S|^2}\sum_{x,y \in S}(x-y)(x-y)^\top$, we have that 
    $ \max_{x,y \in S} \|x-y\|_2^2/(2|S|^2) \leq \| \Cov(S) \|_\op \leq \max_{x,y \in S} \|x-y\|_2^2$.
    The algorithm adds to the output list every number starting from $\max_{x,y \in S} \|x-y\|_2$ down to $\max_{x,y \in S} \|x-y\|_2/\sqrt{2|S|^2}$ in factors of $\sqrt{2}$. %
    This means that the list $L$ contains an $s$ such that $s^2  \leq \| \Cov(S) \|_{\op} \leq 2s^2$.
    By multiplying each element in the list by $\sqrt{2}$, the final list $L'$ contains an  $\hat {s}$ such that $\| \Cov(S) \|_{\op}  \leq \hat s^2 \leq 2\| \Cov(S) \|_{\op}$.
\end{proof}

\end{document}